%% file: main_neurips.tex
\documentclass{article}

\PassOptionsToPackage{numbers, compress}{natbib}
\usepackage[preprint]{neurips_2026}

\usepackage[utf8]{inputenc} 
\usepackage[T1]{fontenc}    
\usepackage{hyperref}       
\usepackage{url}            
\usepackage{booktabs}       
\usepackage{amsfonts}       
\usepackage{nicefrac}       
\usepackage{microtype}      
\usepackage[dvipsnames]{xcolor}
\usepackage{etoolbox}

\usepackage{todonotes}

\input{settings}

\input{commands}

\title{FullFlow: Upgrading Text-to-Image Flow Matching Models for Bidirectional Vision--Language Generation}
\author{%
    Eric Tillmann Bill\textsuperscript{1}
    \quad
    Enis Simsar\textsuperscript{1}
    \quad
    Alessio Tonioni\textsuperscript{2}
    \quad
    Thomas Hofmann\textsuperscript{1}\vspace{1mm}\\
    \textsuperscript{1}ETH Zurich
    \quad
    \textsuperscript{2}Google\\[1mm]
    {\tt\small \url{https://ericbill21.github.io/fullflow/}}
}

\begin{document}

\maketitle

\begin{abstract}
    Modern text-to-image diffusion models encode rich visual priors, but expose them only through one-way text-conditioned generation. Existing unified vision--language models derived from them recover bidirectional capability through large-scale joint pretraining or substantial retraining of the text pathway, discarding the strong image prior the text-to-image backbone already encodes. We introduce \emph{FullFlow}, a parameter-efficient recipe that upgrades a pretrained rectified-flow text-to-image model into a bidirectional vision--language generator by training only LoRA adapters and lightweight text heads. FullFlow keeps images in their native continuous flow and adds a discrete insertion process for text. Separate image and text timesteps turn inference into trajectory selection in a two-dimensional generative space, enabling text$\rightarrow$image, image$\rightarrow$text, joint sampling, and partial-text prediction with a single backbone. On Stable Diffusion 3~(SD3) under an identical trainable-parameter count and matched LoRA rank, FullFlow improves text$\rightarrow$image FID from $62.7$ to $31.6$ and image$\rightarrow$text CIDEr from $2.0$ to $99.4$ over a LoRA equivalent following the previous SOTA formulation (Dual~Diffusion) at matched wall-clock training time, while reducing peak VRAM from ${\sim}84$\,GB to ${\sim}38$\,GB and raising throughput by ${\sim}8\times$ on two RTX~A5000 GPUs in under 24 hours, training only ${\sim}5\%$ of the backbone parameters. The same recipe transfers to FLUX.1-dev and supports downstream VQA through partial-text generation. These results show that strong bidirectional vision--language capability can be unlocked from pretrained text-to-image flow models without full multimodal pretraining.
\end{abstract}

\input{sections/intro}
\input{sections/background}
\input{sections/method}
\input{sections/related_work}
\input{sections/results}
\input{sections/limitations}

\clearpage
\bibliographystyle{plainnat}
\bibliography{references}

\clearpage
\appendix
\section*{Appendix}

\input{appendix/training}

\input{appendix/theory}
\input{appendix/architecture}
\input{appendix/results}
\input{appendix/examples}


\end{document}

%% file: settings.tex
\usepackage{amsmath,amsfonts,amsthm,amssymb,mathrsfs,bbm,mathtools,nicefrac}
\usepackage[italic]{derivative}
\usepackage{xparse}
\usepackage{upgreek}

\usepackage{csvsimple}
\usepackage{graphicx}
\usepackage{tikz}
\usetikzlibrary{calc,arrows,shapes,positioning,arrows.meta}
\usepackage{pgfplots}
\usepgfplotslibrary{groupplots,fillbetween,statistics}
\pgfplotsset{compat=1.18}
\usepackage{tabularx}
\usepackage{multirow}
\usepackage{longtable}
\usepackage{makecell}
\usepackage{wrapfig}
\usepackage{subcaption}
\usepackage{adjustbox}
\usepackage{array}
\usepackage{bm}
\usepackage{dsfont}
\usepackage{placeins}

\usepackage{siunitx}
\usepackage{enumitem}
\usepackage[noabbrev,capitalize]{cleveref}
\usepackage{ifthen}
\usepackage{xspace}
\allowdisplaybreaks

\newtheorem{thm}{Theorem}[section]

\theoremstyle{definition}

\numberwithin{equation}{section}
\Crefname{thm}{Theorem}{Theorems}

\definecolor{blue}{RGB}{2,106,253}
\definecolor{red}{RGB}{245,51,30}
\definecolor{green}{RGB}{96,189,69}
\definecolor{purple}{RGB}{200,0,240}

\definecolor{dit_blue}{HTML}{5580E6}
\definecolor{dit_green}{HTML}{88B06D}
\definecolor{dit_red}{HTML}{D93B36}
\definecolor{dit_yellow}{HTML}{F5A83D}
\definecolor{dit_grey}{HTML}{808080}
\definecolor{dit_lightgrey}{HTML}{D5D7E2}
\definecolor{dit_lightblue}{HTML}{CADAFB}
\definecolor{dit_lightgreen}{HTML}{C9E2C0}
\definecolor{dit_lightyellow}{HTML}{F2CE94}
\definecolor{dit_lightred}{HTML}{E8827B}

\newcommand{\tok}[1]{\ensuremath{\langle \text{\texttt{#1}} \rangle}}

\newcommand{\tinytit}[1]{\noindent\textbf{#1}\enspace\ignorespaces}

\usepackage[most]{tcolorbox}
\tcbuselibrary{theorems}

\newtcbtheorem
  [use counter*=thm,number within=section,crefname={Example}{Examples},Crefname={Example}{Examples}]%
  {ex}
  {Example}
  {%
    before skip=10pt,after skip=10pt,
    left=0.2cm,right=0.2cm,top=0cm,
    toptitle=0.2cm,bottomtitle=0cm,
    breakable,
    toprule at break=0.2cm,
    sharp corners,
    colback=blue!10,
    coltitle=black,
    colframe=blue!10,
    fonttitle=\bfseries,
    parbox=false,
    halign=justify,
  }%
  {ex}%

\newtcbtheorem
  [use counter*=thm,number within=section,crefname={Remark}{Remarks},Crefname={Remark}{Remarks}]%
  {rmk}
  {Remark}
  {%
    before skip=10pt,after skip=10pt,
    left=0.2cm,right=0.2cm,top=0cm,
    toptitle=0.2cm,bottomtitle=0cm,
    breakable,
    toprule at break=0.2cm,
    sharp corners,
    colback=gray!20,
    coltitle=black,
    colframe=gray!20,
    fonttitle=\bfseries,
    parbox=false,
    halign=justify,
  }%
  {rmk}%

\newtcbtheorem
  [use counter*=thm,number within=section,crefname={Problem}{Problems},Crefname={Problem}{Problems}]%
  {exc}
  {Problem}
  {%
    before skip=10pt,after skip=10pt,
    left=0.2cm,right=0.2cm,top=0cm,
    toptitle=0.2cm,bottomtitle=0cm,
    breakable,
    toprule at break=0.2cm,
    sharp corners,
    colback=purple!10,
    coltitle=black,
    colframe=purple!10,
    fonttitle=\bfseries,
    parbox=false,
    halign=justify,
  }%
  {exercise}%

\newtcolorbox{readings}{%
    before skip=10pt,after skip=10pt,
    left=0.2cm,right=0.2cm,top=0cm,
    toptitle=0.2cm,bottomtitle=0cm,
    breakable,
    toprule at break=0.2cm,
    sharp corners,
    colback=green!30,
    coltitle=black,
    colframe=green!30,
    fonttitle=\bfseries,
    title={Readings},
    parbox=false,
    halign=justify,
}
\newtcolorbox{oreadings}{%
    before skip=10pt,after skip=10pt,
    left=0.2cm,right=0.2cm,top=0cm,
    toptitle=0.2cm,bottomtitle=0cm,
    breakable,
    toprule at break=0.2cm,
    sharp corners,
    colback=green!20,
    coltitle=black,
    colframe=green!20,
    fonttitle=\bfseries,
    title={Optional Readings},
    parbox=false,
    halign=justify,
}

\newtcolorbox{thmb}{%
    before skip=10pt,after skip=10pt,
    left=0.2cm,right=0.2cm, top=0cm,
    toptitle=0cm,bottomtitle=0cm,
    breakable,
    toprule at break=0.2cm,
    sharp corners,
    colback=gray!10,
    coltitle=black,
    colframe=gray!10,
    fonttitle=\bfseries,
    title={},
    parbox=false,
    halign=justify,
}

\makeatletter
\NewDocumentCommand{\embeq}{m}{%
  \leavevmode\hfill\refstepcounter{equation}\textup{\tagform@{\theequation}}\label{#1}%
}
\makeatother

\makeatletter
\NewDocumentCommand{\algeq}{m}{%
  \leavevmode\Comment*[r]{\refstepcounter{equation}\textup{\tagform@{\theequation}}\label{#1}}%
}
\makeatother

\setlist[enumerate]{noitemsep,topsep=2pt,leftmargin=16pt}
\setlist[itemize]{noitemsep,topsep=2pt}

\usepackage{import}
\usepackage{xifthen}
\usepackage{transparent}
\NewDocumentCommand{\incfig}{om}{%
    \IfValueTF{#1}{%
        \def\svgwidth{#1}%
    }{%
        \def\svgwidth{\columnwidth}%
    }%
    \centering\import{./figures/}{#2.pdf_tex}%
}

%% file: commands.tex
\newcommand{\course}{\ifthenelse{\boolean{manuscript}}{manuscript}{course}\xspace}

\NewDocumentCommand{\norm}{sm}{\IfBooleanTF{#1}{\|#2\|}{\left\| #2 \right\|}}

\DeclarePairedDelimiter\parentheses{(}{)}
\DeclarePairedDelimiter\brackets{[}{]}
\DeclarePairedDelimiter\braces{\{}{\}}

\newcommand{\R}{\mathbb{R}}

\renewcommand{\vec}[1]{\boldsymbol{#1}}
\newcommand{\mat}[1]{\boldsymbol{#1}}

\newcommand{\spa}[1]{\mathcal{#1}}

\NewDocumentCommand{\Ind}{m}{\mathds{1}\{{#1}\}}
\NewDocumentCommand{\fnPr}{}{\mathbb{P}}
\RenewDocumentCommand{\Pr}{om}{\fnPr\IfValueT{#1}{_{#1}}\parentheses*{#2}}

\NewDocumentCommand{\Entropy}{mo}{\mathrm{H}\IfValueTF{#2}{\!\left[#1\ \middle|\ #2\right]}{\brackets*{#1}}}
\RenewDocumentCommand{\H}{mo}{\Entropy{#1}[#2]}
\NewDocumentCommand{\Hsm}{mo}{\mathrm{H}\IfValueTF{#2}{[#1 \mid #2]}{\brackets{#1}}}
\NewDocumentCommand{\I}{mmo}{\mathrm{I}\IfValueTF{#3}{\!\left(#1;#2\ \middle|\ #3\right)}{\parentheses*{#1; #2}}}
\NewDocumentCommand{\Ism}{mmo}{\mathrm{I}\IfValueTF{#3}{(#1;#2 \mid #3)}{\parentheses{#1; #2}}}

\NewDocumentCommand{\E}{somo}{\ensuremath{\mathbb{E}\IfValueT{#2}{_{#2}}{} \IfBooleanTF{#1}{#3}{\IfValueTF{#4}{\!\left[#3\ \middle|\ #4\right]}{\brackets*{#3}}}}}
\NewDocumentCommand{\Var}{somo}{\mathrm{Var}\IfValueT{#2}{_{#2}}{} \IfBooleanTF{#1}{#3}{\IfValueTF{#4}{\!\left[#3\ \middle|\ #4\right]}{\brackets*{#3}}}}
\NewDocumentCommand{\Cov}{som}{\mathrm{Cov}\IfValueT{#2}{_{#2}}{} \IfBooleanTF{#1}{#3}{\brackets*{#3}}}
\NewDocumentCommand{\Cor}{som}{\mathrm{Cor}\IfValueT{#2}{_{#2}}{} \IfBooleanTF{#1}{#3}{\brackets*{#3}}}

\NewDocumentCommand{\fnv}{oo}{v\IfValueT{#2}{_{#2}}\IfValueT{#1}{^{#1}}}
\NewDocumentCommand{\Value}{somo}{\IfBooleanTF{#1}{\fnv[\star][#4]\parentheses{#3}}{\fnv[#2][#4]\parentheses{#3}}}
\RenewDocumentCommand{\v}{somo}{\Value#1[#2]{#3}[#4]}
\NewDocumentCommand{\fnq}{oo}{q\IfValueT{#2}{_{#2}}\IfValueT{#1}{^{#1}}}
\NewDocumentCommand{\q}{sommo}{\IfBooleanTF{#1}{\fnq[\star][#5]\parentheses{#3,#4}}{\fnq[#2][#5]\parentheses{#3,#4}}}
\NewDocumentCommand{\fnV}{oo}{V\IfValueT{#2}{_{#2}}\IfValueT{#1}{^{#1}}}
\NewDocumentCommand{\V}{somo}{\IfBooleanTF{#1}{\fnV[\star][#4]\parentheses{#3}}{\fnV[#2][#4]\parentheses{#3}}}
\NewDocumentCommand{\fnQ}{oo}{Q\IfValueT{#2}{_{#2}}\IfValueT{#1}{^{#1}}}
\NewDocumentCommand{\Q}{sommo}{\IfBooleanTF{#1}{\fnQ[\star][#5]\parentheses{#3,#4}}{\fnQ[#2][#5]\parentheses{#3,#4}}}
\NewDocumentCommand{\fna}{oo}{a\IfValueT{#2}{_{#2}}\IfValueT{#1}{^{#1}}}
\RenewDocumentCommand{\a}{sommo}{\IfBooleanTF{#1}{\fna[\star][#5]\parentheses{#3,#4}}{\fna[#2][#5]\parentheses{#3,#4}}}
\NewDocumentCommand{\fnA}{oo}{A\IfValueT{#2}{_{#2}}\IfValueT{#1}{^{#1}}}
\NewDocumentCommand{\fnAhat}{oo}{\hat{A}\IfValueT{#2}{_{#2}}\IfValueT{#1}{^{#1}}}
\NewDocumentCommand{\A}{sommo}{\IfBooleanTF{#1}{\fnA[\star][#5]\parentheses{#3,#4}}{\fnA[#2][#5]\parentheses{#3,#4}}}
\NewDocumentCommand{\fnj}{o}{J\IfValueT{#1}{_{#1}}}
\RenewDocumentCommand{\j}{mo}{\fnj[#2]\parentheses{#1}}
\NewDocumentCommand{\fnJ}{o}{\widehat{J}\IfValueT{#1}{_{#1}}}
\NewDocumentCommand{\J}{mo}{\fnJ[#2]\parentheses{#1}}

\NewDocumentCommand{\grad}{e_}{\boldsymbol{\nabla}\IfValueT{#1}{_{\!\!#1}\,}}

\NewDocumentCommand{\transpose}{m}{#1^\top}

\NewDocumentCommand{\diag}{om}{\mathrm{diag}\IfValueT{#1}{_{#1}}{}\braces{#2}}

\NewDocumentCommand{\N}{somm}{\mathcal{N}\IfBooleanTF{#1}{\left(}{(}\IfValueT{#2}{#2;}{} #3, #4\IfBooleanTF{#1}{\right)}{)}}
\NewDocumentCommand{\SN}{o}{\mathcal{N}(\IfValueT{#1}{#1;}{} \vzero, \mI)}
\NewDocumentCommand{\uSN}{o}{\mathcal{N}(\IfValueT{#1}{#1;}{} 0, 1)}
\NewDocumentCommand{\KF}{ommmmm}{\mathcal{KF}(\IfValueT{#1}{#1;}{} #2, #3, #4, #5, #6)}
\NewDocumentCommand{\GP}{omm}{\mathcal{GP}(\IfValueT{#1}{#1;}{} #2, #3)}
\NewDocumentCommand{\Unif}{om}{\mathrm{Unif}(\IfValueT{#1}{#1;}{} #2)}
\NewDocumentCommand{\Bern}{om}{\mathrm{Bern}(\IfValueT{#1}{#1;}{} #2)}
\NewDocumentCommand{\Bin}{omm}{\mathrm{Bin}(\IfValueT{#1}{#1;}{} #2, #3)}
\NewDocumentCommand{\Beta}{omm}{\mathrm{Beta}(\IfValueT{#1}{#1;}{} #2, #3)}
\NewDocumentCommand{\Laplace}{omm}{\mathrm{Laplace}(\IfValueT{#1}{#1;}{} #2, #3)}
\NewDocumentCommand{\Cauchy}{omm}{\mathrm{Cauchy}(\IfValueT{#1}{#1;}{} #2, #3)}
\NewDocumentCommand{\GammaDistr}{omm}{\mathrm{Gamma}(\IfValueT{#1}{#1;}{} #2, #3)}
\NewDocumentCommand{\Pareto}{omm}{\mathrm{Pareto}(\IfValueT{#1}{#1;}{} #2, #3)}

\newcommand{\vzero}{\vec{0}}

\newcommand{\vx}{\vec{x}}

\newcommand{\vy}{\vec{y}}

\newcommand{\mI}{\mat{I}}

\newcommand{\spY}{\spa{Y}}

\newcounter{qualcellnum}

\newcommand{\qualexgrid}[3]{
  \setcounter{qualcellnum}{0}%
  \begingroup
  \setlength{\fboxsep}{6pt}%
  \begin{tcolorbox}[%
    colback=dit_lightblue!25, colframe=dit_lightblue!25,
    arc=4pt, boxrule=0pt,
    width=\textwidth, left=2pt, right=2pt, top=2pt, bottom=2pt]
    \centering
    \foreach \qualidx in {#3} {%
      \stepcounter{qualcellnum}%
      \begin{minipage}[t]{0.235\textwidth}%
        \begin{tcolorbox}[%
          colback=dit_lightblue!75, colframe=dit_lightblue!25,
          arc=4pt, boxrule=0pt,
          width=\linewidth, left=2pt, right=2pt, top=2pt, bottom=2pt]
          \centering
          {\scriptsize%
          \setlength{\baselineskip}{4pt}%
            \csvreader[separator=semicolon]{#1}%
              {idx=\qualidxcol,txt=\qualtxtcol}%
              {\ifthenelse{\equal{\qualidxcol}{\qualidx}}{\qualtxtcol}{}}%
          \par}%
          \smallskip
          \includegraphics[width=\linewidth, keepaspectratio]{#2/img_\qualidx.jpg}%
        \end{tcolorbox}%
      \end{minipage}%
      \pgfmathtruncatemacro{\qualmod}{Mod(\value{qualcellnum},4)}%
      \ifnum\qualmod=0
        \ifnum\value{qualcellnum}<16 \\[4pt]\fi
      \else
        \hfill
      \fi
    }%
    \par\smallskip
    \incfig{colorbar}%
  \end{tcolorbox}%
  \endgroup
}

\newcommand{\qualexrow}[3]{
  \begingroup
  \setlength{\fboxsep}{6pt}%
  \begin{tcolorbox}[%
    colback=dit_lightblue!25, colframe=dit_lightblue!25,
    arc=4pt, boxrule=0pt,
    width=\textwidth, left=2pt, right=2pt, top=2pt, bottom=2pt]
    \centering
    \foreach \qualidx [count=\qualn from 1] in {#3} {%
      \begin{minipage}[t]{0.185\textwidth}%
        \begin{tcolorbox}[%
          colback=dit_lightblue!75, colframe=dit_lightblue!25,
          arc=4pt, boxrule=0pt,
          width=\linewidth, left=2pt, right=2pt, top=2pt, bottom=2pt]
          \centering
          {\scriptsize%
          \setlength{\baselineskip}{4pt}%
            \csvreader[separator=semicolon]{#1}%
              {idx=\qualidxcol,txt=\qualtxtcol}%
              {\ifthenelse{\equal{\qualidxcol}{\qualidx}}{\qualtxtcol}{}}%
          \par}%
          \smallskip
          \includegraphics[width=\linewidth, keepaspectratio]{#2/img_\qualidx.jpg}%
        \end{tcolorbox}%
      \end{minipage}%
      \ifnum\qualn<5 \hfill\fi
    }%
    \par\smallskip
    \incfig{colorbar}%
  \end{tcolorbox}%
  \endgroup
}

\newcommand{\qualexrowtwo}[4]{
  \begingroup
  \setlength{\fboxsep}{6pt}%
  \begin{tcolorbox}[%
    colback=dit_lightblue!25, colframe=dit_lightblue!25,
    arc=4pt, boxrule=0pt,
    width=\textwidth, left=2pt, right=2pt, top=2pt, bottom=2pt]
    \centering
    \foreach \qualidx [count=\qualn from 1] in {#4} {%
      \begin{minipage}[t]{0.185\textwidth}%
        \begin{tcolorbox}[%
          colback=dit_lightblue!75, colframe=dit_lightblue!25,
          arc=4pt, boxrule=0pt,
          width=\linewidth, left=2pt, right=2pt, top=2pt, bottom=2pt]
          \centering
          {\scriptsize%
          \setlength{\baselineskip}{4pt}%
            \csvreader[separator=semicolon]{#1}%
              {idx=\qualidxcol,txt=\qualtxtcol}%
              {\ifthenelse{\equal{\qualidxcol}{\qualidx}}{\qualtxtcol}{}}%
          \par}%
          \smallskip
          \includegraphics[width=\linewidth, keepaspectratio]{#2/img_\qualidx.jpg}%
          \par\smallskip
          \includegraphics[width=\linewidth, keepaspectratio]{#3/img_\qualidx.jpg}%
          \par
        \end{tcolorbox}%
      \end{minipage}%
      \ifnum\qualn<5 \hfill\fi
    }%
    \par\smallskip
    \incfig{colorbar}%
  \end{tcolorbox}%
  \endgroup
}

%% file: sections/intro.tex
\section{Introduction}\label{sec:intro}
Pretrained text-to-image flow models are strong image generators, yet they work in only one direction: text in, image out. Diffusion established the paradigm for high-fidelity synthesis~\citep{ho_denoising_2020,rombach_high-resolution_2022}, and rectified-flow transformers further improved scalability, prompt fidelity, and compositional generation~\citep{peebles_scalable_2023,esser_scaling_2024,lipman_flow_2023,liu_flow_2022,labs_flux1_2025}. Their internal representations are deeply structured: attention localizes entities~\citep{chefer_attend-and-excite_2023,jin2025diffusion}, token interventions affect attribute binding~\citep{hertz_prompt--prompt_2022,bill_focus_2026,meral_conform_2023}, and benchmarks confirm increasingly rich object-relation modeling~\citep{huang_t2i-compbench_2023}. A backbone that has learned how words, objects, and spatial layouts co-vary should, in principle, also support the inverse direction: describing an image, answering questions about it, or sampling coherent image--text pairs.

Exploiting this latent capability is not straightforward. Most vision--language models (VLMs) are built in the opposite direction: a pretrained language model is given a visual encoder, making vision a conditioning signal for autoregressive text generation~\citep{alayrac_flamingo,li_blip_2022,liu2023visual,zhu2023minigpt,dai_instructblip_2023}. Unified generators that combine autoregressive text with diffusion-based images~\citep{zhou_transfusion_2024,ma_janusflow_2024} are similarly language-first. These designs discard the strong image prior the text-to-image backbone already encodes, and rebuilding it requires expensive joint pretraining. This raises a natural question: can a pretrained text-to-image flow model be made bidirectional without retraining from scratch?

We study this question in the \textbf{lightweight-upgrade regime}: starting from an SD3- or FLUX-style generator, preserving its pretrained image process, and adding text generation with minimal new parameters. This regime is both scientifically motivated, as a test of whether the visual-semantic structure already encoded by text-to-image models is sufficient for image-to-text, joint, and partial-text generation, and practically motivated: full multimodal pretraining is prohibitively expensive, whereas a low-cost recipe places follow-up experimentation on bidirectional rectified-flow models within reach of the broader research community. Prior unified diffusion models show that joint generation is possible, but operate in different regimes: end-to-end multimodal pretraining~\citep{bao_one_2023,zhou_transfusion_2024,nguyen_oneflow_2025}, shared discrete or flow-based representations~\citep{wang_fudoki_2025}, or substantial backbone retraining~\citep{li_dual_2025}.

\input{figures/example/fig}

We propose \emph{FullFlow}, a parameter-efficient recipe for upgrading pretrained rectified-flow text-to-image models into bidirectional vision--language generators. The key insight is to treat the two modalities asymmetrically but compatibly: images remain in the native continuous rectified-flow latent space, while text is handled by a lightweight discrete insertion process built on Edit Flows~\citep{havasi_edit_2025,gat_discrete_2024,campbell_generative_2024}. A single shared backbone processes both modalities under decoupled image and text timesteps $(t, \tau)$, turning inference into trajectory selection in a two-dimensional generative space: fixing one modality recovers the conditional tasks; evolving both yields unconditional image--text sampling; constraining a target span enables partial-text prediction for tasks such as VQA. For SD3, only LoRA adapters and added text heads are trained~\citep{hu_lora_2021}~--- ${\sim}5\%$ of parameters~--- and the full upgrade fits in $24$ hours on two commodity RTX~A5000 GPUs.

Empirically, FullFlow turns strong text-to-image generators into bidirectional vision--language models without sacrificing the original image prior. On SD3 against a matched Dual~Diffusion-LoRA~\citep{li_dual_2025} baseline at identical trainable-parameter count and LoRA rank, FullFlow improves text-to-image generation FID from $62.7$ to $29.3$ and image captioning CIDEr from $2.0$ to $13.4$ at matched gradient steps, with captioning CIDEr further reaching $99.4$ at matched wall-clock time, at ${\sim}8\times$ throughput and less than half the peak VRAM. The same recipe transfers to FLUX.1-dev, enables joint image--text sampling, and supports downstream VQA through the same insertion-based text interface.

FullFlow makes three contributions:
\begin{enumerate}[
    label=\textbf{C\arabic*.},
    leftmargin=*,
    itemsep=1pt,
    topsep=1pt,
    parsep=0pt,
    partopsep=0pt
]
    \item Converting a pretrained text-to-image rectified-flow model into a bidirectional vision--language generator by training only LoRA adapters and lightweight text heads, while preserving the pretrained image prior.

    \item Keeping images in the native continuous flow, adding a discrete insertion process for text, and decoupling image and text timesteps so that conditional, joint, and partial-text generation become different trajectories in a single $(t,\tau)$ space.

    \item On SD3 and FLUX.1-dev, outperforming a matched Dual~Diffusion-LoRA baseline on both text$\rightarrow$image retention and image$\rightarrow$text quality and supporting downstream VQA, while training on two commodity GPUs.
\end{enumerate}

%% file: figures/example/fig.tex
\newcommand{\subfig}[4][\textwidth]{%
    \begin{tabular}{@{}p{#1}@{}} 
        \parbox[c][0.38\textwidth][t]{#1}{#2}\\ 
        \includegraphics[width=#1, keepaspectratio]{#3} 
    \end{tabular}%
}

\begin{figure*}[t]
    \centering

        
        \begin{subfigure}[b]{0.24\textwidth}
            \begin{tcolorbox}[colback=dit_lightblue, colframe=dit_lightblue, arc=4pt, boxrule=0pt, width=\textwidth, left=1pt, right=1pt, top=1pt, bottom=0pt]
                \centering
                \subfig{%
                    \footnotesize%
                    \setlength{\baselineskip}{4pt}%
                    \textcolor[rgb]{0.001,0.000,0.014}{A} \textcolor[rgb]{0.001,0.000,0.014}{tiny} \textcolor[rgb]{0.001,0.000,0.014}{dragon} \textcolor[rgb]{0.001,0.000,0.014}{curled} \textcolor[rgb]{0.001,0.000,0.014}{up} \textcolor[rgb]{0.001,0.000,0.014}{inside} \textcolor[rgb]{0.001,0.000,0.014}{a} \textcolor[rgb]{0.001,0.000,0.014}{ceramic} \textcolor[rgb]{0.001,0.000,0.014}{coffee} \textcolor[rgb]{0.001,0.000,0.014}{cup} \textcolor[rgb]{0.001,0.000,0.014}{on} \textcolor[rgb]{0.001,0.000,0.014}{a} \textcolor[rgb]{0.001,0.000,0.014}{desk,} \textcolor[rgb]{0.001,0.000,0.014}{warm} \textcolor[rgb]{0.001,0.000,0.014}{morning} \textcolor[rgb]{0.001,0.000,0.014}{light.}%
                }{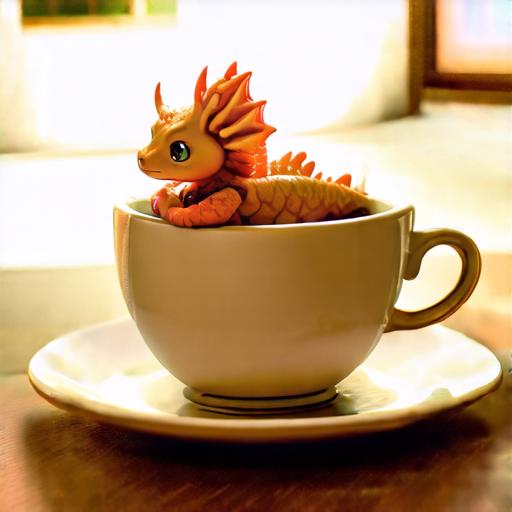}{\footnotesize$\vx \sim p_\theta(\cdot \mid \vy)$}
            \end{tcolorbox}
            \vspace{-3mm}
            \caption{\footnotesize Text-to-Image}
        \end{subfigure}\hfill
        \begin{subfigure}[b]{0.24\textwidth}
            \begin{tcolorbox}[colback=dit_lightblue, colframe=dit_lightblue, arc=4pt, boxrule=0pt, width=\textwidth, left=1pt, right=1pt, top=1pt, bottom=1pt]
            \centering
                \subfig{%
                    \footnotesize%
                    \setlength{\baselineskip}{4pt}%
                    \textcolor[rgb]{0.385,0.079,0.433}{A} \textcolor[rgb]{0.541,0.135,0.415}{classic} \textcolor[rgb]{0.978,0.558,0.035}{red} \textcolor[rgb]{0.827,0.280,0.262}{car} \textcolor[rgb]{0.689,0.192,0.358}{is} \textcolor[rgb]{0.385,0.079,0.433}{parked} \textcolor[rgb]{0.541,0.135,0.415}{on} \textcolor[rgb]{0.385,0.079,0.433}{a} \textcolor[rgb]{0.689,0.192,0.358}{street} \textcolor[rgb]{0.541,0.135,0.415}{with} \textcolor[rgb]{0.827,0.280,0.262}{a} \textcolor[rgb]{0.923,0.399,0.155}{building} \textcolor[rgb]{0.541,0.135,0.415}{in} \textcolor[rgb]{0.827,0.280,0.262}{the} \textcolor[rgb]{0.689,0.192,0.358}{background.}%
                }{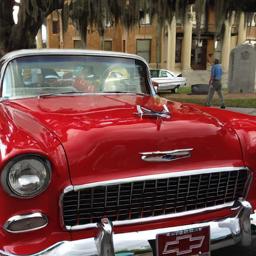}{\footnotesize$\vy \sim p_\theta(\cdot \mid \vx)$}
            \end{tcolorbox}
            \vspace{-3mm}
            \caption{\footnotesize Image-to-Text}
        \end{subfigure}\hfill
        \begin{subfigure}[b]{0.24\textwidth}
            \begin{tcolorbox}[colback=dit_lightblue, colframe=dit_lightblue, arc=4pt, boxrule=0pt, width=\textwidth, left=1pt, right=1pt, top=1pt, bottom=1pt]
            \centering
            \subfig{%
                \footnotesize%
                \setlength{\baselineskip}{4pt}%
                \textcolor[rgb]{0.817,0.271,0.270}{A} \textcolor[rgb]{0.689,0.192,0.358}{close-up} \textcolor[rgb]{0.753,0.226,0.319}{image} \textcolor[rgb]{0.890,0.348,0.198}{of} \textcolor[rgb]{0.874,0.327,0.217}{a} \textcolor[rgb]{0.012,0.009,0.063}{pineapple} \textcolor[rgb]{0.837,0.288,0.253}{with} \textcolor[rgb]{0.197,0.038,0.368}{a} \textcolor[rgb]{0.969,0.516,0.063}{leaf-like} \textcolor[rgb]{0.902,0.364,0.184}{appearance,} \textcolor[rgb]{0.976,0.544,0.044}{set} \textcolor[rgb]{0.278,0.042,0.414}{against} \textcolor[rgb]{0.099,0.046,0.244}{a} \textcolor[rgb]{0.497,0.119,0.424}{light} \textcolor[rgb]{0.566,0.144,0.408}{blue} \textcolor[rgb]{0.978,0.558,0.035}{grey} \textcolor[rgb]{0.712,0.204,0.344}{background.}%
            }{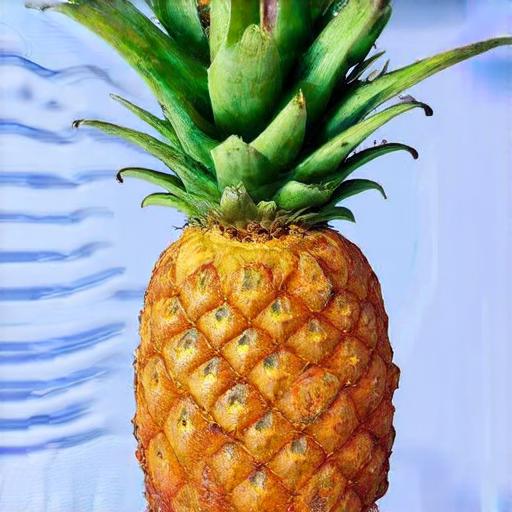}{\footnotesize$\vx, \vy \sim p_\theta(\cdot)$}
            \end{tcolorbox}
            \vspace{-3mm}
            \caption{\footnotesize Joint Text \& Image}
        \end{subfigure}\hfill
        \begin{subfigure}[b]{0.24\textwidth}
            \begin{tcolorbox}[colback=dit_lightblue, colframe=dit_lightblue, arc=4pt, boxrule=0pt, width=\textwidth, left=1pt, right=1pt, top=1pt, bottom=1pt]
            \centering
            \subfig{%
                \footnotesize%
                \setlength{\baselineskip}{4pt}%
                \textcolor[rgb]{0.001,0.000,0.014}{Question: What color is the child's hat?}\\
                \footnotesize\textcolor[rgb]{0.001,0.000,0.014}{Answer: }\textcolor[rgb]{0.978,0.558,0.035}{black.}
            }{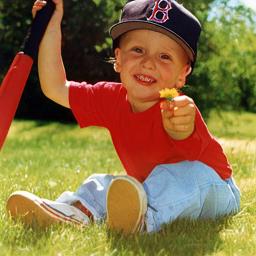}{\footnotesize$\vy_\text{ans} \sim p_\theta(\cdot \mid \vx, \vy_\text{quest})$}
            \end{tcolorbox}
            \vspace{-3mm}
            \caption{\footnotesize VQA}
        \end{subfigure}

        \medskip

        \incfig{colorbar}        
    
    \caption{Overview of the multimodal generation capabilities after finetuning Stable Diffusion 3. Beyond standard text-to-image synthesis, the model enables image-to-text generation, joint sampling over text and image and Visual Question Answering. The color bar visualizes text diffusion time.}
    \label{fig:overview}
    \vspace{-5mm}
\end{figure*}

%% file: sections/background.tex
\section{Background}\label{sec:background}
We frame both image and text generation in a single \emph{path-based} view: each modality moves from a simple base distribution $p_0$ to the data distribution $p_1$ along a path indexed by time, and is trained to predict local transport along that path.


We use $t \in [0,1]$ for image time and $\tau \in [0,1]$ for text time, with $0$ the base state and $1$ clean data. Images (or latents) are denoted by $\vx \in \mathcal{X} \subset \mathbb{R}^d$. Text is a variable-length token sequence $\vy \in \spY$ with
\begin{equation}
    \spY \coloneq \bigcup_{L \ge 0} V^L \times \{\tok{EOS}\},
\end{equation}
where $V$ is a fixed vocabulary and \tok{EOS} is the end-of-sequence token. We write $\varepsilon = (\tok{EOS})$ for the empty sequence; the discrete base distribution places all mass on $\varepsilon$.

\subsection{Continuous Flow Matching for Images}\label{sec:fm}
In continuous space, a model defines a time-dependent vector field $v_\theta(t,\vx)$ and generates samples by integrating
\begin{equation}
    \frac{d\vx_t}{dt} = v_\theta(t,\vx_t), \qquad t \in [0,1].
\end{equation}
Flow Matching~\citep{lipman_flow_2023} trains this vector field by specifying a path between a base sample $\vx_0$ and a target sample $\vx_1$, and then regressing onto the instantaneous velocity along that path. Concretely, given an interpolation $\vx_t = \psi_t(\vx_0,\vx_1)$, the target velocity is
\begin{equation}
    u(t,\vx_t \mid \vx_0,\vx_1) = \frac{\partial}{\partial t}\psi_t(\vx_0,\vx_1),
\end{equation}
and the model is trained with
\begin{equation}\label{eq:img_loss}
    \mathcal{L}_{\mathrm{img}}(\theta)
    =
    \E*[t \sim \Unif{[0,1]}]{}
    \E[\substack{\vx_0 \sim p_0\\\vx_1 \sim p_1}]{
        \norm{v_\theta(t,\vx_t) - u(t,\vx_t \mid \vx_0,\vx_1)}_2^2
    },
\end{equation}
where $\vx_t = \psi_t(\vx_0,\vx_1)$. In this work, we use rectified flow~\citep{liu_flow_2022}, which uses the linear interpolation
\begin{equation}\label{eq:linear_path}
    \vx_t = (1-t)\vx_0 + t\vx_1 \implies
    u(t,\vx_t \mid \vx_0,\vx_1) = \vx_1 - \vx_0.
\end{equation}

\subsection{Discrete Flow Matching for Text}\label{sec:fm_text}
For text, the state space $\spY$ is discrete and variable-length, so generation cannot be described by a Euclidean velocity field. Instead, local transport is specified by transition rates between sequences. We model these transitions by a continuous-time Markov chain (CTMC) with time-dependent generator $Q_\tau$, where $Q_\tau(\vy,\vy')$ is the instantaneous rate of jumping from sequence $\vy$ to sequence $\vy'$. Using the row-vector convention for distributions, the marginal law $p_\tau$ evolves according to
\begin{equation}
    \frac{d}{d\tau}p_\tau = p_\tau Q_\tau .
\end{equation}
Such jump processes provide a discrete analogue of flow-based generation~\citep{campbell_generative_2024,gat_discrete_2024}: the generator $Q_\tau$ plays the role of a discrete velocity field by specifying the local direction of probability transport.

Following Edit Flows~\citep{havasi_edit_2025}, we use a deletion-to-base forward process and learn the reverse insertion process. Given a clean sequence $\vy \sim p_1$, we construct a corrupted subsequence $\tilde{\vy}_\tau$ by independently retaining each non-\tok{EOS} token with probability $\tau$, while preserving order:
\begin{equation}
    \tilde{\vy}_\tau \sim q_\tau(\cdot \mid \vy), \qquad \tilde{\vy}_0 = \varepsilon, \qquad \tilde{\vy}_1 = \vy.
\end{equation}
We choose Edit Flows over mask-based discrete diffusion alternatives such as MDLM~\citep{sahoo_simple_2024} and SEDD~\citep{lou_discrete_2024} because deletion-based corruption yields valid token subsequences at every noise level, preserving compatibility with multiple frozen pretrained text encoders (\Cref{sec:peft}).
Thus, $\tau$ controls how much of the original sequence remains: at $\tau=0$ only the \tok{EOS} anchor is present, and at $\tau=1$ the full sequence is recovered. For example, for $\vy=\tok{A}\tok{black}\tok{cat}\tok{EOS}$, one possible corruption at $\tau=0.5$ is $\tilde{\vy}_\tau=\tok{cat}\tok{EOS}$, as illustrated in \Cref{fig:fm_discrete}.

The reverse process reconstructs $\vy$ by inserting missing tokens into the retained subsequence. For each retained token $\tilde{y}_\tau^i$ including \tok{EOS}, let $A_i$ be the deleted tokens immediately to its left; in the example, $A_0=\{\tok{A},\tok{black}\}$ before \tok{cat} and $A_1=\emptyset$ before \tok{EOS}. The model parameterises the time-reversed CTMC with a nonnegative insertion rate $\lambda_\theta^i(\tilde{\vy}_\tau,\tau)$ at each retained position and a categorical distribution $w_\theta^i(\cdot\mid\tilde{\vy}_\tau,\tau)$ over inserted token identities.

Training combines a count loss and a token loss: $\mathcal{L}_{\mathrm{txt}}(\theta) = \mathcal{L}_{\mathrm{cnt}}(\theta) + \mathcal{L}_{\mathrm{tok}}(\theta)$, defined as
\begin{align}\label{eq:txt_loss}
    \mathcal{L}_{\mathrm{cnt}}(\theta) &=
    \E*[\begin{subarray}{l}
        \scriptstyle\tau \sim \Unif{[0,1]}\\
        \scriptstyle\vy \sim p_1
    \end{subarray}]{}
    \E[\tilde{\vy}_\tau \sim q_\tau(\cdot \mid \vy)]{
        \sum_i \lambda_\theta^i(\tilde{\vy}_\tau, \tau) - |A_i| \cdot \log \lambda_\theta^i(\tilde{\vy}_\tau, \tau)
    },\\
    \mathcal{L}_{\mathrm{tok}}(\theta) &= 
    \E*[\begin{subarray}{l}
        \scriptstyle\tau \sim \Unif{[0,1]}\\
        \scriptstyle\vy \sim p_1
    \end{subarray}]{}
    \E[\tilde{\vy}_\tau \sim q_\tau(\cdot \mid \vy)]{
        -\sum_i \sum_{a \in A_i} \log w^i_\theta(a \mid \tilde{\vy}_\tau, \tau)
    }.
\end{align}
The first term is the Poisson negative log-likelihood for the number of insertions, up to constants independent of $\theta$, and the second term is the cross-entropy loss for the inserted token identities. Together, they define a reverse-time insertion process that transports $\varepsilon$ to natural text.

%% file: sections/method.tex
\section{Methodology}\label{sec:method}
We describe how to upgrade a pretrained rectified-flow text-to-image model into a bidirectional vision--language generator: the native-space joint formulation, parameter-efficient backbone extension, dual-timestep objective and stabilizers, and inference interface.

\subsection{Native-Space Multimodal Extension}\label{sec:native_space_extension}
Given a pretrained text-to-image model, we add text-generation capability while preserving the original image prior. We retain the pretrained continuous image process~\citep{liu_flow_2022,lipman_flow_2023} and introduce a discrete insertion process for text in the T5 tokenizer space, built on Edit Flows~\citep{havasi_edit_2025}. We choose Edit Flows over mask diffusion because their deletion-based forward process keeps every intermediate state $\tilde{\vy}_\tau$ a valid token subsequence that can be decoded and re-tokenized for SD3's and FLUX's auxiliary CLIP encoders, whereas mask tokens fall outside their frozen vocabularies.

Let $(\vx,\vy) \sim p_1$ denote paired image--text data. Images are corrupted by $q_t^x(\vx_t \mid \vx)$ and reconstructed with an image velocity head~$v_\theta$. Text is corrupted by $q_\tau^y(\tilde{\vy}_\tau \mid \vy)$ via token deletion and reconstructed with lightweight heads for insertion counts and token identities. A single shared backbone jointly processes $(\vx_t, t, \tilde{\vy}_\tau, \tau)$, making the two modalities mutually informative at every forward pass.

\begin{figure}[t]
    \centering
    \begin{minipage}[b]{0.39\textwidth}
      \vspace{0pt}
      \centering
      \resizebox{\textwidth}{!}{%
          \input{figures/overview/fig}%
      }
      \caption{\small Dual-timestep space with image time $t$ (x-axis) and text time $\tau$ (y-axis). All arrows, including the black ones, correspond to valid trajectories in the joint generative space. Colored trajectories highlight the canonical modes: \textcolor{dit_red}{\textbf{\textit{text$\to$image}}}, \textcolor{dit_blue}{\textbf{\textit{image$\to$text}}}, and \textcolor{dit_green}{\textbf{\textit{joint}}} generation.}
      \label{fig:timespace}
    \end{minipage}%
    \hfill%
    \begin{minipage}[b]{0.59\textwidth}
        \vspace{0pt}
        \centering
        \resizebox{\textwidth}{!}{%
            \input{figures/architecture/model}%
        }
        \caption{\small Overview of the multimodal DiT architecture adapted from Stable Diffusion 3. The model jointly processes noisy text tokens and image patches, along with their modality-specific timesteps $(\tau, t)$, to predict cross-modal flows. Blue blocks denote frozen pretrained components, yellow blocks indicate modules augmented with LoRA layers, red blocks mark newly added trainable components, and green blocks represent auxiliary operations.}
        \label{fig:mmdit_overview}
    \end{minipage}\vspace{-2mm}
\end{figure}

\subsection{Parameter-Efficient Extension of the Backbone}\label{sec:peft}
We freeze the backbone and train only: (i)~a conditioning pathway for the new text timestep~$\tau$, (ii)~text heads for insertion counts and token identities, and (iii)~LoRA adapters in transformer layers~\citep{hu_lora_2021}. \Cref{fig:mmdit_overview} illustrates the modifications to the SD3 MMDiT backbone. With LoRA rank $r=32$, this yields ${\sim}136$M trainable parameters for SD3 (${\approx}5\%$ of total) and ${\sim}298$M for FLUX.1 (${\approx}2.5\%$).

SD3 conditions on three frozen encoders (T5-XXL and two CLIP variants) using both token-level and pooled embeddings. We perform discrete text diffusion
entirely in the T5 tokenizer space, then bridge to the auxiliary encoders by decoding the partially deleted sequence $\tilde{\vy}_\tau$ to a string and
re-tokenizing it for each CLIP encoder, which is precisely possible because our chosen deletion produces valid token fragments rather than out-of-vocabulary mask symbols. Details are given in \Cref{app:multiencoder}.

\subsection{Dual-Timestep Training Objective}\label{sec:dual_time_obj}
Image and text are corrupted independently according to modality-specific timesteps $t$ and $\tau$, whose joint distribution is specified by a training schedule $\pi$.
Given the partially corrupted pair $(\vx_t,\tilde{\vy}_\tau)$, the model predicts both the image velocity and the missing text insertions:
\begin{align}\label{eq:joint_objective}
    \mathcal{L}(\theta)
    =
    \E*[(\vx,\vy) \sim p_1]{}
    \E*[(t,\tau)\sim \pi]{}
    \E[\substack{\vx_t \sim q_t^x(\cdot \mid \vx)\\ \tilde{\vy}_\tau \sim q_\tau^y(\cdot \mid \vy)}]{
        \mathcal{L}_{\mathrm{img}}(\theta;\, \vx_t, t, \tilde{\vy}_\tau, \tau)
        +
        \lambda_{\mathrm{txt}}\,
        \mathcal{L}_{\mathrm{txt}}(\theta;\, \tilde{\vy}_\tau, \tau, \vx_t, t)
    },
\end{align}
where $\mathcal{L}_{\mathrm{img}}$ and $\mathcal{L}_{\mathrm{txt}}$ are the rectified-flow and insertion losses from \Cref{eq:img_loss,eq:txt_loss}, and $\lambda_{\mathrm{txt}}$ controls the relative contribution of the text objective.

\tinytit{Time coupling.}
The schedule $\pi(t,\tau)$ determines which jointly corrupted states are seen during training. The \emph{alternating-clean} baseline
\begin{align}\label{eq:pi_ac}
    \pi_{\mathrm{ac}}:\quad
    \text{w.p. } \tfrac{1}{2},\; t \sim \Unif{[0,1]},\ \tau = 1;\qquad
    \text{w.p. } \tfrac{1}{2},\; t = 1,\ \tau \sim \Unif{[0,1]},
\end{align}
matches Dual Diffusion~\citep{li_dual_2025} but supervises only the two endpoint-conditioned tasks, leaving the model without signal on intermediate joint states; this can push the two conditioning pathways toward orthogonal representations that impede cross-modal alignment. Our default
\begin{align}\label{eq:pi_ind}
    \pi_{\mathrm{ind}}:\quad (t,\tau) \sim \Unif{[0,1]^2},
\end{align}
exposes the model to all combinations of partial image and partial text, directly training cross-modal correspondence across the full $(t,\tau)$ space.

\subsection{Stabilization}\label{sec:stabilizers}
Adding text generation to a pretrained image model creates two challenges: gradient scale mismatch, since the image branch is near convergence while text heads are randomly initialized; and image-prior degradation from fine-tuning outside the original data distribution. We address each with a stabilizer.

\tinytit{Adaptive gradient balancing.}
The two objectives induce gradients at different scales: the image branch is pretrained and uses an MSE velocity loss over high-dimensional latents, while the newly initialized text heads use count and cross-entropy losses with larger output-level gradients (\Cref{app:gradnorm}). Text updates can therefore dominate shared-parameter training early on. We periodically estimate the relative gradient scale on shared parameters and update $\lambda_{\mathrm{txt}}$ by EMA:
\begin{align}\label{eq:grad_ratio}
    r &= \frac{\norm{\grad_{\theta_{\mathrm{sh}}}\mathcal{L}_{\mathrm{img}}}_2}
              {\norm{\grad_{\theta_{\mathrm{sh}}}\mathcal{L}_{\mathrm{txt}}}_2 + \epsilon},
    &
    \lambda_{\mathrm{txt}} &\leftarrow \beta \lambda_{\mathrm{txt}} + (1-\beta)\, r,
\end{align}
with $\beta=0.99$ and $\epsilon=10^{-8}$, adding negligible overhead ($<0.1\%$ runtime) and removing the need to hand-tune a fixed weight per backbone.

\tinytit{Teacher matching.}
To preserve the pretrained image prior when fine-tuning on data outside the original pretraining distribution, we replace the velocity regression with a \emph{teacher-matching} objective:
\begin{align}\label{eq:img_loss_teacher}
    \mathcal{L}_{\mathrm{img}}(\theta) =
    \norm{
        \mathrm{sg}\!\left[v_{\mathrm{base}}(\vx_t, t, \bar{\vy})\right]
        -
        v_\theta(\vx_t, t, \tilde{\vy}_\tau, \tau)
    }_2^2,
\end{align}
where $\mathrm{sg}[\cdot]$ stops gradients and $\bar{\vy}$ is the teacher conditioning (clean text $\bar{\vy}=\vy$ or same-noise text $\bar{\vy}=\tilde{\vy}_\tau$). Crucially, the teacher $v_{\mathrm{base}}$ is obtained by simply disabling the LoRA layers of the same backbone, so no additional frozen copy needs to be held in memory.


\subsection{Inference Modes}\label{sec:vqa}

\tinytit{Trajectory selection in the time space.}
Decoupling $t$ and $\tau$ turns inference into trajectory selection in a two-dimensional generative space; see \Cref{fig:timespace}. Fixing $\tau=1$ recovers text$\rightarrow$image; fixing $t=1$ gives image$\rightarrow$text; evolving both jointly yields unconditional image--text sampling. Any trajectory through the $(t,\tau)$ square is valid, enabling flexible co-denoising without architectural change.

\tinytit{Partial-Text Generation for Downstream Tasks} 
Restricting the insertion process to a target span while fixing all other tokens gives a simple, decoder-free interface for $(\text{image},\text{text})\!\rightarrow\!\text{text}$ tasks: for VQA, we condition on the image and question and generate only the answer span, extending the same backbone and inference procedure to structured understanding without any additional components.

%% file: figures/overview/fig.tex
\begin{tikzpicture}[x=2.9cm,y=3.0cm,>=Latex]
    
    \tikzset{
      cell/.style={
        draw, rounded corners=2pt,
        fill=black!2,
        inner sep=2pt,
        minimum width=1.9cm,
        minimum height=2.0cm
      },
      img/.style={draw, inner sep=0pt},
      caption/.style={font=\small, align=center, text width=1.9cm},
      flow/.style={->, line width=2.0pt}
    }
    
    \newcommand{\mmcell}[4]{
      \node[cell] (#1) at #2 {};
      \node[img] (#1img) at ($(#1.north)+(0,-0.75cm)$) {%
        \includegraphics[width=1.3cm,height=1.3cm]{#3}%
      };
      \node[caption] (#1txt) at ($(#1.south)+(0,0.30cm)$) {#4};
    }
    
    \mmcell{g00}{(0,0)}{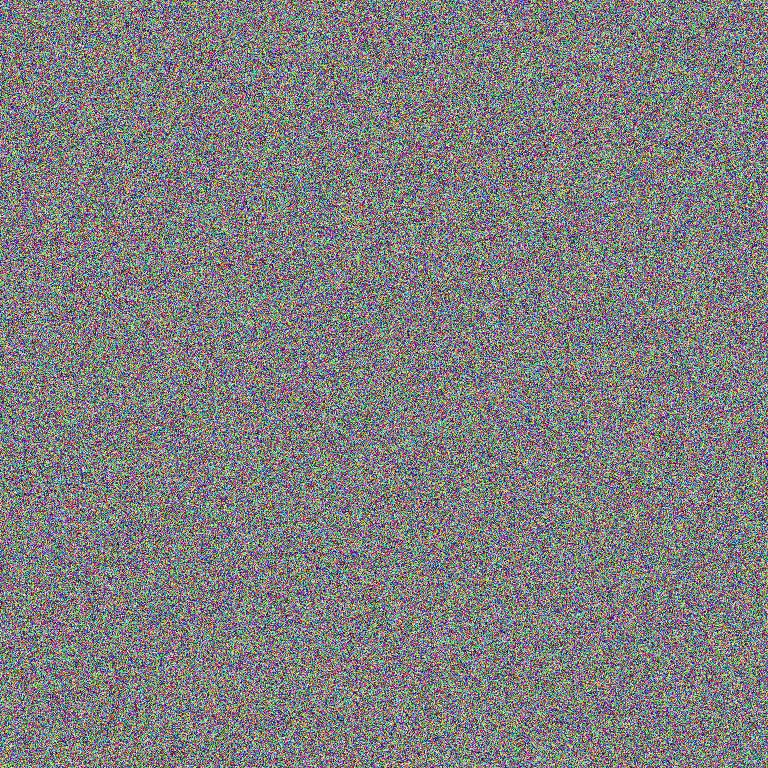}{$\varepsilon$}
    \mmcell{g_00_10}{(1,0)}{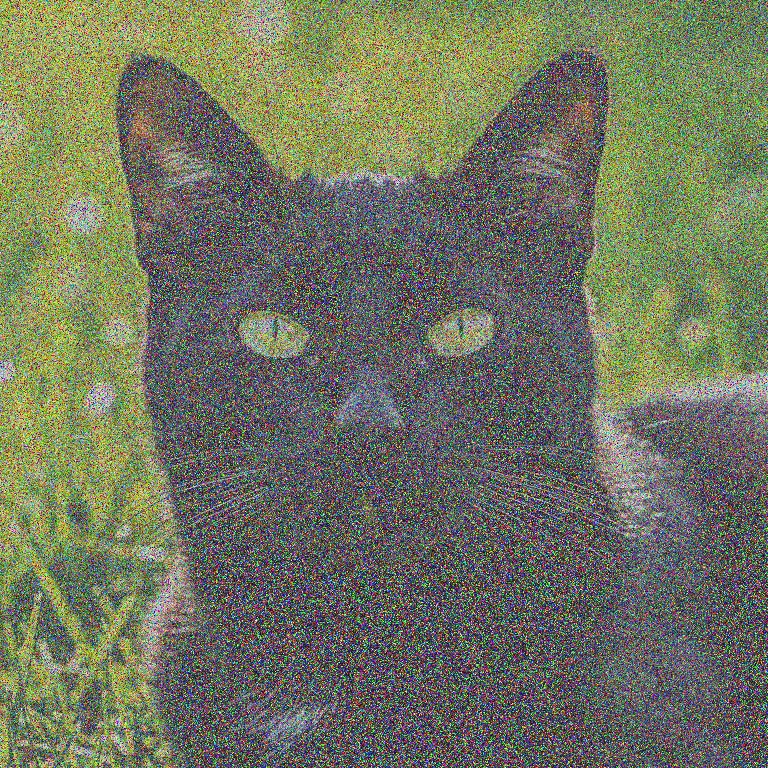}{$\varepsilon$}
    \mmcell{g10}{(2,0)}{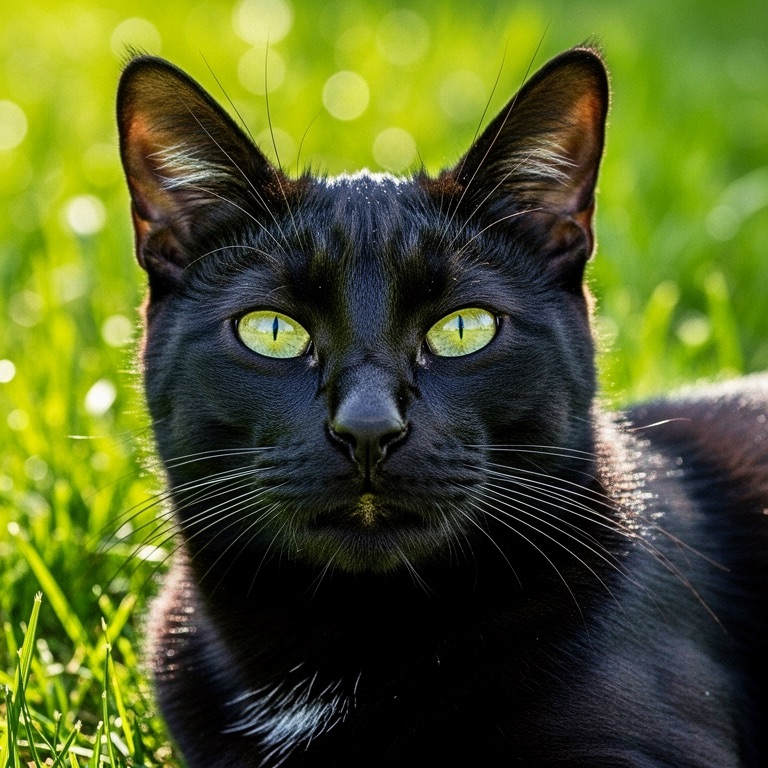}{$\varepsilon$}
    
    \mmcell{g_00_01}{(0,1)}{figures/overview/noise.jpeg}{``A cat''}
    \mmcell{gmid}{(1,1)}{figures/overview/cat_noisy.jpeg}{``A cat''}
    \mmcell{g_10_11}{(2,1)}{figures/overview/cat_clean.jpeg}{``A cat''}
    
    \mmcell{g01}{(0,2)}{figures/overview/noise.jpeg}{\textit{``A black cat''}}
    \mmcell{g_01_11}{(1,2)}{figures/overview/cat_noisy.jpeg}{``A black cat''}
    \mmcell{g11}{(2,2)}{figures/overview/cat_clean.jpeg}{``A black cat''}
    
    \coordinate (x0) at ($(g00.south west)+(0,-0.1)$);
    \coordinate (x1) at ($(g10.south east)+(0,-0.1)$);
    \draw[->, line width=1.5pt] (x0) -- (x1)
      node[midway, below=2pt]
      {Image time $t_{\text{img}}$};
    \node[below=2pt] at ($(x0)!0.05!(x1)$) {$t_{\text{img}}=0$};
    \node[below=2pt] at ($(x0)!0.95!(x1)$) {$t_{\text{img}}=1$};
    
    \coordinate (y0) at ($(g00.south west)+(-0.1,0)$);
    \coordinate (y1) at ($(g01.north west)+(-0.1,0.0)$);
    \draw[->, line width=1.5pt] (y0) -- (y1)
      node[midway, left=8pt, align=center, rotate=90, anchor=center]
      {Text time $\tau_{\text{txt}}$};
    
    \node[left=6pt, rotate=90] at ($(y0)!0.1!(y1)$) {$\tau_{\text{txt}}=0$};
    \node[left=6pt, rotate=90] at ($(y0)!0.95!(y1)$) {$\tau_{\text{txt}}=1$};
    
    \draw[flow] (g00.east) -- (g_00_10.west);
    \draw[flow] (g_00_10.east) -- (g10.west);
    
    \draw[flow] (g_00_01.east) -- (gmid.west);
    \draw[flow] (gmid.east) -- (g_10_11.west);
    
    \draw[flow,color=dit_red] (g01.east) -- (g_01_11.west);
    \draw[flow,color=dit_red] (g_01_11.east) -- (g11.west);
    
    \draw[flow] (g00.north) -- (g_00_01.south);
    \draw[flow] (g_00_01.north) -- (g01.south);
    
    \draw[flow] (g_00_10.north) -- (gmid.south);
    \draw[flow] (gmid.north) -- (g_01_11.south);
    
    \draw[flow,color=dit_blue] (g10.north) -- (g_10_11.south);
    \draw[flow,color=dit_blue] (g_10_11.north) -- (g11.south);
    
    \draw[flow,color=dit_green] (g00.north east) -- (gmid.south west);
    \draw[flow,color=dit_green] (gmid.north east) -- (g11.south west);
    
    \draw[flow] (g_00_01.north east) -- (g_01_11.south west);
    \draw[flow] (g_00_10.north east) -- (g_10_11.south west);
\end{tikzpicture}

%% file: figures/architecture/model.tex
\begin{tikzpicture}[
        font=\fontsize{12}{14}\selectfont,
        every path/.style={line width=0.8pt},
        block/.style={draw, fill=white, rectangle, minimum width=3cm,rounded corners=0.2cm},
        trainableblock/.style={block,fill=dit_red!70,thick},
        lorablock/.style={block,fill=dit_yellow!70,thick},
        frozenblock/.style={block,fill=dit_lightblue,thick},
        fnblock/.style={block,fill=dit_lightgreen,thick},
        operation/.style={draw=black!70, fill=white, circle, minimum width=1.5em, thick},
        tensor/.style={draw=black!70, fill=dit_grey!20, circle, thick, dashed, minimum width=1.5em},
        input/.style={block,fill=dit_grey!50,thick},
        outlined arrow/.style={preaction={draw, white, line width=4pt}}
    ]

    \pgfdeclarelayer{arrowlayer}
    \pgfdeclarelayer{residuallayer}
    \pgfdeclarelayer{decorations}
    \pgfdeclarelayer{bg}
    \pgfsetlayers{bg,arrowlayer,residuallayer,main,decorations}

    \newcommand{\clipGcolor}{LimeGreen!40}
    \newcommand{\clipLcolor}{ForestGreen!40}
    \newcommand{\tfiveColor}{YellowGreen!40}

    \pgfdeclarelayer{arrowlayer}
    \pgfdeclarelayer{residuallayer}
    \pgfdeclarelayer{decorations}
    \pgfdeclarelayer{bg}
    \pgfsetlayers{bg,arrowlayer,residuallayer,main,decorations}
    
    \node [input] (text) {\textbf{Caption $\vy$}};

    \node [frozenblock, fill=\clipGcolor, below=0.5cm of text] (clipG) {\textbf{CLIP-G/14}};
    \node [frozenblock, fill=\clipLcolor, left=1.5cm of clipG.center] (clipL) {\textbf{CLIP-L/14}};
    \node [frozenblock, fill=\tfiveColor, right=1.5cm of clipG.center] (t5) {\textbf{T5 XXL\phantom{/}}};

    \coordinate [below=2.0cm of clipG] (conditioning);
    \coordinate [left=3.0cm of conditioning] (pool_conditioning);
    
    \coordinate (timestep_align) at ($(pool_conditioning)+(-3.6,0)$);
    \coordinate (timestep_top) at (timestep_align |- text.center);
    \node [input, right=3.6cm of text] (image) {\textbf{Image $\vx$}};

    \draw [fill=\tfiveColor] (conditioning) |-++ (1.5, 1) coordinate (conditioning_br) --++ (0, -1.5) coordinate (conditioning_tr) -| cycle;
    \draw [fill=\clipGcolor] (conditioning) --++ (0, 1) coordinate (conditioning_bm) --++ (-1.5, 0) coordinate (conditioning_bl) --++ (0, -1) -- cycle;
    \draw [fill=\clipLcolor] (conditioning) --++ (-1.5, 0) coordinate (conditioning_ml) --++ (0, -0.5) coordinate (conditioning_tl) -| cycle;

    \draw [->] (t5) to [out=270, in=0] ($(conditioning_br)!0.5!(conditioning_tr)+(-0.5, 0)$);
    \draw [->] (clipG) to [out=270, in=90] ($(conditioning_bl)!0.5!(conditioning_bm)-(0, 0.5)$);
    \draw [->] (clipL) to [out=270, in=180] ($(conditioning_ml)!0.5!(conditioning_tl)+(0.5, 0)$);

    \draw [fill=\clipGcolor] (pool_conditioning) -|++ (0.2, 1) -|++ (-0.4, -1) -- cycle;
    \draw [fill=\clipLcolor] (pool_conditioning) -|++ (0.2, -0.5) -|++ (-0.4, 0.5) -- cycle;

    \draw [->] (clipG) to [out=270, in=90] ($(pool_conditioning)+(0, 0.75)$);
    \draw [->] (clipL) to [out=270, in=0] ($(pool_conditioning)+(0, -0.25)$);


    \node [frozenblock, below=1cm of conditioning] (y_lin) {\textbf{Linear}};
    \node [frozenblock, below=1cm of pool_conditioning] (pool_mlp) {\textbf{MLP}};

    \begin{pgfonlayer}{arrowlayer}
        \draw [->] (text.south) to [out=270, in=90] (clipG);
        \draw [->] (text.south) to [out=270, in=90] (clipL);
        \draw [->] (text.south) to [out=270, in=90] (t5);

        \draw [->] (conditioning) -- (y_lin);
    \end{pgfonlayer}


    \node [fnblock, minimum width=5.4cm, below=1.5cm of timestep_top] (t_enc) {\textbf{Sinusoidal Encoding}};
    \node [frozenblock, minimum width=5.4cm, below=0.2cm of t_enc] (t_emb) {\textbf{MLP}};
    \coordinate (t_enc_left_in) at ($(t_enc.north west)!0.28!(t_enc.north east)$);
    \coordinate (t_enc_right_in) at ($(t_enc.north west)!0.72!(t_enc.north east)$);
    \coordinate (t_enc_left_out) at ($(t_enc.south west)!0.28!(t_enc.south east)$);
    \coordinate (t_enc_right_out) at ($(t_enc.south west)!0.72!(t_enc.south east)$);
    \coordinate (t_emb_left_in) at ($(t_emb.north west)!0.28!(t_emb.north east)$);
    \coordinate (t_emb_right_in) at ($(t_emb.north west)!0.72!(t_emb.north east)$);
    \coordinate (t_emb_left_out) at ($(t_emb.south west)!0.28!(t_emb.south east)$);
    \coordinate (t_emb_right_out) at ($(t_emb.south west)!0.72!(t_emb.south east)$);
    \coordinate (t_enc_top_mid) at ($(t_enc.north west)!0.5!(t_enc.north east)$);
    \node [input] (timestep_tau) at ($(t_enc_top_mid |- timestep_top)+(-1.3cm,0)$) {\textbf{Timestep $\tau$}};
    \node [input] (timestep_t) at ($(t_enc_top_mid |- timestep_top)+(1.85cm,0)$) {\textbf{Timestep $t$}};
    \node [operation] (delta_diff) at ($(t_emb_left_out)+(0,-0.75cm)$) {$\bm{-}$};
    \node [operation] (t_plus) at (t_emb_right_out |- pool_mlp.center) {$\bm{+}$};
    \node [operation, below=1.8cm of t_plus] (tau_update) {$\bm{+}$};
    \node [trainableblock, below=1.5cm of delta_diff] (delta_mlp) {\textbf{Linear}};
    \node [operation] (delta_mult) at (delta_mlp.center |- tau_update) {$\bm{\times}$};
    \node [fnblock, below=0.5cm of delta_mult] (delta_tanh) {\textbf{tanh}};
    \node [trainableblock, below=0.5cm of delta_tanh] (delta_scale) {\textbf{Scalar}};

    \coordinate (y_branch) at ($(pool_mlp.west)+(-1.0,0)$);
    \node [tensor] (tau_final) at ($(tau_update)+(4.2cm,0)$ |- attn1.center) {$\tau$};
    
    \begin{pgfonlayer}{arrowlayer}
        \draw [->] (pool_conditioning) -- (pool_mlp);
        \draw [->] (timestep_tau.south) -- (timestep_tau.south -| t_enc_left_in) -- (t_enc_left_in);
        \draw [->] (timestep_t.south) -- (timestep_t.south -| t_enc_right_in) -- (t_enc_right_in);
        \draw [->] (t_enc_left_out) -- (t_emb_left_in);
        \draw [->] (tau_update) -- (tau_final);
        \draw [->] (t_enc_right_out) -- (t_emb_right_in);
        \draw [->] (t_emb_left_out) -- (delta_diff);
        \draw [->] (delta_diff) -- (delta_mlp);
        \draw [->] (t_emb_right_out) -- (t_plus);
        \draw [->] (t_plus) -- (tau_update);
        \draw [->] (t_plus) -- (tau_update);
        \draw [->] (t_emb_right_out) |- (delta_diff);
        \draw [->] (pool_mlp.west) |- (t_plus.east);
        \draw [->] (delta_scale) -- (delta_tanh);
        \draw [->] (delta_tanh) -- (delta_mult);
        \draw [->] (delta_mlp.south) -- (delta_mult.north);
        \draw [->] (delta_mult) -- (tau_update);
    \end{pgfonlayer}

    \node [fnblock, below=1.5cm of image] (patching) {\textbf{Patching}};
    \node [frozenblock, below=0.2cm of patching] (patch_proj) {\textbf{Linear}};
    \node [operation, below=0.5cm of patch_proj] (image_plus_pos) {$\bm{+}$};
    \node [frozenblock, left=0.2cm of image_plus_pos,align=center] (pos_emb) {\textbf{Positional}\\\textbf{Embedding}};

    \begin{pgfonlayer}{arrowlayer}
        \draw [->] (image) -- (patching);
        \draw [->] (patching) -- (patch_proj);
        \draw [->] (patch_proj) -- (image_plus_pos);
        \draw [->] (pos_emb) -- (image_plus_pos);
    \end{pgfonlayer}

    \coordinate (y_stream) at ($(y_lin.south)+(0,-0.15cm)$);
    \coordinate (x_stream) at ($(image_plus_pos.south)+(0,-0.55cm)$);
    \coordinate (stream_mid) at ($(y_stream)!0.5!(x_stream)$);
    
    \node [lorablock, below=1.0cm of stream_mid, minimum width=5cm] (attn1) {\textbf{MM-Dit-Block 1}};
    \node [below=0.15cm of attn1, minimum width=5cm] {$\ldots$};
    \node [lorablock, below=0.6cm of attn1, minimum width=5cm] (attnN) {\textbf{MM-Dit-Block N}};

    \coordinate (attn1_mid_left) at ($(attn1.north west)!0.3!(attn1.north)$);
    \coordinate (attn1_mid_right) at ($(attn1.north east)!0.3!(attn1.north)$);

    \coordinate (attnN_mid_left) at ($(attnN.north west)!0.3!(attnN.north)$);
    \coordinate (attnN_mid_right) at ($(attnN.north east)!0.3!(attnN.north)$);

    \node [frozenblock, below=1.3cm of attnN_mid_right] (out_mod) {\textbf{Modulation}};
    \node [frozenblock, below=0.2cm of out_mod] (out_mlp) {\textbf{Linear}};
    \node [fnblock, below=0.2cm of out_mlp] (depatch) {\textbf{Unpatching}};
    
    \node [trainableblock, below=1.3cm of attnN_mid_left] (out_txt_mod) {\textbf{Modulation}};
    \node [trainableblock, below=0.2cm of out_txt_mod] (out_head) {\textbf{Linear}};
    \node [input, below=1.2cm of out_head] (output_txt) {\textbf{Text Output}};
    \node [input] (output) at (depatch.center |- output_txt) {\textbf{Image Output}};

    \coordinate (attn1_left) at ($(attn1.west)+(-1.4, 0)$);
    \coordinate (attn1_right) at ($(attn1.east)+(1.4, 0)$);

    \node [tensor, above=0.6cm of attn1_right] (t_final) {$t$};

    \begin{pgfonlayer}{arrowlayer}
        \draw [->] (image_plus_pos.south) -- (x_stream) -| (attn1_mid_right) -- (output);
        \draw [->] (y_lin.south) -- (y_stream) -| (attn1_mid_left) -- (output_txt);

        \draw [->] (tau_final) -| (attn1_left) -- (attn1);
        \draw [->] (attn1_left) |- (attnN);
        \draw [->] (attn1_left) |- (out_txt_mod);
        
        \draw [->, outlined arrow] (t_plus |- t_final) -- (t_final);
        \draw [->] (t_final.south) -| (attn1_right) -- (attn1);
        \draw [->] (attn1_right) |- (attnN);
        \draw [->] (attn1_right) |- (out_mod);
    \end{pgfonlayer}

    \begin{pgfonlayer}{bg}
        \draw [fill=dit_yellow!30] ($(attn1.north east)+(0.3, 0.3)$) -| ($(attnN.south west)+(-0.3, -0.3)$) -| cycle;
    \end{pgfonlayer}

\end{tikzpicture}

%% file: sections/related_work.tex
\section{Related Work}\label{sec:related}
\leavevmode

\begin{wraptable}{r}{0.3\textwidth}
    \centering
    \vspace{-4mm}
    \caption{\small Trainable-component variants trained for 100k steps; text-conditioning variants for 200k steps. Metrics averaged over 5k validation examples.}
    \setlength{\tabcolsep}{3pt}
    \renewcommand{\arraystretch}{0.92}
    \footnotesize
    \textbf{Trainable components}\\[0.1em]
    \begin{tabular}{l S[table-format=2.1]}
        \toprule
        \textbf{Variant} & {\textbf{CIDEr} $\uparrow$} \\
        \midrule LoRA only & \textbf{40.1} \\
        Refiner only & 0.0 \\
        LoRA + Refiner & 19.9 \\
        \bottomrule
    \end{tabular}\\[3mm]
    \textbf{Text-conditioning pathways}\\[0.1em]
    \begin{tabular}{l S[table-format=2.2] S[table-format=2.2]}
        \toprule
        \textbf{Metric} & {\textbf{Only T5}} & {\textbf{All Enc.}} \\
        \midrule
        CMMD $\downarrow$ & 0.50 & \textbf{0.11} \\
        FID $\downarrow$ & 44.55 & \textbf{34.25} \\
        \midrule CIDEr $\uparrow$ & 63.96 & \textbf{65.42} \\
        CLIP $\uparrow$ & \textbf{0.26} & 0.25 \\
        \bottomrule
    \end{tabular}
    \label{tab:a1_ablation}
    \vspace{-5mm}
\end{wraptable}

\tinytit{Unified multimodal generation.} Earlier diffusion-based systems such as Versatile Diffusion and UniDiffuser showed that one diffusion-style model can support text-to-image, image-to-text, and joint or marginal generation through multi-flow designs or modality-specific timesteps~\citep{xu_versatile_2024,bao_one_2023}. Later unified models explored hybrid or tokenized alternatives: Transfusion and Show-o combine diffusion with autoregressive text modeling, while Chameleon and Janus use tokenized visual representations in shared transformers~\citep{zhou_transfusion_2024,xie_show-o_2025,team_chameleon_2025,wu_janus_2024}. Recent non-autoregressive systems such as UniDisc and Muddit move closer to unified image--text diffusion, with Muddit also reusing a pretrained text-to-image backbone~\citep{swerdlow_unified_2025,shi_muddit_2025}. Most of these approaches nevertheless require large-scale pretraining or substantial retraining, whereas we target a lightweight adaptation of existing text-to-image flow models. Specialized vision--language models such as Flamingo, BLIP, and LLaVA achieve strong multimodal understanding through large-scale autoregressive pretraining~\citep{alayrac_flamingo,li_blip_2022,li_llava-next-interleave_2024}; our goal is bidirectional generation from a single lightweight-adapted model, not maximal understanding performance alone.\looseness=-1

\tinytit{Flow-based multimodal models.} OmniFlow extends rectified flow to any-to-any multimodal generation with an MMDiT architecture and modular modality streams~\citep{li_omniflow_2025}. OneFlow combines image Flow Matching with insertion-based Edit Flows for concurrent interleaved generation~\citep{nguyen_oneflow_2025}. FUDOKI formulates both modalities with discrete flow matching, while FlowTok maps text and images into a shared compact token space~\citep{wang_fudoki_2025,he_flowtok_2025}. JanusFlow harmonizes autoregressive text generation with rectified flow for images inside a unified transformer with decoupled visual encoders~\citep{ma_janusflow_2024}; Janus-Pro scales this approach further~\citep{chen_januspro_2025}. The closest prior work, Dual Diffusion, similarly couples continuous image flow with masked text diffusion in an MM-DiT backbone, but trains endpoint-conditioned tasks rather than the full $(t,\tau)$ joint space, uses mask diffusion that yields non-decodable text states, and requires large-scale end-to-end retraining~\citep{li_dual_2025}. In contrast, we keep images in the pretrained rectified-flow latent space and text in a deletion-based insertion process whose corrupted states remain valid strings---preserving frozen multi-encoder compatibility and allowing the backbone to recover its pretrained image prior exactly---while requiring only LoRA adapters and lightweight text heads trainable on two RTX A5000 GPUs.

\tinytit{Diffusion-language-model multimodal models.}
A complementary line starts from diffusion language models. LaViDa, Dimple, and LLaDA-V attach visual encoders to masked or discrete diffusion language backbones, mainly targeting multimodal understanding~\citep{li_lavida_2025,yu_dimple_2025,you_llada-v_2025}. More recent models such as MMaDA and LLaDA2.0-Uni extend this paradigm toward unified understanding, reasoning, image generation, and editing~\citep{yang_mmada_2025,ai_llada20-uni_2026}. Our direction is the reverse: we start from a strong pretrained text-to-image flow model and add a lightweight text-generation process.

%% file: sections/results.tex
\section{Experiments}\label{sec:results}


Our experiments address five questions: (i) which training ingredients are necessary for stable, resource-efficient adaptation (§\ref{sec:ablations}); (ii) how the adapted model compares to the closest prior work, Dual Diffusion (§\ref{sec:main_dd}); (iii) whether the same interface supports downstream understanding (§\ref{sec:downstream}); (iv) whether the recipe transfers across backbones (§\ref{sec:transfer})s; and (v) which new inference modes are enabled by the dual-timestep formulation (§\ref{sec:joint}).

\subsection{Experimental Setup}\label{sec:train_protocol}
We evaluate SD3-M~\citep{esser_scaling_2024} as the primary backbone and FLUX.1-dev~\citep{labs_flux1_2025} as a transfer setting. Models are fine-tuned on a fixed LAION-Aesthetic subset annotated with Moondream captions~\citep{schuhmann_laion-5b_2022}. SD3 experiments use $256\times256$ resolution on two RTX A5000 GPUs; FLUX.1-dev experiments use three GPUs with manual sharding. We evaluate text$\rightarrow$image with FID~\citep{heusel_gans_2018} and CMMD~\citep{jayasumana_rethinking_2024}, computed against the frozen base SD3 model on a 1.9k-prompt suite curated for this work that spans natural, compositional, and dense long-form prompts. We evaluate image$\rightarrow$text with CIDEr~\citep{vedantam_cider_2015} and BERTScore-F1~\citep{zhang_bertscore_2020} against held-out Moondream captions on a 5k validation split, joint generation with CLIP~\citep{radford_learning_2021} score, and downstream VQA using official benchmark protocols. Further dataset, prompt, and implementation details are provided in \Cref{app:training}.

\subsection{Adaptation Recipe Ablations}\label{sec:ablations}
We isolate the ingredients needed to stably uplift a pretrained text-to-image backbone under limited compute, ablating each in turn.

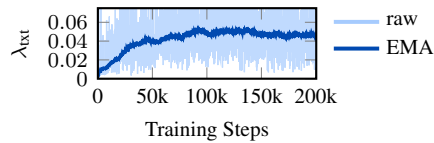
\begin{wrapfigure}{r}{0.4\textwidth}
  \vspace{-2mm}
  \centering
  \input{figures/gradient_norm/fig}\vspace{-5mm}
  \caption{\small Evolution of the adaptive text-loss weight $\lambda_\text{txt}$. Thin: raw ratio estimate; thick: EMA used in training.}
  \label{fig:ema_ratio}
  \vspace{-4mm}
\end{wrapfigure}

\tinytit{A1: Trainable parameters and text-conditioning pathways.}
SD3 uses three frozen text encoders with distinct tokenizers, whereas our discrete text process is defined in a single T5 token space. We therefore diffuse T5 tokens, decode the partially corrupted sequence, and re-tokenize it for the remaining encoders.

\emph{Text-conditioning.} Although diffusion is performed only in T5 token space, conditioning all frozen SD3 text encoders substantially improves text$\rightarrow$image retention while leaving image$\rightarrow$text quality nearly unchanged. We therefore retain the full conditioning stack.

\emph{Trainable parameters.} Full fine-tuning of SD3 is infeasible under our compute budget, so we restrict adaptation to lightweight modules and ask whether additional trainable capacity beyond LoRA is warranted. \Cref{tab:a1_ablation} compares LoRA-only updates to the shared backbone against an extended variant that adds a text refiner, a stack of single-stream DiT blocks placed before the text heads. The refiner does not improve image$\rightarrow$text quality at matched parameter budget, so we adopt LoRA-only adaptation, resulting in updating ${\sim}136$M parameters in total ($\sim$5\% of SD3-M).

\begin{wraptable}{r}{0.3\textwidth}
    \centering
    \vspace{-4mm}
    \caption{\small Image-prior preservation: direct flow matching ({FM}), same-noise teacher matching ({Teacher-SN}), and clean-text teacher matching ({Teacher-CT}).}
    \setlength{\tabcolsep}{4pt}
    \renewcommand{\arraystretch}{1.05}
    \small
    \begin{tabular}{l
        S[table-format=1.3]
        S[table-format=2.2]}
        \toprule
        \textbf{Variant} & \textbf{CMMD} $\downarrow$ & \textbf{FID} $\downarrow$ \\
        \midrule
        {FM} & 0.579 & 63.46 \\
        {Teacher-SN} & \textbf{0.084} & \textbf{30.38} \\
        {Teacher-CT} & 0.310 & 39.49 \\
        \bottomrule
    \end{tabular}
    \label{tab:tf_res}
    \vspace{-3mm}
\end{wraptable}

\tinytit{A2: Loss-balance stabilization.}
The image loss is an MSE regression objective while the text loss combines count NLL and token cross-entropy. In our setup, text gradients are roughly $25\times$ larger than image gradients on shared parameters, but this ratio is backbone-specific and difficult to estimate a priori, so a fixed weight $\lambda_{\text{txt}}$ would require a per-backbone hyperparameter sweep to avoid divergence or modality collapse. The EMA update from \Cref{sec:stabilizers} eliminates this hyperparameter by adapting $\lambda_{\text{txt}}$ from observed gradient ratios. As shown in \Cref{fig:ema_ratio}, the raw ratio is highly noisy, but its EMA settles on a stable value near $0.04$ for SD3 at negligible overhead ($<0.1\%$ runtime), yielding a value we found difficult to identify by inspection alone.

\tinytit{A3: Preserving the pretrained image prior.}\label{sec:TF}
Adding text generation can degrade the pretrained image prior. We compare three image objectives: direct flow matching against ground-truth latents (\textbf{FM}); same-noise teacher matching (\textbf{Teacher-SN}), where the frozen text-to-image teacher denoises the same noisy latent as the student conditioned on the corrupted text; and clean-text teacher matching (\textbf{Teacher-CT}), where the teacher conditions on the clean caption while the student conditions on the corrupted text. \Cref{tab:tf_res} shows that Teacher-SN gives the best fidelity--retention trade-off, reducing FID from $63.46$ to $30.38$ and CMMD from $0.579$ to $0.084$. We adopt Teacher-SN in all subsequent experiments.

\tinytit{A4: Which corruption schedule best supports both correspondence and deployment?}\label{sec:time_coupling}
We compare two corruption schedules: alternating-clean ($\pi_{\mathrm{ac}}$), which keeps one modality clean and matches the endpoint-conditioned setting used at deployment for pure captioning, and mixed-corruption ($\pi_{\mathrm{ind}}$), which samples $(t,\tau)\sim\Unif{[0,1]^2}$ and exposes the model to jointly corrupted image--text states. In isolation, $\pi_{\mathrm{ac}}$ outperforms $\pi_{\mathrm{ind}}$ on captioning (\Cref{fig:time_sampling}), which we attribute to a closer match between training and deployment noise distributions for that task. We hypothesize, however, that $\pi_{\mathrm{ind}}$ provides broader joint supervision that is useful as a pretraining stage prior to specialization. To test this, we run five independent fine-tuning experiments that switch from $\pi_{\mathrm{ind}}$ to $\pi_{\mathrm{ac}}$ at $\{75\text{k}, 100\text{k}, 125\text{k}, 150\text{k}, 175\text{k}\}$ steps; all five converge to a CIDEr score that surpasses the pure-$\pi_{\mathrm{ac}}$ run, reaching the pure-$\pi_{\mathrm{ac}}$ final value in roughly half the training steps. We therefore adopt mixed-corruption pretraining followed by alternating-clean refinement, using $100$k $\pi_{\mathrm{ind}}$ steps followed by $50$k $\pi_{\mathrm{ac}}$ steps (total wall-clock time under $24$\,h) in all subsequent SD3 experiments. For convergence analysis, runs were continued to $200$k total steps.

\begin{wraptable}{r}{0.49\textwidth}
    \vspace{-4mm}
    \centering
    \caption{\small Downstream captioning and VQA for unified models. Grouped by paradigm: AR/hybrids (top), diffusion (middle), and ours (bottom). Parentheses denote training resolution. COCO reports CIDEr; VQAv2, VizWiz, and OKVQA report accuracy.}
    \label{tab:vqa_short}

    \vspace{-1mm}
    
    \setlength{\tabcolsep}{2pt}
    \renewcommand{\arraystretch}{1.02}
    \footnotesize

    \newcommand{\head}[1]{\multicolumn{1}{c}{\rotatebox{90}{\textbf{#1}}}}

    \begin{tabular}{
        l
        c
        @{\hspace{8pt}} S[table-format=2.1]
        @{\hspace{8pt}} S[table-format=2.1]
        @{\hspace{8pt}} S[table-format=2.1]
        @{\hspace{8pt}} S[table-format=2.1]
    }
        \toprule
        \textbf{Model} & {\makecell[b]{\textbf{Trainable}\\\textbf{Params}}} & \head{COCO$\uparrow$} & \head{VQAv2$\uparrow$} & \head{VizWiz$\uparrow$} & \head{OKVQA$\uparrow$} \\
        \midrule
        CM3Leon~\citep{yu_scaling_2023}  & 7{B} & 61.6 & 47.6 & 37.6 & 23.8 \\
        Chameleon~\citep{team_chameleon_2025} & 7B & 18.0 & {--} & {--} & {--} \\
        LWM~\citep{liu_world_2025} & 7B & {--} & 55.8 & 11.6 & {--} \\
        Show-O (256)~\citep{xie_show-o_2025} & 1.3B & {--} & 64.7 & {--} & {--} \\
        Show-O (512)~\citep{xie_show-o_2025} & 1.3B & {--} & 69.4 & {--} & {--} \\
        Transfusion~\citep{zhou_transfusion_2024} & 7B & 29.0 & {--} & {--} & {--} \\
        \hline
        DualDiff (256)~\citep{li_dual_2025} & 2B & {--} & 59.5 & 19.4 & 28.5 \\
        DualDiff (512)~\citep{li_dual_2025} & 2B & 56.2 & 60.1 & 29.9 & 25.3 \\
        \hline
        FullFlow (Ours) & \textbf{130M} & 54.93 & 56.5 & 50.7 & 23.3\\
        \bottomrule
    \end{tabular}
    \vspace{-3mm}
\end{wraptable}

\subsection{Comparison to Dual Diffusion}\label{sec:main_dd}
Dual Diffusion is originally trained end-to-end at large scale, which would confound architectural and training-budget effects. To isolate the modeling contribution, we construct a matched baseline (\emph{DualDiff.-LoRA}) by re-implementing its architecture on the same SD3 backbone with LoRA at the same rank used by FullFlow ($r=32$), identical data, and the same optimizer; trainable parameter counts are nearly identical (136.8M vs.\ 136.1M). The two methods differ only in their cross-modal interface: DualDiff.-LoRA retains Dual Diffusion's masked text-diffusion pathway and partially trains text-conditioning components, whereas FullFlow operates on valid partially deleted strings via the tokenizer-space Edit Flow with all text encoders frozen, enabling same-noise teacher matching for image-prior retention. Since FullFlow is ${\sim}8\times$ faster per step, we report \emph{step-matched} and \emph{time-matched} configurations.\looseness=-1

\begin{table*}[t]
  \centering
  \begin{minipage}[t]{0.51\linewidth}
    \vspace{0pt}
    \centering
    \caption{\small Cross-modal generation on SD3 with matched data, LoRA rank, and training setup. FID and CMMD are computed against the frozen base SD3 model; CIDEr and BERT-F1 are computed against held-out Moondream captions on the validation split; full details in \Cref{tab:main_comparison_extend}.}
    \setlength{\tabcolsep}{3pt}
    \renewcommand{\arraystretch}{1.05}
    \footnotesize
    \begin{tabular}{l
                    S[table-format=2.2]
                    S[table-format=1.3]
                    S[table-format=3.1]
                    S[table-format=1.2]}
      \toprule
      & \multicolumn{2}{c}{\textbf{Text$\rightarrow$Image}} & \multicolumn{2}{c}{\textbf{Image$\rightarrow$Text}} \\
      \cmidrule(lr){2-3}\cmidrule(lr){4-5}
      \textbf{Method} &
      {\textbf{FID}$\downarrow$} &
      {\textbf{CMMD}$\downarrow$} &
      {\textbf{CIDEr}$\uparrow$} &
      {\textbf{BERT-F1}$\uparrow$} \\
      \midrule
      DualDiff.-LoRA & 62.65 & 0.870 & 2.0 & -0.27 \\
      Ours, step-match & \textbf{29.32} & \textbf{0.021} & 13.4 & 0.02 \\
      Ours, time-match & 31.57 & 0.121 & \textbf{99.4} & \textbf{0.44} \\
      \bottomrule
    \end{tabular}
    \label{tab:main_comparison}
  \end{minipage}
  \hfill
  \begin{minipage}[t]{0.47\linewidth}
    \vspace{-2.5mm}
    \centering
    \input{figures/time_sampling/fig}
    \vspace{-6mm}
    
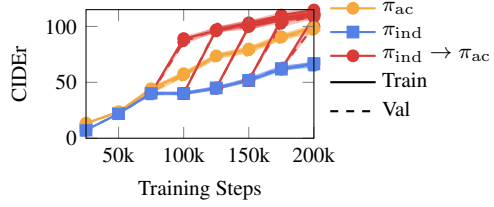
\captionof{figure}{\small CIDEr for alternating-clean ($\pi_{\text{ac}}$), mixed-corruption ($\pi_{\text{ind}}$), and five independent runs that switch from $\pi_{\text{ind}}$ to $\pi_{\text{ac}}$ at $75$k, $100$k, $125$k, $150$k, and $175$k steps. Mean over a 5k held-out split; error bars show 95\% CIs.}
    \label{fig:time_sampling}
  \end{minipage}
  \vspace{-4mm}
\end{table*}

\tinytit{Conditional generation.}
\Cref{tab:main_comparison} reports cross-modal generation against DualDiff.-LoRA. We report step- and time-matched FullFlow configurations to isolate complementary effects: per-update modeling quality and practical advantage at a fixed compute budget. At matched wall-clock, FullFlow reduces FID from $62.65$ to $31.57$ and CMMD from $0.870$ to $0.121$, and improves CIDEr from $2.0$ to $99.4$ and BERTScore-F1 from $-0.27$ to $0.44$. Step-matched gains move in the same direction with a smaller image$\rightarrow$text margin (CIDEr $13.4$), indicating that both architecture and per-step efficiency contribute to the improvement. Qualitative examples are provided in \cref{fig:image2text1,fig:image2text2,fig:t2i_sd3_app}.

\tinytit{Efficiency.}
The matched-LoRA setting also exposes the computational impact. DualDiff.-LoRA backpropagates through its masked text-diffusion pathway and partially trains text-conditioning components, requiring ${\sim}84$\,GB peak VRAM. By keeping the text encoders fully frozen and operating on valid partially deleted strings, FullFlow runs at ${\sim}38$\,GB and raises throughput from $0.25$ to $1.96$ steps/s, fitting the full training on two RTX A5000 GPUs.

\subsection{Downstream Understanding}\label{sec:downstream}
We test whether the same text-generation interface supports downstream multimodal understanding. \Cref{tab:vqa_short} positions FullFlow against two families: autoregressive and AR--diffusion hybrids (top), and fully diffusion-based models (middle). At $10$--$50\times$ fewer trainable parameters than every baseline and at half the resolution of the strongest competitor, FullFlow attains the highest VizWiz accuracy in the table ($50.7$); on MS-COCO captioning it is competitive at $256$ resolution ($54.9$ CIDEr) with DualDiff at $512$ resolution ($56.2$) and the $7$B CM3Leon ($61.6$), and remains close on VQAv2 ($56.5$, vs.\ $60.1$ for DualDiff(512) and $69.4$ for the larger Show-O). The weakest result is OKVQA ($23.3$, ${\sim}18\%$ below DualDiff(256)), which we attribute to that benchmark's reliance on external world knowledge that scales with pretraining-data scale: FullFlow is fine-tuned on a small LAION subset, whereas DualDiff is trained on multiple orders of magnitude more data. Overall, a fully diffusion-based interface matches autoregressive and hybrid unified models at a small fraction of their size, supporting a single diffusion process across both modalities; further comparison in \Cref{tab:lm_benchmarks}.

\subsection{Transfer to FLUX.1-dev.}\label{sec:transfer}
To verify the recipe is not specific to SD3, we apply it to FLUX.1-dev with minimal architecture-specific changes, sharding the 22\,GB model across three RTX A5000 GPUs at $0.49$ steps/s. Following the two-stage schedule of \Cref{sec:time_coupling}, $40$k $\pi_{\mathrm{ind}}$ steps reach CIDEr $57.5$ and a further $20$k $\pi_{\mathrm{ac}}$ steps raise it to $73.6$, while the text$\rightarrow$image prior holds at FID $24.7$ / CMMD $0.014$; total wall-clock training is $36$\,h. Qualitative samples are provided in \Cref{fig:image2text_flux,fig:t2i_flux_app}.

\begin{wrapfigure}{r}{0.45\textwidth}
    \vspace{-4mm}
    \centering
    \input{figures/sd3_joint/fig}
    \vspace{-1mm}
    \caption{\small Joint generation over $\tau=t^{2^p}$ trajectories. Color shows mean CLIP over 1k samples.}
    \label{fig:sd3_joint}
    \vspace{-2.5em}
\end{wrapfigure}
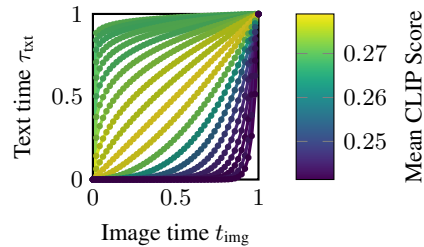

\subsection{Joint generation.}\label{sec:joint}
The dual-timestep formulation also enables image and text to be sampled jointly rather than conditioning on one modality. We sweep trajectories by fixing image time to a linear schedule and setting text time to $\tau=t^{2^p}$, with $p$ controlling which modality denoises earlier. \Cref{fig:sd3_joint} shows the diagonal trajectory ($p=0$) yields the highest CLIP score, indicating that balanced co-denoising gives the strongest image--text agreement; text-first trajectories remain competitive, while delaying text degrades alignment. Qualitative examples are in \cref{fig:joint1,fig:joint2,fig:joint3,fig:joint4}.


%% file: figures/gradient_norm/fig.tex
\pgfplotstableread[col sep=comma]{figures/gradient_norm/data.csv}\datatable

\begin{tikzpicture}
  \begin{axis}[
      width=0.80\linewidth,
      height=0.45\linewidth,
      xmin=0, xmax=200300,
      ymin=0, ymax=0.075,
      xtick={0,50000,100000,150000,200000,250000},
      xticklabels={0,50k,100k,150k,200k,250k},
      xlabel={Training Steps},
      ylabel={$\lambda_\text{txt}$},
      tick label style={font=\footnotesize},
      label style={font=\footnotesize},
      legend style={
        at={(1.03,0.60)},
        anchor=west,
        draw=none,
        fill=none,
        font=\small,
        cells={anchor=west}
      },
      axis on top,
      clip=true,
      unbounded coords=jump,
      line width=0.8pt,
      scaled x ticks=false,
      scaled y ticks=false,
      xticklabel style={
        font=\footnotesize,
      },
      yticklabel style={
        font=\footnotesize,
        /pgf/number format/fixed,
        /pgf/number format/precision=3
      },
    ]

    \addplot[
      blue!35,
      line width=0.5pt,
      opacity=0.7,
      restrict y to domain*=0:0.2,
      legend image post style={opacity=1.0, line width=1.5pt}
    ] table[x=step,y=g_ratio] {\datatable};
    \addlegendentry{\footnotesize raw}

    \addplot[
      blue!70!black,
      very thick,
      restrict y to domain*=0:0.2,
      legend image post style={opacity=1.0, line width=1.5pt}
    ] table[x=step,y=ema_alpha] {\datatable};
    \addlegendentry{\footnotesize EMA}

  \end{axis}
\end{tikzpicture}

%% file: figures/time_sampling/fig.tex
\newcommand{\addmeanciplot}[3]{

    \addplot[draw=none, forget plot, name path=upper#1] table [
        x=ckpt_num, 
        y expr=\thisrow{#1_mean} + \thisrow{#1_ci_95}, 
        col sep=comma
    ] {./figures/time_sampling/data.csv};

    \addplot[draw=none, forget plot, name path=lower#1] table [
        x=ckpt_num, 
        y expr=\thisrow{#1_mean} - \thisrow{#1_ci_95}, 
        col sep=comma
    ] {./figures/time_sampling/data.csv};

    \addplot[fill=#2, opacity=0.3, forget plot] fill between[of=upper#1 and lower#1];

    \addplot[#2, #3] table [
        x=ckpt_num, 
        y=#1_mean, 
        col sep=comma
    ] {./figures/time_sampling/data.csv};
}

\begin{tikzpicture}

\begin{axis}[
      width=0.7\linewidth,
      height=0.5\linewidth,
      xmin=25000, xmax=200000,
      ymin=0, ymax=1.15,
      xtick={50000,100000,150000,200000,250000},
      xticklabels={50k,100k,150k,200k,250k},
      ytick={0,0.5,1},
      yticklabels={0,50,100},
      xlabel={Training Steps},
      ylabel={CIDEr},
      tick label style={font=\footnotesize},
      label style={font=\footnotesize},
      legend style={
        at={(1.03,0.60)},
        anchor=west,
        draw=none,
        fill=none,
        font=\small,
        cells={anchor=west}
      },
      axis on top,
      clip=true,
      unbounded coords=jump,
      line width=0.8pt,
      scaled x ticks=false,
      scaled y ticks=false,
      xticklabel style={
        font=\footnotesize,
      },
      yticklabel style={
        font=\footnotesize,
        /pgf/number format/fixed,
        /pgf/number format/precision=3
      },
    ]

    \addmeanciplot{res_004_txt_train}{dit_yellow}{thick, mark=*, mark size=2pt, solid}
    \addlegendentry{\footnotesize $\pi_{\mathrm{ac}}$}
    \addmeanciplot{res_004_txt_val}{dit_yellow}{thick, mark=*, mark size=2pt, dashed, forget plot}

    \addmeanciplot{res_001_txt_train}{dit_blue}{thick, mark=square*, mark size=2pt, solid}
    \addlegendentry{\footnotesize $\pi_{\mathrm{ind}}$}
    \addmeanciplot{res_001_txt_val}{dit_blue}{thick, mark=square*, mark size=2pt, dashed, forget plot}

    \addmeanciplot{res_001_100k_txt_train}{dit_red}{thick, mark=*, mark size=2pt, solid}
    \addlegendentry{\footnotesize $\pi_{\mathrm{ind}}\rightarrow\pi_{\mathrm{ac}}$} 
    \addmeanciplot{res_001_100k_txt_val}{dit_red}{thick, mark=*, mark size=2pt, dashed, forget plot}
    
    \addmeanciplot{res_001_125k_txt_train}{dit_red}{thick, mark=*, mark size=2pt, solid, forget plot}
    \addmeanciplot{res_001_125k_txt_val}{dit_red}{thick, mark=*, mark size=2pt, dashed, forget plot}
    
    
    \addmeanciplot{res_001_75k_txt_train}{dit_red}{thick, mark=*, mark size=2pt, solid, forget plot}
    \addmeanciplot{res_001_75k_txt_val}{dit_red}{thick, mark=*, mark size=2pt, dashed, forget plot}
    
    \addmeanciplot{res_001_150k_txt_train}{dit_red}{thick, mark=*, mark size=2pt, solid, forget plot}
    \addmeanciplot{res_001_150k_txt_val}{dit_red}{thick, mark=*, mark size=2pt, dashed, forget plot}
    
    \addmeanciplot{res_001_175k_txt_train}{dit_red}{thick, mark=*, mark size=2pt, solid, forget plot}
    \addmeanciplot{res_001_175k_txt_val}{dit_red}{thick, mark=*, mark size=2pt, dashed, forget plot}
    
    \addmeanciplot{res_004_txt_train}{dit_yellow}{only marks, mark=*, mark size=2pt, solid, forget plot}
    \addmeanciplot{res_004_txt_val}{dit_yellow}{only marks, mark=*, mark size=2pt, dashed, forget plot}

    \addmeanciplot{res_001_txt_train}{dit_blue}{only marks, mark=square*, mark size=2pt, solid, forget plot}
    \addmeanciplot{res_001_txt_val}{dit_blue}{only marks, mark=square*, mark size=2pt, dashed, forget plot}

    \addlegendimage{black, thick, solid, mark=none}
    \addlegendentry{\footnotesize Train} 
    
    \addlegendimage{black, thick, dashed, mark=none}
    \addlegendentry{\footnotesize Val}   
\end{axis}
\end{tikzpicture}

%% file: figures/sd3_joint/fig.tex
\pgfplotstableread[col sep=comma]{figures/sd3_joint/data.csv}\datatable
\pgfplotstablegetrowsof{\datatable}
\pgfmathtruncatemacro{\lastrow}{\pgfplotsretval-1}

\begin{tikzpicture}
  \begin{axis}[
      width=0.6\linewidth,
      height=0.6\linewidth,
      xmin=0, xmax=1,
      ymin=0, ymax=1,
      domain=0:1,
      samples=50,
      axis equal image,
      axis on top,
      clip=true,
      xlabel={Image time $t_\text{img}$},
      ylabel={Text time $\tau_\text{txt}$},      
      xtick={0,0.5,1},
      ytick={0,0.5,1},
      tick label style={font=\small},
      label style={font=\small},
      line width=0.8pt,
      colormap/viridis,
      colorbar,
      colorbar style={
        ylabel={Mean CLIP Score},
        yticklabel style={
          /pgf/number format/fixed,
          /pgf/number format/precision=3
        },
        tick label style={font=\small},
        label style={font=\small},
      },
    ]

    \pgfplotsinvokeforeach{0,...,\lastrow}{
      \pgfplotstablegetelem{#1}{log_pow}\of{\datatable}
      \edef\logpow{\pgfplotsretval}

      \pgfplotstablegetelem{#1}{clip}\of{\datatable}
      \edef\clipscore{\pgfplotsretval}
      
      \edef\temp{
        \noexpand\addplot+[
          mesh,
          scatter,
          shader=flat,        
          mark=*,
          mark options={draw=none},
          mark repeat=1,
          mark size=0.6pt,
          very thick,
          opacity=1.0,
          point meta=\clipscore,
          forget plot,
        ]
        {pow(x,pow(2,\logpow))};
      }
      \temp
    }
  \end{axis}
\end{tikzpicture}

%% file: sections/limitations.tex
\section{Limitations and Conclusion}\label{sec:limitations}


\tinytit{Limitations.}
We focused on a compute-feasible proof of concept rather than peak end-to-end performance. For this reason the main results in this work use data-, parameter-, and compute-matched comparisons against prior work (\Cref{sec:main_dd}). We did not explore scaling of the method across: training at $512{\times}512$ and above, larger paired datasets, instruction tuning for VQA and dialogue (closing the gap to specialized AR VLMs), and porting the same uplift to other rectified-flow backbones. All these directions remain interesting future developments which are conceptually supported by our method, but would require more compute and time. 

\tinytit{Conclusion.}
FullFlow upgrades pretrained text-to-image models into bidirectional image--text generators by operating in valid tokenizer space with the text-conditioning stack frozen, supporting text-to-image, captioning, and joint generation from a single adapted backbone. The fact that bidirectional capability emerges from such a thin adaptation suggests that pretrained text-to-image priors encode richer visual--semantic structure than commonly assumed, and that similar lightweight uplifts may extend to other larger backbones, where the compute savings would be even more impactful.

\tinytit{Broader impact.}
While broadening access to interactive image--text applications, our method may also amplify risks involving synthetic generation, biased captions, and dataset contamination.

%% file: appendix/training.tex
\section{Flow Matching: Similarity Between Continuous and Discrete}

Despite their apparent differences, continuous rectified flow and discrete Edit Flows instantiate the same fundamental principle: learn a model to predict local transport updates along a path from a base distribution to the data distribution. In continuous flow matching, we learn a velocity field that describes how samples move continuously through latent space. In discrete flow matching applied to text, we instead learn to predict insertion discrete token placements that progressively transform an empty sequence into the target sequence. Both formulations share the same path-based training objective: at each point along the trajectory, the model is trained to make the locally-optimal next move (either a velocity in continuous space or a jump in discrete space) to reach the data. This unified perspective on flow matching enables us to train image and text processes with a shared modeling philosophy while respecting the geometric differences between continuous image latents and discrete token sequences. \Cref{fig:fm_concept} illustrates this conceptual connection.

\begin{figure*}[ht]
    \centering
    \hfill
    \begin{subfigure}[t]{0.48\linewidth}
        \centering
        \frame{\includegraphics[width=\linewidth]{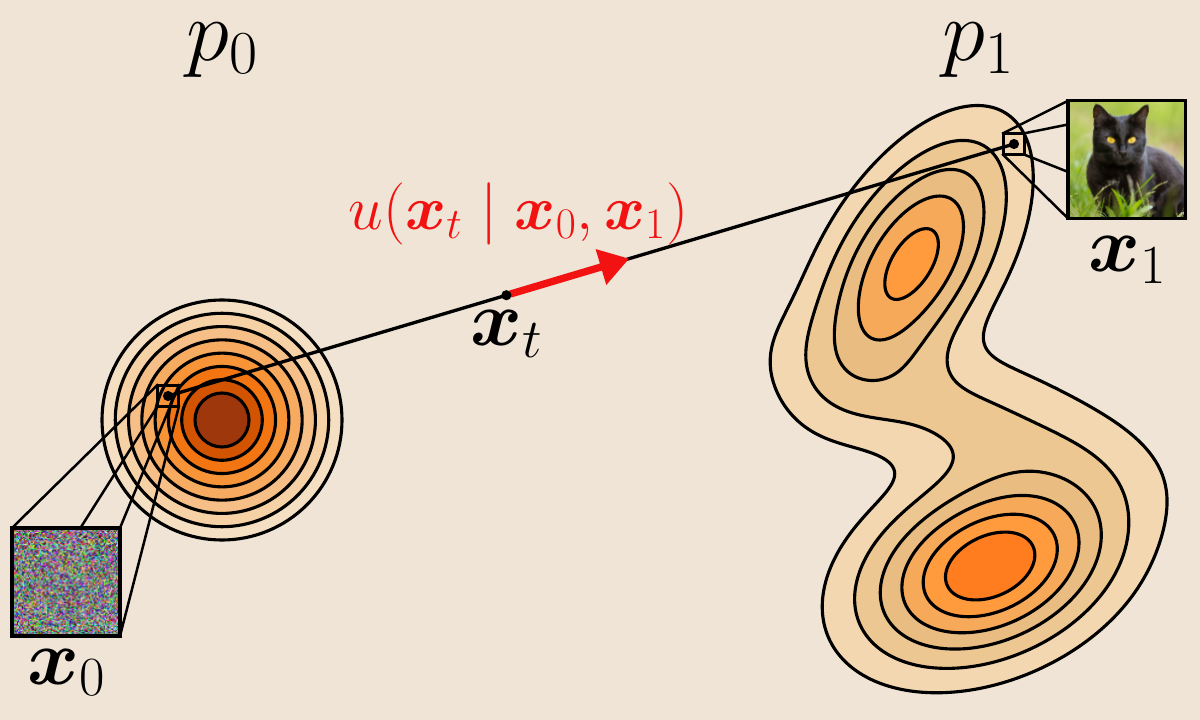}}
        \caption{Continuous flow matching}
        \label{fig:fm_continuous}
    \end{subfigure}%
    \hfill
    \begin{subfigure}[t]{0.48\linewidth}
        \centering
        \frame{\includegraphics[width=\linewidth]{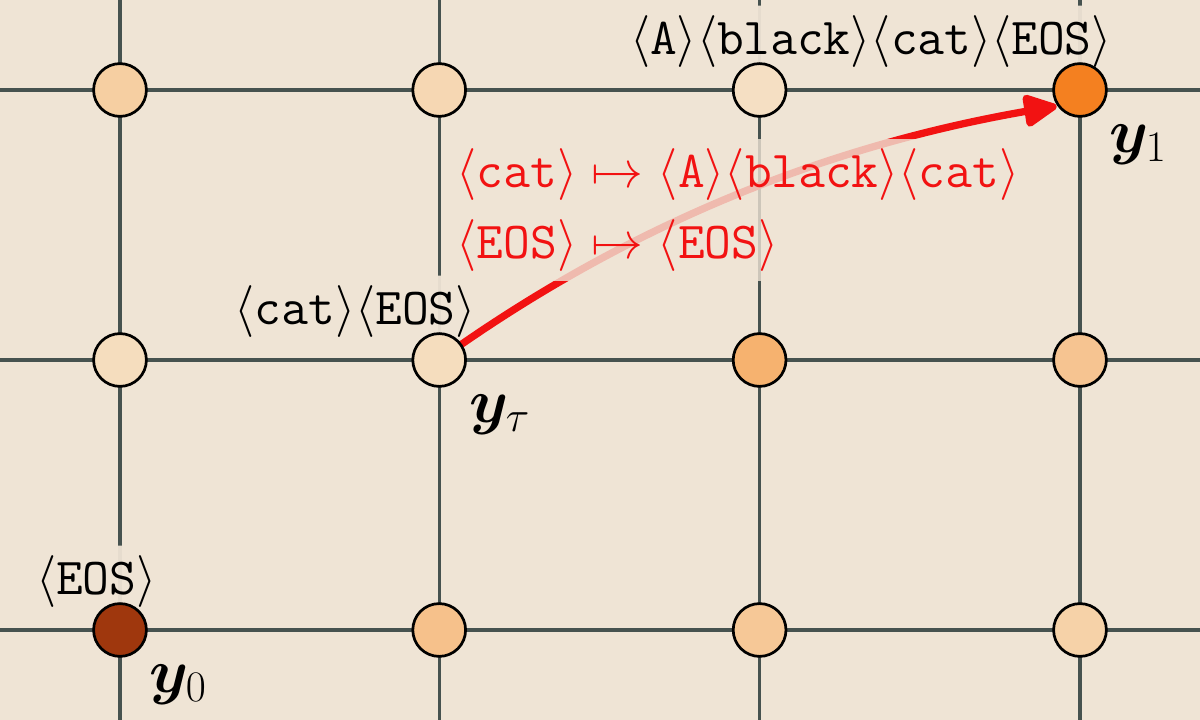}}
        \caption{Discrete flow via jumps}
        \label{fig:fm_discrete}
    \end{subfigure}%
    \hfill\mbox{}%
    \caption{\small Continuous rectified flow (left) and discrete Edit Flows (right) instantiate the same path-based principle: learn local transport from a base distribution to data. In continuous space the supervised target is a velocity; in discrete space it is an insertion jump. The red arrow/token marks the quantity the model is trained to predict.}
    \label{fig:fm_concept}
\end{figure*}

\section{Training Details}\label{app:training}

\subsection{Dataset Preparation and Pre-Encoding}

We train on a subset of the LAION-COCO-Aesthetic collection~\citep{schuhmann_laion-5b_2022}, a curated split of LAION-2B-EN filtered for aesthetic quality scores and paired with semantically aligned captions. Before training we apply a mild aspect-ratio filter (retaining images with width/height ratio below $1.4$ in both orientations) to avoid extreme crops, yielding a training set of ${\sim}280{,}000$ image--caption pairs.

Images are pre-encoded offline with the SD3 VAE to avoid redundant encoder passes during training. Each image is resized to $256 \times 256$ with bicubic interpolation and center-cropped, normalized to $[-1,1]$, then encoded to a latent $\vx_1 \in \mathbb{R}^{16 \times 32 \times 32}$ via
\begin{equation}
  \vx_1 = \bigl(\text{VAE-Enc}(\text{img}) - s_\text{shift}\bigr) \cdot s_\text{scale},
\end{equation}
where $s_\text{shift}$ and $s_\text{scale}$ are the SD3 VAE shift and scaling factors. The encoded latents and raw caption strings are saved to disk; the VAE is not loaded during training. We visualize the caption-length distribution of the dataset in \cref{fig:token_lengths}.

\input{figures/data/fig}

\subsection{Text-to-Image Evaluation Prompt Suite}\label{app:t2i_prompts}
For text$\rightarrow$image evaluation, we use a fixed prompt suite rather than a held-out LAION split. The goal is to better approximate open-ended generation, where prompts vary in length, compositional complexity, and descriptive density. We therefore combine three complementary sources.

First, we sample natural image captions from MS-COCO, stratified by a simple regex word-token count: short captions contain 4--8 tokens, medium captions contain 9--17 tokens, and long captions contain 18 or more tokens; captions with fewer than 4 tokens are discarded. We sample 300 captions from each length bucket. Second, we sample compositional prompts from T2I-CompBench~\citep{huang_t2i-compbench_2023}, using 100 prompts from each of eight categories: color, shape, texture, spatial, 3D spatial, non-spatial, numeracy, and complex. Third, we include 200 dense long-form prompts from DrawBench-upsampled, stratified across its available categories when possible. This yields a nominal 1.9k-prompt suite before cross-source deduplication.

All prompts are whitespace-normalized, lowercased for duplicate detection, and deduplicated across sources. We use a fixed random seed for sampling and keep the same prompt suite for all methods.

\subsection{Resolution: Why \texorpdfstring{$256\times256$}{256x256}?}\label{app:resolution}

SD3 includes $256 \times 256$ images in its multi-resolution pretraining curriculum~\citep{esser_scaling_2024}. We therefore evaluate at $256 \times 256$, a resolution already covered by the pretrained backbone, while keeping the image-token sequence length small enough for our two-GPU setting. At this resolution, the SD3 VAE encodes images into $32 \times 32$ spatial latents, which SD3's patch-embed layer ($2 \times 2$ patches) reduces to $n = 256$ image tokens. The joint sequence length---image tokens plus three text encoder outputs of length $77$ each, is therefore manageable within the $2 \times 24\,\text{GB}$ GPU budget.

Compared to higher resolutions, a $512 \times 512$ input quadruples the image token count to $1{,}024$, increasing the full-sequence self-attention cost by roughly $16\times$, which is prohibitive for our hardware setup. Compared to $128 \times 128$, the VAE compression is less severe at $256 \times 256$ and perceptual quality is substantially better, making image-generation evaluation more meaningful. The $256 \times 256$ regime therefore gives the best compute--quality trade-off available on two RTX A5000 GPUs.

\subsection{Optimizer and Learning-Rate Schedule}
We use AdamW~\citep{loshchilov_decoupled_2019} with $\beta_1=0.9$, $\beta_2=0.999$, $\epsilon=10^{-8}$, and \emph{two parameter groups} with separate learning rates and regularization:

\begin{itemize}[leftmargin=*, topsep=2pt, itemsep=2pt]
  \item \textbf{LoRA adapters and timestep processors} (text timestep $\tau$ extra linear layer and scalar.): $\text{lr}=10^{-4}$, weight decay $=0$.
  \item \textbf{Final AdaLN text modulation and text heads} and other newly introduced weights not covered by LoRA: $\text{lr}=5\times10^{-4}$, weight decay $=10^{-2}$.
\end{itemize}

A constant schedule with a $5{,}000$-step linear warmup is used; no learning-rate decay is applied. Gradients are clipped to a global $\ell_2$ norm of $2.0$ before each optimizer step.

\subsection{LoRA Configuration}

LoRA adapters~\citep{hu_lora_2021} are attached to all linear projection layers inside the \texttt{transformer\_blocks} of the MM-DiT backbone (query, key, value, output projections of both self-attention streams). New heads (text-conditioning layers and text output heads) are excluded from LoRA and trained directly. The LoRA configuration is:

\begin{center}
\begin{tabular}{ll}
\toprule
\textbf{Hyperparameter} & \textbf{Value} \\
\midrule
Rank $r$ & 32 \\
Scaling $\alpha$ & 32 \\
Dropout & 0.05 \\
Target modules & All \texttt{nn.Linear} in \texttt{transformer\_blocks} \\
Bias & none \\
\bottomrule
\end{tabular}
\end{center}

With this configuration, the total trainable parameter count is ${\sim}136$M for SD3-M (${\approx}5\%$ of the full model) and ${\sim}298$M for FLUX.1 (${\approx}2.5\%$).

\subsection{Zero-Inflated Insertion Counts}\label{app:pi_head}
Both backbones we adapt operate at a fixed sequence length ($77$ for SD3, $128$ for FLUX) inherited from their pretrained text encoders. Under deletion-to-empty corruption with retention probability $\tau$, the vast majority of retained positions have no deleted left-neighbours ($|A_i|=0$), especially at high $\tau$ where most tokens are kept. The Poisson target on $\lambda_\theta^i$ is therefore strongly skewed toward zero, biasing the rate head and starving the model of gradient signal on the rare nonzero counts that actually drive insertion behaviour.

Following OneFlow~\citep{nguyen_oneflow_2025}, we factor the insertion-count distribution as a zero-inflated Poisson by adding a binary gating head $\pi_\theta^i(\tilde{\vy}_\tau, \tau) \in [0,1]$ that predicts $\Pr[]{|A_i|>0}$, while $\lambda_\theta^i$ models the Poisson rate conditional on $|A_i|>0$. The text loss \cref{eq:txt_loss} then becomes
\begin{align*}
    \mathcal{L}_{\mathrm{txt}}(\theta) =
    \mathrm{BCE}\bigl(\pi_\theta^i,\, \Ind{|A_i|>0}\bigr)
    \;+\; \Ind{|A_i|>0} \cdot \mathcal{L}_{\mathrm{Poisson}}(\lambda_\theta^i; |A_i|)
    \;+\; \mathcal{L}_{\mathrm{tok}}(w_\theta^i; A_i),
\end{align*}
summed over retained positions $i$, with the Poisson and token-identity losses evaluated only on positions with at least one insertion. At inference, position $i$ is kept empty with probability $1-\pi_\theta^i$ and otherwise samples a count from $\lambda_\theta^i$ followed by token identities from $w_\theta^i$. This decouples the abundant zero-class supervision from the rare nonzero counts and prevents the rate head from collapsing toward zero on long padded sequences.

\subsection{Training Hyperparameters}\label{app:hyperparams}

\Cref{tab:hyperparams} lists the training hyperparameters for the main FullFlow-SD3 run and the matched DualDiff.-LoRA baseline. The two runs use the same optimizer, batch size, learning-rate schedule, LoRA configuration, and training budget. They differ only in the corruption schedule and image-supervision target: FullFlow uses a mixed-to-alternating schedule, training with $\pi_{\mathrm{ind}}$ for the first 100k steps and then switching to $\pi_{\mathrm{ac}}$ via \texttt{ac\_sched}, whereas DualDiff.-LoRA uses $\pi_{\mathrm{ac}}$ throughout. FullFlow also uses same-noise teacher matching (\texttt{teacher\_matching=noisy}), while DualDiff.-LoRA uses the standard rectified-flow image objective.

\begin{table}[ht]
    \centering
    \small
    \setlength{\tabcolsep}{6pt}
    \renewcommand{\arraystretch}{1.1}
    \caption{\small Training hyperparameters for the SD3-M experiments. Values apply to both FullFlow and the matched DualDiff.-LoRA baseline unless explicitly noted; the two runs differ only in the corruption schedule and the image-supervision target.}
    \label{tab:hyperparams}
    \begin{tabular}{lll}
        \toprule
        \textbf{Hyperparameter} & \textbf{Value} & \textbf{Note} \\
        \midrule
        \multicolumn{3}{l}{\textit{Hardware \& precision}} \\
        GPUs & 2 $\times$ RTX A5000 (24\,GB) & DDP, NCCL backend \\
        Mixed precision & BFloat16 & Optimizer states in FP32 \\
        Gradient checkpointing & No &  \\
        \midrule
        \multicolumn{3}{l}{\textit{Data \& batching}} \\
        Training resolution & $256 \times 256$ & Pre-encoded VAE latents \\
        Training set size & ${\sim}280{,}000$ image--caption pairs & LAION-COCO-Aesthetic \\
        Global batch size & 10 & 5 per GPU \\
        Micro-batch size & 5 & 1 accumulation step \\
        Max sequence length & 77 & T5 tokenizer tokens \\
        \midrule
        \multicolumn{3}{l}{\textit{Optimisation}} \\
        Optimizer & AdamW & $\beta_1{=}0.9$, $\beta_2{=}0.999$, $\varepsilon{=}10^{-8}$ \\
        LoRA / head learning rate & $1\times10^{-4}$ & No weight decay \\
        Base parameter learning rate & $5\times10^{-4}$ & Weight decay $10^{-2}$ \\
        LR schedule & Constant + warmup & 5{,}000 warmup steps \\
        Gradient clip ($\ell_2$) & 2.0 & \\
        Training steps & 200{,}000 & \\
        Random seed & 42 & \\
        \midrule
        \multicolumn{3}{l}{\textit{Adaptive gradient balancing}} \\
        Estimation frequency & Every 100 steps & On a micro-batch of size 3 \\
        EMA decay $\beta$ & 0.99 & \\
        Gradient ratio scale & 5.0 & Applied before EMA \\
        \midrule
        \multicolumn{3}{l}{\textit{Loss weights}} \\
        Image weight $w_\text{img}$ & 1.0 & \\
        Text weight $w_\text{txt}$ & 0.05 & Scaled further by $\lambda_\text{txt}$ \\
        Teacher matching & noisy & (\texttt{Teacher-SN}) \\
        \midrule
        \multicolumn{3}{l}{\textit{EMA of model weights}} \\
        EMA decay & 0.9995 & Applied to trainable params \\
        EMA dtype & FP32 & \\
        \bottomrule
    \end{tabular}
\end{table}

\subsection{DualDiff.-LoRA Baseline Reproduction}\label{app:dd_baseline}
We re-implement the DualDiff.-LoRA baseline by directly using the official Dual Diffusion training pipeline\footnote{\url{https://github.com/zijieli-Jlee/Dual-Diffusion}}, with two modifications to match our setup:
\begin{itemize}[leftmargin=*, topsep=2pt, itemsep=2pt]
  \item \emph{Data loader.} Replaced with our pre-encoded LAION-Aesthetic loader (\Cref{app:training}), so that the VAE is not held in VRAM and the dataset, captions, and resolution exactly match FullFlow.
  \item \emph{Trainable parameters.} Instead of full backbone fine-tuning, we attach the same LoRA configuration as FullFlow (rank $r=32$ on all linear layers in \texttt{transformer\_blocks}; \Cref{app:training}) and additionally train the newly initialized text-conditioning modules and text heads.
\end{itemize}
All other Dual-Diffusion components, including the mask-noise schedule, text-head architecture, sampler, and DD-specific hyperparameters, are kept exactly as released. The optimizer, learning-rate schedule, batch size, and total number of training steps match the FullFlow training run (\Cref{tab:hyperparams}); the trainable parameter counts are also nearly matched, at ${\sim}136.8$M for DualDiff.-LoRA and ${\sim}136.1$M for FullFlow.

\subsection{VQA Finetuning}\label{app:vqa_finetuning}

VQA adaptation follows the partial-text generation interface described in \Cref{sec:vqa}: the image and the question tokens are held fixed, and the insertion process is restricted to the answer span. The model learns to denoise only the answer tokens, conditioned on the image and the full question.

Concretely, finetuning resumes from a pretrained FullFlow checkpoint and replaces the joint image--text loss with a text-only VQA loss. No image velocity supervision is applied during this stage ($w_\text{img}=0$). The trainable parameters and LoRA configuration are identical to the pretraining stage.

Training runs for an additional $100{,}000$ steps on pre-encoded VQAv2 question--answer--image triples at $256\times256$ resolution.

\subsection{VQA Inference}\label{app:vqa_inference}
At inference time, the image, fixed at $t=1$, and the tokenized question are provided as conditioning. Sampling is restricted to the answer span by allowing insertions only after the question tokens; the question itself remains fixed and is never modified.

We use the \texttt{FlowMatchEulerDiscreteScheduler} with $\text{shift}=1.0$ for the text process and run $8$ denoising steps, which we found sufficient for short answer spans. Decoding is greedy (top-$k=1$, temperature $0.7$); classifier-free guidance is not applied ($\text{cfg}=1.0$). The maximum answer sequence length is capped at $77$ tokens. All results are obtained from the EMA checkpoint.

\subsection{Inference Hyperparameters}

All reported results use the EMA checkpoint. Sampling uses the \texttt{FlowMatchEulerDiscreteScheduler} with 28 denoising steps for both the image and text processes. For image-to-text captioning, the text scheduler uses shift $=1.0$ with temperature $0.7$ and top-$k=1$ (greedy) decoding. For text-to-image generation, the image scheduler uses the default SD3 shift. Classifier-free guidance is not applied in the main comparisons ($\text{cfg}=1.0$); qualitative examples in \Cref{app:cross} use image-side CFG where noted.

\section{Evaluation Protocols}\label{app:eval}

Unless noted, all metrics are computed on the EMA checkpoint with the inference settings of \Cref{app:vqa_inference} and the prompt suite of \Cref{app:t2i_prompts}. Image$\to$text metrics are computed on a fixed 5k-pair held-out split of LAION-Aesthetic with single-reference Moondream captions; text$\to$image metrics use the 1.9k-prompt suite.

\tinytit{Reference set for FID and CMMD.}
The reference set is the \emph{base text-to-image model itself} sampled under identical conditions (same prompts, seed, scheduler, denoising steps, and CFG scale). This isolates the effect of multimodal uplift from generic backbone-quality differences: any deviation directly reflects retention or degradation of the pretrained image prior.

\tinytit{FID.}
Computed with the \texttt{clean-fid} package~\citep{parmar2021cleanfid} using the standard Inception-V3 features, via
\verb|fid.compute_fid(base_fid_path, gen_dir, batch_size, num_workers=4)|.

\tinytit{CMMD.}
We compute CLIP Maximum Mean Discrepancy (CMMD) following \citet{jayasumana_rethinking_2024}, using the public PyTorch implementation \footnote{\href{https://github.com/sayakpaul/cmmd-pytorch}{https://github.com/sayakpaul/cmmd-pytorch}}, which extracts image embeddings with \texttt{openai/clip-vit-large-patch14-336}.

\tinytit{CIDEr.}
For image-to-text evaluation on our held-out LAION-Aesthetic split, we compute CIDEr with the standard \texttt{pycocoevalcap} implementation~\citep{vedantam_cider_2015}, using the Moondream caption as the single reference for each image. For MS-COCO captioning, we use the standard validation annotations and score each generated caption against all five human reference captions for the corresponding image.

\tinytit{BERTScore.}
Computed via the \texttt{bert\_score} package~\citep{zhang_bertscore_2020} with \verb|lang="en"| (which selects \texttt{roberta-large} as the underlying scorer) and \verb|rescale_with_baseline=True|; we report the F1 component, which can be negative under baseline rescaling.

\tinytit{CLIP score.}
Cosine similarity between CLIP image and text embeddings using \texttt{openai/clip-vit-base-patch32}~\citep{radford_learning_2021}, averaged over generated image-caption pairs.

\tinytit{VQA benchmarks.}
For VQAv2~\citep{goyal2017making}, VizWiz~\citep{gurari2018vizwiz}, and OKVQA~\citep{marino2019ok} we follow each benchmark's official evaluation protocol on its standard validation split.
\begin{align*}
    \mathrm{acc}(a) = \min\!\left(\tfrac{\#\{\text{annotators agreeing with } a\}}{3},\, 1\right).
\end{align*}

\section{Compute Budget}\label{app:compute}
All experiments run on RTX A5000 (24\,GB) GPUs in BF16 mixed precision; per-run training settings match \Cref{tab:hyperparams}. \Cref{tab:compute} lists the compute used for the experiments reported in this paper.

\begin{table}[ht]
    \centering
    \small
    \setlength{\tabcolsep}{6pt}
    \renewcommand{\arraystretch}{1.1}
    \caption{\small Compute budget for the experiments reported in this paper, broken down by run group. Preliminary or failed configurations not included in the final results are excluded.}
    \label{tab:compute}
    \begin{tabular}{lccc}
        \toprule
        \textbf{Run group} & \textbf{GPUs / run} & \textbf{Wall-clock / run} & \textbf{Total GPU-hours} \\
        \midrule
        SD3-M ablations + final (18 runs) & 2 & 28\,h & 864 \\
        FLUX.1-dev (1 run)                & 3 & 36\,h & 108 \\
        DualDiff.-LoRA (1 run)            & 4 & 24\,h & 96 \\
        \midrule
        \textbf{Total reported}           &   &       & \textbf{1{,}212} \\
        \bottomrule
    \end{tabular}
\end{table}

The 18 SD3-M runs cover the ablations in \Cref{sec:ablations} (trainable parameters, conditioning pathways, gradient balancing, teacher matching, corruption schedule) and the final step- and time-matched checkpoints. Preliminary and failed configurations not reported in the paper are excluded from this total.

%% file: figures/data/fig.tex
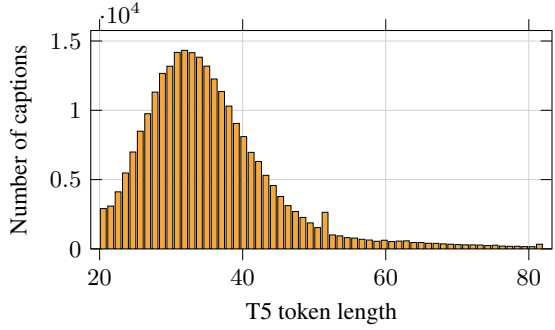
\begin{figure}[ht]
    \centering
    \begin{tikzpicture}
        \begin{axis}[
            width=0.55\linewidth,
            height=0.32\linewidth,
            ybar,
            bar width=2.2pt,
            xlabel={T5 token length},
            ylabel={Number of captions},
            xmin=20,
            xmax=82,
            ymin=0,
            enlarge x limits=0.02,
            axis line style={black},
            tick style={black},
            grid=major,
            grid style={line width=.1pt, draw=gray!25},
            major grid style={line width=.2pt, draw=gray!35},
            tick label style={font=\small},
            label style={font=\small},
        ]
        \addplot[
            fill=dit_yellow,
            draw=black,
            line width=0.2pt,
        ]
        table[
            x=length,
            y=count,
            col sep=tab,
        ]{./figures/data/data.tsv};
        \end{axis}
    \end{tikzpicture}
    \caption{\small Distribution of caption lengths in the training set, measured in T5 tokens after the LAION-COCO-Aesthetic filtering described in \Cref{app:training}. For readability the histogram drops the lowest and highest $1\%$ of caption lengths and uses 60 bins; this distribution motivates our $77$-token sequence cap.}
    \label{fig:token_lengths}
\end{figure}

%% file: appendix/theory.tex
\section{Motivating Analysis: Gradient Scale Mismatch}\label{app:gradnorm}
Our training objective combines an image regression loss and a text classification loss that share a transformer backbone. Because these losses induce gradients of different scale, naive weighting can cause one modality to dominate the shared-parameter updates. This appendix provides simple sanity-check calculations for the squared gradient norm with respect to the loss inputs, which motivates explicit loss reweighting such as the adaptive gradient balancing rule in \cref{sec:stabilizers}.

\begin{thm}\label{theorem:mse}
    Let $\vec{X}, \vec{Y} \sim \mathcal{N}(\vec{0} ,\mat{I}_n)$ with $Cov(\vec{X}, \vec{Y}) = \sigma \mat{I}_n$, then the expected gradient norm is the following:
    \begin{align}
        \E[\vec{X}, \vec{Y}]{ \norm{ \nabla_{\vec{X}} \texttt{MSE}(\vec{X}, \vec{Y}) }_2^2} = \frac{8(1 - \sigma)}{n}
    \end{align}
\end{thm}

\begin{proof}[Proof of \Cref{theorem:mse}]
    Let $\vec{X}, \vec{Y}$ be defined as in \Cref{theorem:mse}.
    \begin{align*}
         \E[\vec{X}, \vec{Y}]{ \left\| \nabla_{\vec{X}} \texttt{MSE}(\vec{X}, \vec{Y}) \right\|_2^2 }
        &=  \E[\vec{X}, \vec{Y}]{ \left\| \frac{2}{n}(\vec{X} - \vec{Y}) \right\|_2^2 }\\
        &= \frac{4}{n^2}  \E[\vec{X}, \vec{Y}]{ \transpose{(\vec{X} - \vec{Y})} (\vec{X} - \vec{Y})}\\
        &= \frac{4}{n^2} \left( \E[\vec{X}, \vec{Y}]{\transpose{\vec{X}}\vec{X}} + \E[\vec{X}, \vec{Y}]{\transpose{\vec{Y}}\vec{Y}} - 2 \E[\vec{X}, \vec{Y}]{ \transpose{\vec{X}}\vec{Y}}\right) \\
        &= \frac{4}{n^2} \left(n + n -2n \right) \\
        &= \frac{8(1 - \sigma)}{n}
    \end{align*}
\end{proof}

\begin{thm}\label{theorem:CE}
    Let $\vec{Z} \in \R^n$ be a random logits vector, $\vec{P} = \texttt{softmax}(\vec{Z}) \in \triangle^{n-1}$, and let $\vec{Q} = e_Y$ be a one-hot encoding of the target label $Y \in [n]$. Further, let $\vec{P}$ and $\vec{Q}$ agree in expectation by $\E[]{\transpose{\vec{P}}\vec{Q}} = \sigma$. Then
    \begin{align}
        \frac{n}{n-1}(1-\sigma)^2
        \;\le\;
        \E[]{\norm{\nabla_{\vec{Z}} \texttt{CE}(\vec{P}, \vec{Q})}^2_2}
        \;\le\;
        2(1-\sigma).
    \end{align}
\end{thm}

\begin{proof}[Proof of \Cref{theorem:CE}]
    Let $\vec{P}, \vec{Q}$ be defined as in \Cref{theorem:CE}. Since $\nabla_{\vec{Z}} \texttt{CE}(\vec{P}, \vec{Q}) = \vec{P} - \vec{Q}$, we have
    \begin{align*}
        \norm{\vec{P}-\vec{Q}}_2^2 = \norm{\vec{P}-e_Y}_2^2 = (1-P_Y)^2 + \sum_{j\neq Y} P_j^2.
    \end{align*}
    For the upper bound, note that $\sum_{j\neq Y}P_j^2 \le \sum_{j\neq Y}P_j$ and $(1-P_Y)^2 \le 1-P_Y$ since $P \in \triangle^{n-1}$, hence
    \begin{align*}
        \norm{\vec{P} - e_Y}_2^2 \le \sum_{j\neq Y}P_j + (1-P_Y) = 2(1-P_Y)
    \end{align*}
    Taking expectation yields
    \begin{align*}
        \E[]{\norm{\vec{P}-\vec{Q}}_2^2} \le 2\left(1-\E[]{P_Y}\right) = 2(1-\sigma).
    \end{align*}
    For the lower bound, by Cauchy--Schwarz on the $n-1$ coordinates $j \in [n]$ with $j\neq Y$,
    \begin{align*}
        \sum_{j\neq Y} P_j^2 \ge \frac{\left(\sum_{j\neq Y} P_j\right)^2}{n-1} = \frac{(1-P_Y)^2}{n-1}.
    \end{align*}
    Therefore,
    \begin{align*}
        \norm{\vec{P}-e_Y}_2^2 \ge (1-P_Y)^2 + \frac{(1-P_Y)^2}{n-1} = \frac{n}{n-1}(1-P_Y)^2.
    \end{align*}
    Taking expectation and using Jensen's inequality,
    \begin{align*}
        \E[]{\norm{\vec{P}-\vec{Q}}_2^2}
        \ge \frac{n}{n-1}\E[]{(1-P_Y)^2}
        \ge \frac{n}{n-1}\left( 1-\E[]{P_Y} \right)^2
        = \frac{n}{n-1}(1-\sigma)^2.
    \end{align*}
    Combining both bounds completes the proof.
\end{proof}

\begin{figure}[t]
    \centering
    \input{figures/gradient_norm/app}
    \[
        \log_{10}\left(\frac{\norm{\grad_{\theta_{\mathrm{sh}}}\mathcal{L}_{\mathrm{img}}}_2}{\norm{\grad_{\theta_{\mathrm{sh}}}\mathcal{L}_{\mathrm{txt}}}_2 + \epsilon}\right)
    \]
    \caption{\small Log-ratio between the expected MSE gradient norm (image branch, \Cref{theorem:mse}) and the lower/upper bounds on the cross-entropy gradient norm (text branch, \Cref{theorem:CE}), as a function of the agreement levels $\sigma_{\mathrm{img}}$ and $\sigma_{\mathrm{txt}}$ for latent dimension $N=4096$. Negative regions, which dominate early training when text predictions are poor and the image prior is already accurate, indicate that the text-branch gradient outweighs the image-branch gradient and motivates the loss reweighting in \Cref{sec:stabilizers}.}
    \label{fig:gradnorm_ratio}
\end{figure}
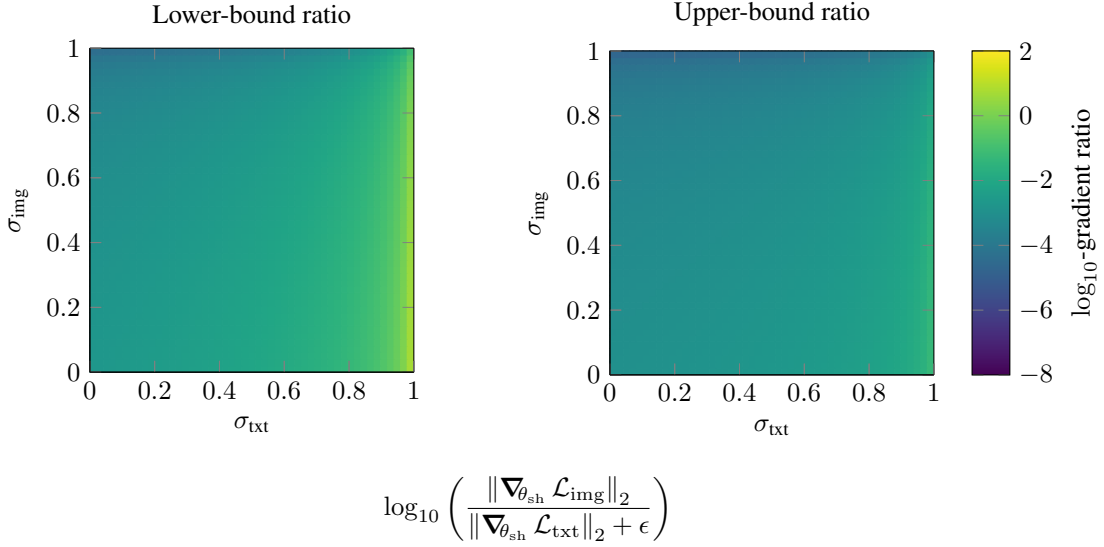

Theorems~\ref{theorem:mse} and~\ref{theorem:CE} highlight a fundamental scale mismatch between the two loss terms. For the image regression objective, the squared gradient w.r.t.\ the prediction scales as $\E[]{\norm{\nabla \mathrm{MSE}}_2^2} = \frac{8(1-\sigma_{\text{img}})}{n}$ and therefore decays as $O(1/n)$. In our setting, SD3 is trained on $256\times256$ images and operates in a VAE latent space of size $n = 4 \times 32 \times 32 = 4096$, making this effect non-negligible.

In contrast, for token cross-entropy with one-hot targets, the squared gradient w.r.t.\ logits is bounded by $\frac{n}{n-1}(1-\sigma_{\text{txt}})^2 \le \E[]{\norm{\nabla \mathrm{CE}}_2^2} \le 2(1-\sigma_{\text{txt}})$, which is $O(1)$ and essentially independent of the vocabulary size for large $n$. Early in training, text prediction is poor, so $\sigma_{\text{txt}}\approx 0$, while the pretrained image branch is already accurate, yielding $\sigma_{\text{img}}\approx 1$; consequently, the text objective can induce substantially larger gradients than the image objective on shared parameters (after applying the chain rule). \Cref{fig:gradnorm_ratio} visualizes this disparity by plotting the $\log_{10}$-ratio between the expected image
gradient norm and the lower/upper bounds for the text gradient norm.

While the assumptions of Theorems~\ref{theorem:mse} and~\ref{theorem:CE} are idealized, they provide a useful
explanation for why explicit loss balancing is required in practice. Consistent with this analysis, our adaptive
weighting in \cref{sec:stabilizers} stabilizes to $\lambda_{\text{txt}}\approx 0.04$ over training, indicating a persistent
order-of-magnitude correction between the two objectives.

%% file: figures/gradient_norm/app.tex
\begin{tikzpicture}
    \pgfmathsetmacro{\Ndim}{4096}

    \begin{axis}[
        name=left,
        width=0.42\textwidth,
        height=0.42\textwidth,
        view={0}{90},
        xmin=0, xmax=1.0,
        ymin=0, ymax=1.0,
        domain=0:1.0,
        y domain=0:1.0,
        samples=50,
        samples y=50,
        xlabel={$\sigma_{\text{txt}}$},
        ylabel={$\sigma_{\text{img}}$},
        title={Lower-bound ratio},
        colormap/viridis,
        point meta min=-8,
        point meta max=2,
    ]
        \addplot3[
            surf,
            shader=flat corner,
        ]
        {log10((8*(1-y)/\Ndim) / (((\Ndim/(\Ndim-1))*(1-x)^2) + 0.00000001))};
    \end{axis}

    \begin{axis}[
        at={(left.outer east)},
        anchor=outer west,
        xshift=1.2cm,
        width=0.42\textwidth,
        height=0.42\textwidth,
        view={0}{90},
        xmin=0, xmax=1.0,
        ymin=0, ymax=1.0,
        domain=0:1.0,
        y domain=0:1.0,
        samples=50,
        samples y=50,
        xlabel={$\sigma_{\text{txt}}$},
        ylabel={$\sigma_{\text{img}}$},
        title={Upper-bound ratio},
        colormap/viridis,
        point meta min=-8,
        point meta max=2,
        colorbar,
        colorbar style={
            ylabel={$\log_{10}$-gradient ratio},
        },
    ]
        \addplot3[
            surf,
            shader=flat corner,
        ]
        {log10((8*(1-y)/\Ndim) / (2*(1-x)))};
    \end{axis}
\end{tikzpicture}

%% file: appendix/architecture.tex
\section{Architecture}\label{app:architecture}

\tinytit{Stable Diffusion 3.}
Our SD3 adaptation follows \Cref{sec:peft} and \Cref{fig:mmdit_overview}. The original MM-DiT blocks are conditioned on the image timestep through AdaLN-style modulation. We add a second timestep input for the text corruption level $\tau$: image time $t$ and text time $\tau$ are embedded separately, passed through MLPs, and injected through the same modulation interface. Apart from this dual-timestep conditioning and the LoRA-augmented linear layers shown in \Cref{fig:mmdit_block}, the SD3 transformer architecture is unchanged.

\begin{figure}[t]
    \centering
    \resizebox{0.5\textwidth}{!}{\input{figures/architecture/block}}
    \caption{\small MM-DiT block adapted from Stable Diffusion~3 for joint image--text generation. We add text-time conditioning $\tau$ alongside the original image-time conditioning $t$, with both timesteps embedded separately and injected through the existing AdaLN-style modulation interface. Yellow blocks denote LoRA-augmented linear layers; green blocks denote frozen auxiliary operations.}
    \label{fig:mmdit_block}
\end{figure}
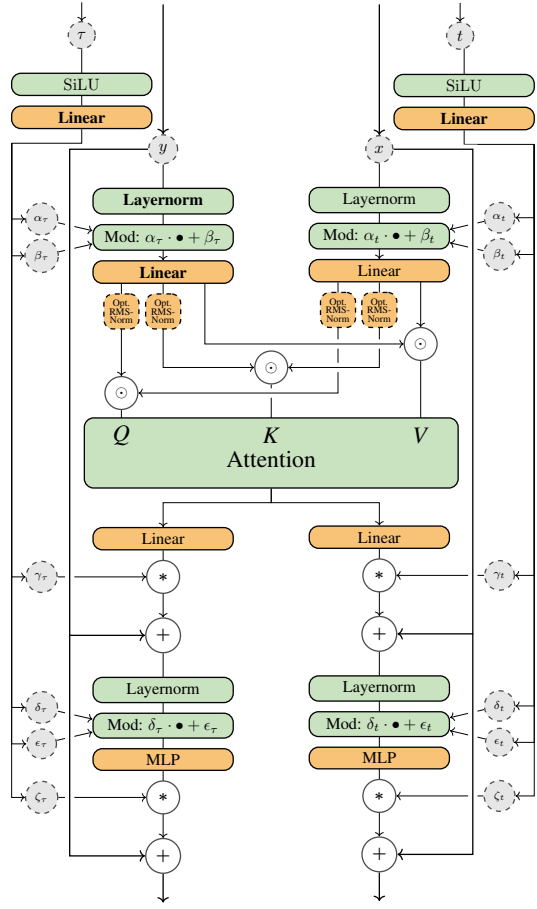

\tinytit{FLUX.1 [dev].}
The FLUX.1 [dev] adaptation uses the same dual-timestep principle. FLUX processes image tokens and T5 text-token embeddings in the main transformer, while also using pooled CLIP text embeddings for global conditioning. We therefore add separate processors for $t$ and $\tau$ and inject both through the model's existing modulation pathway. As in SD3, the backbone is frozen except for LoRA-augmented linear layers, while the text-generation heads and modality-specific timestep processors are newly trained; see \Cref{fig:flux_model}.

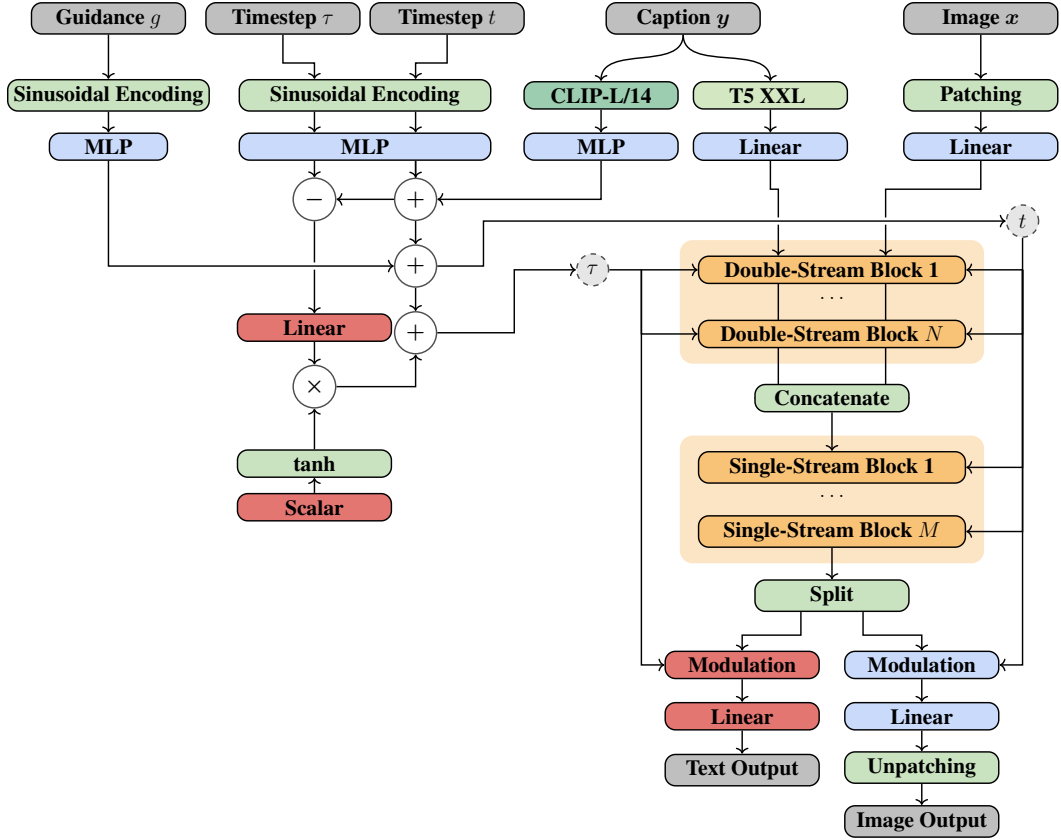
\begin{figure}[t]
    \centering
    \resizebox{\textwidth}{!}{\input{figures/architecture/flux_model}}
    \caption{\small Multimodal DiT architecture adapted from FLUX.1 [dev]. Noisy text tokens and image patches are processed jointly together with their modality-specific timesteps $(\tau, t)$ to predict cross-modal flows. Blue: frozen pretrained components; yellow: LoRA-augmented linear layers; red: newly added trainable modules (text head and timestep processors); green: auxiliary operations.}
    \label{fig:flux_model}
\end{figure}

\section{Multi-encoder Conditioning}\label{app:multiencoder}
Modern text-to-image backbones often combine token-level text features with pooled or global text embeddings for conditioning. In SD3, these signals come from multiple frozen encoders, including T5-XXL and CLIP; in FLUX.1, the main text stream uses T5 embeddings while pooled CLIP embeddings provide additional global conditioning. During multimodal uplift, we therefore need to provide compatible inputs to each conditioning pathway while defining the discrete text process in a single tokenizer space.

\tinytit{Design choice.}
We perform discrete text diffusion in the T5 tokenizer space. This keeps the generative state aligned with a text encoder that is trained for token-level modeling, while still allowing us to reuse the pretrained CLIP encoders as additional conditioning signals.

\tinytit{Key idea: decode \texorpdfstring{$\tilde{\vec{y}}_\tau$}{y~tau} to a natural-language string and re-tokenize.}
At training (and sampling) time we have a partially observed/deleted T5 token sequence $\tilde{\vec{y}}_\tau\sim q_\tau(\cdot\mid \vec{y})$.
Because our corruption is \emph{deletion-to-empty} (not mask-token corruption), $\tilde{\vec{y}}_\tau$ is a subsequence of the original prompt and typically detokenizes into a reasonably natural fragment. We exploit this to obtain inputs for auxiliary encoders:
\begin{enumerate}[leftmargin=*, itemsep=0pt, topsep=0pt]
    \item \textbf{T5 branch (diffusion alphabet):} feed $\tilde{\vec{y}}_\tau$ directly to the frozen T5 encoder to obtain token embeddings for the shared transformer.
    \item \textbf{CLIP branches (auxiliary conditioning):} decode $\tilde{\vec{y}}_\tau$ to a string $\tilde{s}_\tau=\mathrm{decode}_{\text{T5}}(\tilde{\vec{y}}_\tau)$ (dropping special tokens), then re-tokenize $\tilde{s}_\tau$ with each CLIP tokenizer and feed the resulting IDs to the corresponding frozen CLIP encoders (including pooled embeddings).
\end{enumerate}
This yields a simple, deterministic, and cheap mapping between encoder stacks while keeping all text encoders frozen.

\tinytit{Why deletion helps.}
Mask-based discrete diffusion injects artificial symbols (e.g., \texttt{<mask>}) and positional patterns that the pretrained CLIP encoders were never trained on, often requiring explicit mask-token retraining. In contrast, deletion produces strings composed of \emph{valid} natural-language tokens; even when the fragment is ungrammatical, CLIP-style encoders are typically robust enough for conditioning.

\tinytit{Limitations.}
The decoded fragment $\tilde{s}_\tau$ can be incomplete or slightly ungrammatical, especially at small $\tau$. In our setup this is acceptable because (i) these auxiliary encoders provide \emph{conditioning} rather than a supervised target, and (ii) the T5 branch remains the primary text state that the model learns to denoise. \Cref{tab:multiencoder_examples} shows typical decoded fragments produced by deletion at different $\tau$. These decoded strings are the inputs re-tokenized for auxiliary encoders.

\begin{table}[t]
    \centering
    \caption{\small Multi-encoder bridging via deletion at four retention levels $\tau$. The original prompt $s$ is corrupted in T5-tokenizer space to $\tilde{\vec{y}}_\tau$, decoded back into a string $\tilde{s}_\tau$, and re-tokenized for auxiliary encoders such as CLIP. Lower $\tau$ retains fewer tokens, but the surviving fragments remain natural-language tokens rather than mask symbols.}

    \vspace{0.5em}
    
    \setlength{\tabcolsep}{3pt}
    \renewcommand{\arraystretch}{1.15}
    \small
    \begin{tabularx}{\linewidth}{@{} c|p{0.45\linewidth} p{0.45\linewidth}@{}}
    \toprule
    $\boldsymbol{\tau}$ & \textbf{Original prompt $s$} & \textbf{Decoded fragment $\tilde{s}_\tau$} \\
    \midrule
    0.8 & A vintage green automobile with a distinctive grille and round headlights is parked on a paved surface, surrounded by greenery. & A vintage automobile witha distinctivee and roundlight isparked on a, surrounded by greenery.\\
    0.8 & A close-up image of a layered cake with whipped cream and fresh raspberries on top, with a slice being cut out and a spatula resting on the side. & A close-up image of a layered cake withwhipped cream and raspberries on top with a slice cut out and a spatul resting on the side.\\
    0.8 & The image features a repeating pattern of green holly leaves and red berries on a light green background, with some scattered dots. & The image features a repeat pattern of green holly leaves and red berries on  light green background, with some scattered dots.\\
    0.8 & Four well-dressed men in suits are posing for a photo, with one holding a glass of champagne. & Four well-dresse men in suits are posing for  photo, one holding a glass of champagne.\\
    \midrule
    0.6 & The image features a circular rug with a vibrant array of fruits, including bananas, apples, oranges, strawberries, grapes, pineapple, and watermelon, all depicted in a colorful and playful manner. & The image features a circular a vibrant array of fruits includings, apples, strawberries,, pineapplemel, all depicteda colorful and manner\\
    0.6 & A red enamel coffee pot with a white interior sits on a wooden surface, its handle pointing towards the right. & A enamel pot with a whites on  wooden, its  the right.\\
    0.6 & A watercolor illustration depicts a jar of jam with a purple ribbon tied around its neck, set against a white background with splashes of red and pink. & watercolor of jam with a purple tied its neck, set against a white splashes of red and pink.\\
    0.6 & The image depicts a whimsical beach scene with a vintage camper, a dog, and beach paraphernalia, all set against a backdrop of a cloudy sky. & images  beach scene with vintage camper,a dog beach paran, set against ay\\
    \midrule
    0.4 & A man is holding a blue and white device, possibly a blood glucose meter, and appears to be checking his finger for a drop of blood. & A white, possiblyameter finger  drop\\
    0.4 & The image shows a vintage wooden cabinet with glass doors, featuring ornate carvings and brass hardware. & The shows cabinet with, featuringnate and brass hardware.\\
    0.4 & The image features a vintage side table with a wooden top and a shelf, set against a plain white background. & image  table with wooden topa shelf, against \\
    0.4 & The image shows a wooden platter filled with a variety of fresh fruits and vegetables, including strawberries, asparagus, lemon wedges, avocado, and grilled chicken. & The image with variety wedge avocado, and grilled\\
    \midrule
    0.2 & Three cupcakes with colorful frosting are displayed on a white surface, with the central cupcake having a swirl of chocolate and vanilla frosting. & central a of and.\\
    0.2 & The image features a slice of cake with intricate frosting on a plate, with a larger cake in the background on a stand. & with in\\
    0.2 & The image shows the intricate inner workings of a vintage gold watch, featuring ornate engravings and a central blue gemstone. & engraving\\
    0.2 & The image shows a colorful, intricately embroidered garment with a vibrant pattern, displayed on a mannequin. & showsa colorful  onamann\\
    \bottomrule
    \end{tabularx}
    \label{tab:multiencoder_examples}
\end{table}

%% file: figures/architecture/block.tex
\begin{tikzpicture}[
        block/.style={draw, fill=white, rectangle, minimum width=3cm,rounded corners=0.2cm},
        trainableblock/.style={block,fill=dit_yellow!70,thick},
        lorablock/.style={block,fill=dit_yellow!70,thick},
        frozenblock/.style={block,fill=dit_lightblue,thick},
        fnblock/.style={block,fill=dit_lightgreen,thick},
        operation/.style={draw=black!70, fill=white, circle, minimum width=2em, thick},
        tensor/.style={draw=black!70, fill=dit_grey!20, circle, thick, dashed},
        input/.style={block,fill=dit_grey!50,thick},
        outlined arrow/.style={preaction={draw, white, line width=4pt}}
    ]

    \pgfdeclarelayer{arrowlayer}
    \pgfdeclarelayer{residuallayer}
    \pgfsetlayers{arrowlayer,residuallayer,main}

    \newcommand{\inblockdistance}{0.2cm}
    \newcommand{\outblockdistance}{0.5cm}
    \newcommand{\attnskip}{5.2cm}
    \newcommand{\leftmodpos}{-2.6}
    \newcommand{\rightmodpos}{7.2}


    \node [tensor] (c) {$y$};
    \node [tensor, right=4cm of c] (x) {$x$};

    \node [below=\outblockdistance of c, fnblock] (c_ln1) {\textbf{Layernorm}};
    \node [below=\inblockdistance of c_ln1, fnblock] (c_mod) {Mod: $\alpha_\tau \cdot \bullet + \beta_\tau$};
    \node [below=\inblockdistance of c_mod, trainableblock] (c_lin) {\textbf{Linear}};

    \path let \p1=($(x)!0.5!(c)$) in node [fnblock, minimum width=8cm, minimum height=1.5cm, align=center] at (\x1, -6.5) (attn) {\vphantom{Attention}\\ \Large Attention};
    \coordinate [below=0.3cm of attn] (attn_split);

    \node [below=\attnskip of c_lin, trainableblock] (c_proj) {Linear};
    \node [below=\inblockdistance of c_proj, operation] (c_proj_gated) {$\bm{*}$};
    \node [below=\outblockdistance of c_proj_gated, operation] (c_attn_out) {$\bm{+}$};
    \node [below=\outblockdistance of c_attn_out, fnblock] (c_ln2) {Layernorm};
    \node [below=\inblockdistance of c_ln2, fnblock] (c_mod2) {Mod: $\delta_\tau \cdot \bullet + \epsilon_\tau$};
    \node [below=\inblockdistance of c_mod2, trainableblock] (c_mlp) {MLP};
    \node [below=\inblockdistance of c_mlp, operation] (c_mlp_gated) {$\bm{*}$};
    \node [below=\outblockdistance of c_mlp_gated, operation] (c_out) {$\bm{+}$};

    \path let \p1=($(c_mod.north west)!0.5!(c_mod.west)+(0, 0.25)$) in node [tensor] at (\leftmodpos, \y1) (alpha_txt) {$\scriptstyle \alpha_\tau$};
    \path let \p1=($(c_mod.south west)!0.5!(c_mod.west)-(0, 0.25)$) in node [tensor] at (\leftmodpos, \y1) (beta_txt) {$\scriptstyle \beta_\tau$};
    \path let \p1=(c_proj_gated.west) in node [tensor] at (\leftmodpos, \y1) (gamma_txt) {$\scriptstyle \gamma_\tau$};
    \path let \p1=($(c_mod2.north west)!0.5!(c_mod2.west)+(0, 0.25)$) in node [tensor] at (\leftmodpos, \y1) (delta_txt) {$\scriptstyle \delta_\tau$};
    \path let \p1=($(c_mod2.south west)!0.5!(c_mod2.west)-(0, 0.25)$) in node [tensor] at (\leftmodpos, \y1) (epsilon_txt) {$\scriptstyle \epsilon_\tau$};
    \path let \p1=(c_mlp_gated.west) in node [tensor] at (\leftmodpos, \y1) (zeta_txt) {$\scriptstyle \zeta_\tau$};

    \node [tensor, above left=2cm and 1.3cm of c] (pi_y) {$\tau$};
    \node [tensor, above right=2cm and 1.3cm of x] (pi_img) {$t$};
    \node [fnblock, below=\outblockdistance of pi_y] (y_silu) {$\mathrm{SiLU}$};
    \node [trainableblock, below=\inblockdistance of y_silu] (y_linear) {\textbf{Linear}};
    \coordinate (y_split) at ($(y_linear.south)+(-1.5, -0.3)$);

    \coordinate (c_attn_out_left) at ($(c_attn_out) + (-2, 0)$);

    \begin{pgfonlayer}{arrowlayer}
        \draw [->] (c) -- (c_lin);
        \draw [->] (attn) -- (attn_split) -| (c_proj);
        \draw [->] (c_proj) -- (c_attn_out);
        \draw [->] (c_attn_out) -- (c_out);

        \draw [<-] (pi_y.north) --++ (0, 0.4);
        \draw [<-] (pi_img.north) --++ (0, 0.4);
        \draw (pi_y) -- (y_linear) |- (y_split);
        \draw [->] (y_split) |- (alpha_txt);
        \draw [->] (y_split) |- (beta_txt);
        \draw [->] (y_split) |- (gamma_txt);
        \draw [->] (y_split) |- (delta_txt);
        \draw [->] (y_split) |- (epsilon_txt);
        \draw [->] (y_split) |- (zeta_txt);

        \draw [->] (alpha_txt) -- ($(c_mod.north west)!0.5!(c_mod.west)$);
        \draw [->] (beta_txt) -- ($(c_mod.south west)!0.5!(c_mod.west)$);
        \draw [->] (gamma_txt) -- (c_proj_gated.west);
        \draw [->] (delta_txt) -- ($(c_mod2.north west)!0.5!(c_mod2.west)$);
        \draw [->] (epsilon_txt) -- ($(c_mod2.south west)!0.5!(c_mod2.west)$);
        \draw [->] (zeta_txt) -- (c_mlp_gated.west);

    \end{pgfonlayer}
    \begin{pgfonlayer}{residuallayer}
        \draw [<-, line width=2mm, draw=white] (c) --++ (0, 3);
        \draw [->, line width=2mm, draw=white] (c) -| (c_attn_out_left) -- (c_attn_out);
        \draw [->, line width=2mm, draw=white] (c_attn_out_left) |- (c_out);
        \draw [->, line width=2mm, draw=white] (c_out) --++ (0, -1);
        \draw [<-, thick] (c) --++ (0, 3.1);
        \draw [->, thick] (c) -| (c_attn_out_left) -- (c_attn_out);
        \draw [->, thick] (c_attn_out_left) |- (c_out);
        \draw [->, thick] (c_out) --++ (0, -1);
    \end{pgfonlayer}


    \node [below=\outblockdistance of x, fnblock] (x_ln1) {Layernorm};
    \node [below=\inblockdistance of x_ln1, fnblock] (x_mod) {Mod: $\alpha_t \cdot \bullet + \beta_t$};
    \node [below=\inblockdistance of x_mod, trainableblock] (x_lin) {Linear};

    \node [below=\attnskip of x_lin, trainableblock] (x_proj) {Linear};
    \node [below=\inblockdistance of x_proj, operation] (x_proj_gated) {$\bm{*}$};
    \node [below=\outblockdistance of x_proj_gated, operation] (x_attn_out) {$\bm{+}$};
    \node [below=\outblockdistance of x_attn_out, fnblock] (x_ln2) {Layernorm};
    \node [below=\inblockdistance of x_ln2, fnblock] (x_mod2) {Mod: $\delta_t \cdot \bullet + \epsilon_t$};
    \node [below=\inblockdistance of x_mod2, trainableblock] (x_mlp) {MLP};
    \node [below=\inblockdistance of x_mlp, operation] (x_mlp_gated) {$\bm{*}$};
    \node [below=\outblockdistance of x_mlp_gated, operation] (x_out) {$\bm{+}$};

    \path let \p1=($(x_mod.north east)!0.5!(x_mod.east)+(0, 0.25)$) in node [tensor] at (\rightmodpos , \y1) (alpha_img) {$\scriptstyle \alpha_t$};
    \path let \p1=($(x_mod.south east)!0.5!(x_mod.east)-(0, 0.25)$) in node [tensor] at (\rightmodpos , \y1) (beta_img) {$\scriptstyle \beta_t$};
    \path let \p1=(x_proj_gated.east) in node [tensor] at ( \rightmodpos , \y1) (gamma_img) {$\scriptstyle \gamma_t$};
    \path let \p1=($(x_mod2.north east)!0.5!(x_mod2.east)+(0, 0.25)$) in node [tensor] at (\rightmodpos , \y1) (delta_img) {$\scriptstyle \delta_t$};
    \path let \p1=($(x_mod2.south east)!0.5!(x_mod2.east)-(0, 0.25)$) in node [tensor] at (\rightmodpos , \y1) (epsilon_img) {$\scriptstyle \epsilon_t$};
    \path let \p1=(x_mlp_gated.east) in node [tensor] at (\rightmodpos , \y1) (zeta_img) {$\scriptstyle \zeta_t$};

    \coordinate [above right=2cm and 1.6cm of x] (yx) {};
    \node [fnblock, below=0.6cm of yx] (y_silu) {$\mathrm{SiLU}$};
    \node [trainableblock, below=\inblockdistance of y_silu] (y_linear) {\textbf{Linear}};
    \coordinate (y_split) at ($(y_linear.south)+(1.5, -0.3)$);

    \coordinate (x_attn_out_left) at ($(x_attn_out) + (2, 0)$);

    \begin{pgfonlayer}{arrowlayer}
        \draw [->] (x) -- (x_lin);
        \draw [->] (attn) -- (attn_split) -| (x_proj);
        \draw [->] (x_proj) -- (x_attn_out);
        \draw [->] (x_attn_out) -- (x_out);

        \draw (pi_img) -| (yx) -- (y_linear) |- (y_split);
        \draw [->] (y_split) |- (alpha_img);
        \draw [->] (y_split) |- (beta_img);
        \draw [->] (y_split) |- (gamma_img);
        \draw [->] (y_split) |- (delta_img);
        \draw [->] (y_split) |- (epsilon_img);
        \draw [->] (y_split) |- (zeta_img);

        \draw [->] (alpha_img) -- ($(x_mod.north east)!0.5!(x_mod.east)$);
        \draw [->] (beta_img) -- ($(x_mod.south east)!0.5!(x_mod.east)$);
        \draw [->] (gamma_img) -- (x_proj_gated.east);
        \draw [->] (delta_img) -- ($(x_mod2.north east)!0.5!(x_mod2.east)$);
        \draw [->] (epsilon_img) -- ($(x_mod2.south east)!0.5!(x_mod2.east)$);
        \draw [->] (zeta_img) -- (x_mlp_gated.east);

    \end{pgfonlayer}
    \begin{pgfonlayer}{residuallayer}
        \draw [<-, line width=2mm, draw=white] (x) --++ (0, 3);
        \draw [->, line width=2mm, draw=white] (x) -| (x_attn_out_left) -- (x_attn_out);
        \draw [->, line width=2mm, draw=white] (x_attn_out_left) |- (x_out);
        \draw [->, line width=2mm, draw=white] (x_out) --++ (0, -1);
        \draw [<-, thick] (x) --++ (0, 3.1);
        \draw [->, thick] (x) -| (x_attn_out_left) -- (x_attn_out);
        \draw [->, thick] (x_attn_out_left) |- (x_out);
        \draw [->, thick] (x_out) --++ (0, -1);
    \end{pgfonlayer}

    \node [below=0.1cm of attn.north] (k) {\Large \emph{K}};
    \node [left= 2.9cm of k.center] (q) {\Large \emph{Q}};
    \node [right=2.9cm of k.center] (v) {\Large \emph{V}};
    \node [above=0.2cm of q, operation] (q_join) {$\odot$};
    \node [above=0.8cm of k, operation] (k_join) {$\odot$};
    \node [above=1.3cm of v, operation] (v_join) {$\odot$};

    \coordinate (x_lin_vout) at (x_lin.south -| v_join);
    \coordinate (x_lin_kout) at (x_lin.south);
    \coordinate (x_lin_qout) at ($(x_lin_kout)+(x_lin_kout)-(x_lin_vout)$);

    \coordinate (c_lin_qout) at (c_lin.south -| q_join);
    \coordinate (c_lin_kout) at (c_lin.south);
    \coordinate (c_lin_vout) at ($(c_lin_kout)+(c_lin_kout)-(c_lin_qout)$);

    \node [below=\inblockdistance of c_lin_qout, trainableblock, dashed, minimum width=0cm, align=center] (q_norm_txt) {\tiny Opt.\\[-0.2cm]\tiny RMS-\\[-0.2cm] \tiny Norm};
    \node [below=\inblockdistance of c_lin_kout, trainableblock, dashed, minimum width=0cm, align=center] (k_norm_txt) {\tiny Opt.\\[-0.2cm]\tiny RMS-\\[-0.2cm] \tiny Norm};
    \node [below=\inblockdistance of x_lin_qout, trainableblock, dashed, minimum width=0cm, align=center] (q_norm_img) {\tiny Opt.\\[-0.2cm]\tiny RMS-\\[-0.2cm] \tiny Norm};
    \node [below=\inblockdistance of x_lin_kout, trainableblock, dashed, minimum width=0cm, align=center] (k_norm_img) {\tiny Opt.\\[-0.2cm]\tiny RMS-\\[-0.2cm] \tiny Norm};

    \begin{pgfonlayer}{arrowlayer}
        \draw (k_join) -- (k);

        \draw [->] (c_lin_qout) -| (q_join);
        \draw [line width=2mm, white] (x_lin_qout) |- (q_join);
        \draw [->] (x_lin_qout) |- (q_join);
        \draw (q_join) -- (q);

        \draw [->] (c_lin_kout) |- (k_join);
        \draw [line width=2mm, white] (x_lin_kout) |- (k_join);
        \draw [->] (x_lin_kout) |- (k_join);

        \draw [line width=2mm, white] (c_lin_vout) |- (v_join);
        \draw [->] (c_lin_vout) |- (v_join);
        \draw [->] ($(x_lin)+(0.25, 0)$) -| (v_join);
        \draw [->] (v_join) -- (v);
    \end{pgfonlayer}
\end{tikzpicture}

%% file: figures/architecture/flux_model.tex
\begin{tikzpicture}[
        font=\fontsize{12}{14}\selectfont,
        every path/.style={line width=0.8pt},
        block/.style={draw, fill=white, rectangle, minimum width=3cm,rounded corners=0.2cm},
        trainableblock/.style={block,fill=dit_red!70,thick},
        lorablock/.style={block,fill=dit_yellow!70,thick},
        frozenblock/.style={block,fill=dit_lightblue,thick},
        fnblock/.style={block,fill=dit_lightgreen,thick},
        operation/.style={draw=black!70, fill=white, circle, minimum width=1.5em, thick},
        tensor/.style={draw=black!70, fill=dit_grey!20, circle, thick, dashed, minimum width=1.5em},
        input/.style={block,fill=dit_grey!50,thick},
        outlined arrow/.style={preaction={draw, white, line width=4pt}}
    ]
    \pgfdeclarelayer{arrowlayer}
    \pgfdeclarelayer{residuallayer}
    \pgfdeclarelayer{decorations}
    \pgfdeclarelayer{bg}
    \pgfsetlayers{bg,arrowlayer,residuallayer,main,decorations}
    
    \newcommand{\clipLcolor}{ForestGreen!40}
    \newcommand{\tfiveColor}{YellowGreen!40}
    
    \coordinate (timecol)  at (-6.2,0);
    \coordinate (gcol)     at (-11.2,0);
    \coordinate (textcol)  at (0,0);
    \coordinate (imagecol) at (5.8,0);
    \coordinate (stackcol) at (2.9,-4.9);
    
    \node[input] (text) at (textcol) {\textbf{Caption $\vy$}};
    \node[input] (image) at (imagecol) {\textbf{Image $\vx$}};
    
    \node[input] (guidance) at (gcol) {\textbf{Guidance $g$}};
    \node[input] (timestep_tau) at ($(timecol)+(-1.6,0)$) {\textbf{Timestep $\tau$}};
    \node[input] (timestep_t)   at ($(timecol)+(1.6,0)$) {\textbf{Timestep $t$}};
    
    \node[frozenblock, fill=\clipLcolor] (clipL) at ($(textcol)+(-1.6,-1.5)$)
        {\textbf{CLIP-L/14}};
    
    \node[frozenblock, fill=\tfiveColor] (t5) at ($(textcol)+(1.7,-1.5)$)
        {\textbf{T5 XXL\phantom{/}}};
    
    \node[frozenblock] (pool_mlp) at ($(clipL)+(0,-1.0)$)
        {\textbf{MLP}};
    
    \node[frozenblock] (y_lin) at ($(t5)+(0,-1.0)$)
        {\textbf{Linear}};
    
    \node[fnblock] (patching) at ($(imagecol)+(0,-1.5)$)
        {\textbf{Patching}};
    
    \node[frozenblock] (patch_proj) at ($(patching)+(0,-1.0)$)
        {\textbf{Linear}};
    
    \node[fnblock, minimum width=4.9cm] (t_enc) at ($(timecol)+(0,-1.5)$)
        {\textbf{Sinusoidal Encoding}};
    
    \node[frozenblock, minimum width=4.9cm] (t_emb) at ($(t_enc)+(0,-1.0)$)
        {\textbf{MLP}};
    
    \node[fnblock, minimum width=2.3cm, align=center] (g_enc) at ($(guidance)+(0,-1.5)$)
        {\textbf{Sinusoidal Encoding}};
    
    \node[frozenblock, minimum width=2.3cm] (g_mlp) at ($(g_enc)+(0,-1.0)$)
        {\textbf{MLP}};
    
    \coordinate (t_enc_left_in)   at ($(t_enc.north west)!0.30!(t_enc.north east)$);
    \coordinate (t_enc_right_in)  at ($(t_enc.north west)!0.70!(t_enc.north east)$);
    \coordinate (t_enc_left_out)  at ($(t_enc.south west)!0.30!(t_enc.south east)$);
    \coordinate (t_enc_right_out) at ($(t_enc.south west)!0.70!(t_enc.south east)$);
    
    \coordinate (t_emb_left_in)   at ($(t_emb.north west)!0.30!(t_emb.north east)$);
    \coordinate (t_emb_right_in)  at ($(t_emb.north west)!0.70!(t_emb.north east)$);
    \coordinate (t_emb_left_out)  at ($(t_emb.south west)!0.30!(t_emb.south east)$);
    \coordinate (t_emb_right_out) at ($(t_emb.south west)!0.70!(t_emb.south east)$);
    
    \node[operation] (delta_diff) at ($(t_emb_left_out)+(0,-0.75)$)
        {$\bm{-}$};
    
    \node[trainableblock] (delta_mlp) at ($(delta_diff)+(0,-2.5)$)
        {\textbf{Linear}};
    
    \node[operation] (delta_mult) at ($(delta_mlp)+(0,-1.15)$)
        {$\bm{\times}$};
    
    \node[fnblock] (delta_tanh) at ($(delta_mult)+(0,-1.5)$)
        {\textbf{tanh}};
    
    \node[trainableblock] (delta_scale) at ($(delta_tanh)+(0,-0.85)$)
        {\textbf{Scalar}};
    
    \node[operation] (t_plus) at ($(t_emb_right_out)+(0,-0.75)$)
        {$\bm{+}$};
    
    \node[operation] (g_plus) at ($(t_plus)+(0,-1.3)$)
        {$\bm{+}$};
    
    \node[operation] (tau_update) at ($(g_plus)+(0,-1.3)$)
        {$\bm{+}$};
    
    \node[lorablock, minimum width=5.2cm] (attn1) at (stackcol)
        {\textbf{Double-Stream Block 1}};
    
    \node at ($(attn1)+(0,-0.55)$) {$\ldots$};
    
    \node[lorablock, minimum width=5.2cm] (attnN) at ($(attn1)+(0,-1.25)$)
        {\textbf{Double-Stream Block $N$}};
    
    \node[fnblock] (concat) at ($(attnN)+(0,-1.25)$)
        {\textbf{Concatenate}};
    
    \node[lorablock, minimum width=5.2cm] (single1) at ($(concat)+(0,-1.35)$)
        {\textbf{Single-Stream Block 1}};
    
    \node at ($(single1)+(0,-0.55)$) {$\ldots$};
    
    \node[lorablock, minimum width=5.2cm] (singleM) at ($(single1)+(0,-1.25)$)
        {\textbf{Single-Stream Block $M$}};
    
    
    
    

    
    

    \node[fnblock] (split) at ($(singleM)+(0,-1.25)$)
    {\textbf{Split}};

    \coordinate (split_out_left)  at ($(split.south west)!0.30!(split.south east)$);
    \coordinate (split_out_right) at ($(split.south west)!0.70!(split.south east)$);
    
    \node[trainableblock] (out_mod_txt) at ($(split)+(-1.75,-1.35)$)
        {\textbf{Modulation}};
    
    \node[trainableblock] (out_mlp_txt) at ($(out_mod_txt)+(0,-1.0)$)
        {\textbf{Linear}};
    
    \node[input] (output_txt) at ($(out_mlp_txt)+(0,-1.05)$)
        {\textbf{Text Output}};
    
    \node[frozenblock] (out_mod) at ($(split)+(+1.75,-1.35)$)
        {\textbf{Modulation}};
    
    \node[frozenblock] (out_mlp) at ($(out_mod)+(0,-1.0)$)
        {\textbf{Linear}};
    
    \node[fnblock] (depatch) at ($(out_mlp)+(0,-1.0)$)
        {\textbf{Unpatching}};
    
    \node[input] (output) at ($(depatch)+(0,-1.05)$)
        {\textbf{Image Output}};
    
    \coordinate (attn1_mid_left)  at ($(attn1.north west)!0.30!(attn1.north east)$);
    \coordinate (attn1_mid_right) at ($(attn1.north west)!0.70!(attn1.north east)$);
    
    \coordinate (attnN_out_left)  at ($(attnN.south west)!0.30!(attnN.south east)$);
    \coordinate (attnN_out_right) at ($(attnN.south west)!0.70!(attnN.south east)$);
    
    \coordinate (left_bus)  at ($(attn1.west)+(-1.1,0)$);
    \coordinate (right_bus) at ($(attn1.east)+(1.1,0)$);
    
    \node[tensor] (tau_final) at ($(left_bus)+(-0.95,0)$) {$\tau$};
    \node[tensor] (t_final) at ($(right_bus)+(0,0.95)$) {$t$};
    
    \begin{pgfonlayer}{bg}
        \path[fill=dit_yellow!30, rounded corners=0.25cm]
            ($(attn1.north west)+(-0.35,0.30)$)
            rectangle
            ($(attnN.south east)+(0.35,-0.30)$);
    
        \path[fill=dit_yellow!30, rounded corners=0.25cm]
            ($(single1.north west)+(-0.35,0.30)$)
            rectangle
            ($(singleM.south east)+(0.35,-0.30)$);
    \end{pgfonlayer}
    
    \begin{pgfonlayer}{arrowlayer}
    
        \draw[->] (text.south) to[out=270,in=90] (clipL.north);
        \draw[->] (text.south) to[out=270,in=90] (t5.north);
        \draw[->] (clipL) -- (pool_mlp);
        \draw[->] (t5) -- (y_lin);
    
        \draw[->] (image) -- (patching);
        \draw[->] (patching) -- (patch_proj);
    
        \draw[->] (guidance) -- (g_enc);
        \draw[->] (g_enc) -- (g_mlp);
        
        \draw[->] (timestep_tau.south) -- ++(0,-0.55) -| (t_enc_left_in);
        \draw[->] (timestep_t.south) -- ++(0,-0.55) -| (t_enc_right_in);
    
        \draw[->] (t_enc_left_out)  -- (t_emb_left_in);
        \draw[->] (t_enc_right_out) -- (t_emb_right_in);
    
        \draw[->] (t_emb_left_out) -- (delta_diff);
        \draw[->] (t_emb_right_out) |- (delta_diff.east);
    
        \draw[->] (delta_diff) -- (delta_mlp);
        \draw[->] (delta_mlp) -- (delta_mult);
    
        \draw[->] (delta_scale) -- (delta_tanh);
        \draw[->] (delta_tanh) -- (delta_mult);
    
        \draw[->] (t_emb_right_out) -- (t_plus);
        \draw[->] (pool_mlp.south) |- (t_plus.east);
    
        \draw[->] (t_plus) -- (g_plus);
        \draw[->] (g_plus) -- (tau_update);
        \draw[->] (delta_mult) -| (tau_update.south);
    
        \draw[->] (tau_update.east) -- ++(1.5,0) |- (tau_final.west);

        \draw[->, outlined arrow] (g_mlp.south) |- (g_plus.west);
    
        \draw[->] (y_lin.south) -- ++(0,-0.55) -| (attn1_mid_left);
        \draw[->] (patch_proj.south) -- ++(0,-0.55) -| (attn1_mid_right);
    
        \draw[->] (attn1_mid_left)  |- (concat.west); 
        \draw[->] (attn1_mid_right) |- (concat.east); 
    
        \draw[->] (concat) -- (single1);
        \draw[->] (singleM) -- (split);
        
        \draw[->] (split_out_left)  -- ++(0,-0.45) -| (out_mod_txt.north);
        \draw[->] (split_out_right) -- ++(0,-0.45) -| (out_mod.north);
        
        \draw[->] (out_mod_txt) -- (out_mlp_txt);
        \draw[->] (out_mlp_txt) -- (output_txt);
        
        \draw[->] (out_mod) -- (out_mlp);
        \draw[->] (out_mlp) -- (depatch);
        \draw[->] (depatch) -- (output);
    
        \draw[->] (tau_final) -- (left_bus) -- (attn1.west);
        \draw[->] (left_bus) |- (attnN.west);
        \draw[->] (left_bus) |- (out_mod_txt.west);
    
        \draw[->, outlined arrow] (g_plus.east) -- ++(0.9,0) |- (t_final.west);
    
        \draw[->] (t_final.south) -- (right_bus) -- (attn1.east);
        \draw[->] (right_bus) |- (attnN.east);
        \draw[->] (right_bus) |- (single1.east);
        \draw[->] (right_bus) |- (singleM.east);
        \draw[->] (right_bus) |- (out_mod.east);
    
    \end{pgfonlayer}
\end{tikzpicture}

%% file: appendix/results.tex
\section{Results}
This appendix section presents additional ablation studies and analysis beyond the main paper results. We examine the impact of different architectural and training choices, including the effect of classifier-free guidance (CFG), teacher-matching strategies, and the corruption schedule used during training. These results complement the main paper evaluation and provide deeper insights into the design decisions that contribute to FullFlow's strong performance.

\subsection{Ablation 1: Classifier-Free Guidance}
Training with the $\pi_{\mathrm{ind}}$ schedule (\Cref{eq:pi_ind}) exposes the model to jointly corrupted states across the full two-timestep space. In particular, the model sees regimes where one modality carries little or no information while the other remains informative. This provides unconditional or near-unconditional predictors for both the image and text branches without adding a separate conditioning-dropout objective. At inference time, these predictors can be combined with their conditional counterparts to implement classifier-free guidance (CFG) for either text-to-image or image-to-text generation.

Dual Diffusion~\citep{li_dual_2025} instead introduces an explicit conditioning-dropout hyperparameter, which controls the probability of fully deleting the conditioning modality during training in order to preserve an unconditional prediction branch. In our formulation, the same functionality arises naturally from the two-timestep training space induced by $\pi_{\mathrm{ind}}$.

\tinytit{Text-to-image.}
For text-to-image generation we condition on a clean caption $\vy$ and evolve the image from Gaussian noise $\vx_0 \sim \SN$ along $t:0\to1$. We define
\begin{align*}
    \vy_1 \coloneq \vy, \qquad \vy_0 \coloneq \varepsilon,
\end{align*}
corresponding to fully informative vs.\ fully uninformative text contexts. The model provides
\begin{align*}
    v_\theta^{\text{cond}}
    &\coloneq v_\theta(\vx_t, t, \vy_1, \tau{=}1),\\
    v_\theta^{\text{uncond}}
    &\coloneq v_\theta(\vx_t, t, \vy_0, \tau{=}0),
\end{align*}
and we use the standard CFG combination
\begin{align*}
v_{\text{cfg}} = v_\theta^{\text{uncond}} + \gamma \cdot \left( v_\theta^{\text{cond}} - v_\theta^{\text{uncond}} \right),
\end{align*}
with guidance scale $\gamma \ge 1$. This directly steers the image flow toward captions that better match $\vy$ without extra networks or retraining.

\tinytit{Image-to-text.}
The same construction applies symmetrically to image-to-text generation, but now the discrete nature of Edit Flows requires combining the parameters of the reverse-time CTMC rather than a Euclidean velocity. We condition on a clean image $\vx_1 \coloneq \vx$ and evolve the text from the empty sequence $\varepsilon$ along $\tau:0\to1$, using
\begin{align*}
    \vx_1 \coloneq \vx, \qquad \vx_0 \sim \SN,
\end{align*}
as fully informative vs.\ fully uninformative image contexts. At each retained position $i$, the backbone produces a conditional and an unconditional pair of insertion rate and token distribution:
\begin{align*}
    \lambda_\theta^{i,\text{cond}} \coloneq \lambda_\theta^{i}\left(\tilde{\vy}_\tau, \tau, \vx_1, t{=}1\right), & \quad
    w_\theta^{i,\text{cond}}(\cdot) \coloneq w_\theta^{i}\left(\cdot \mid \tilde{\vy}_\tau, \tau, \vx_1, t{=}1\right),\\    
    \lambda_\theta^{i,\text{uncond}} \coloneq \lambda_\theta^{i}\left(\tilde{\vy}_\tau, \tau, \vx_0, t{=}0\right) & \quad
    w_\theta^{i,\text{uncond}}(\cdot) \coloneq w_\theta^{i}\left(\cdot \mid \tilde{\vy}_\tau, \tau, \vx_0, t{=}0\right).
\end{align*}
Since insertion rates are nonnegative and token distributions are categorical, we combine them multiplicatively in log-space, which is the discrete analogue of the additive velocity combination above:
\begin{align*}
    \log \lambda_\theta^{i,\text{cfg}}
    &= \log \lambda_\theta^{i,\text{uncond}} + \gamma \cdot \left( \log \lambda_\theta^{i,\text{cond}} - \log \lambda_\theta^{i,\text{uncond}} \right),\\
    w_\theta^{i,\text{cfg}}(v) &\propto w_\theta^{i,\text{uncond}}(v)^{1-\gamma} \cdot w_\theta^{i,\text{cond}}(v)^{\gamma},
\end{align*}
for each candidate token $v \in V$, with $w_\theta^{i,\text{cfg}}$ renormalized over the vocabulary. The guided rate $\lambda_\theta^{i,\text{cfg}}$ and distribution $w_\theta^{i,\text{cfg}}$ are then plugged into the same insertion sampler used at vanilla inference, steering the discrete text trajectory toward sequences that are more strongly explained by the conditioning image $\vx$ at no additional training cost.

\input{figures/cfg/fig}

We evaluate the effect of image-to-text CFG on caption generation in \Cref{fig:cfg_metrics}. Unless stated otherwise, all main results in this paper use unguided decoding, i.e.\ $\gamma=1$, for simplicity and to avoid introducing an additional inference-time hyperparameter. Nevertheless, CFG provides a useful post-hoc control mechanism. A mild guidance scale, $\gamma=1.2$, improves CIDEr from $107.1$ to $110.0$ while slightly shortening the generated captions from $34.4$ to $33.8$ tokens on average, bringing them closer to the reference length of $35.9$ tokens. However, stronger guidance quickly becomes detrimental: as $\gamma$ increases beyond $1.5$, the model generates increasingly long captions, reaching $57.8$ tokens at $\gamma=4.0$, while CIDEr drops monotonically to $15.0$. This suggests that moderate CFG can improve caption-image alignment, but excessive guidance over-amplifies high-probability insertions in the reverse CTMC, producing overly verbose captions and degrading caption quality.

\subsection{Ablation 2: Teacher-matching}
Qualitative examples of the effect of different teacher-matching styles vs the standard rectified flow formulation. As confirmed by the metrics presented in the main part, the standard rectified flow formulation yields high distortion which is likely the result of a training distribution mismatch between our training data and the data the model was pretrained on. Both teacher-matching versions are able to provide coherent images. However, clean teacher-matching introduces a lot of artifacts at first and seems overall more unstable, while same-noise teacher-matching stays very stable, and is hence the standard choice.

\input{figures/sd3_t2i/fig}

\subsection{Ablation 3: Corruption Schedule Ablation: Extended Metrics}\label{app:time_sampling_extended}

\Cref{fig:timesamplingextended} extends the corruption-schedule ablation from \Cref{sec:time_coupling} to four additional metrics: CLIP score (image--text alignment), T5-token union (lexical overlap between the generated and reference captions), BERTScore-F1 (semantic similarity), and GPT-2 perplexity (caption fluency). A detailed quantitative breakdown of these metrics is provided in \Cref{tab:timesamplingextended}. The qualitative pattern from the main-paper CIDEr curves (\Cref{fig:time_sampling}) replicates consistently across all four metrics. Pure $\pi_{\mathrm{ac}}$ (orange) converges fastest early in training because it supervises captioning directly by always conditioning on clean images; this advantage is most visible in GPT-2 perplexity, where $\pi_{\mathrm{ac}}$ achieves lower perplexity from the very first checkpoint. Pure $\pi_{\mathrm{ind}}$ (blue) learns more slowly on captioning-specific metrics but builds richer cross-modal representations by training on all $(t,\tau) \in [0,1]^2$ jointly. After switching from $\pi_{\mathrm{ind}}$ to $\pi_{\mathrm{ac}}$ (red), performance improves sharply and reaches the final pure-$\pi_{\mathrm{ac}}$ level in approximately half the remaining training budget, regardless of the exact switch point between 75k and 175k steps. These results confirm that the mixed-to-alternating schedule is the preferred default across all evaluated metrics.

\input{figures/time_sampling/app_fig}

\begin{table*}
    \centering
    \caption{\small Visual question answering benchmarks. Top group: dedicated autoregressive VLMs; middle group: unified models that handle generation and understanding jointly; bottom group: flow-matching unified models. With only $130$M trainable parameters and lightweight downstream adaptation, FullFlow remains competitive with much larger unified systems while substantially outperforming the closest flow-matching baseline (DualDiff) on VizWiz.}
    \label{tab:lm_benchmarks}

    \vspace{0.5em}
    
    \resizebox{\linewidth}{!}{
    \begin{tabular}{lccccccc}
        \toprule
        \textbf{Model} & \textbf{Params} & \textbf{Text} & \textbf{Image} & \textbf{MS-COCO} & \textbf{VQAv2} & \textbf{VizWiz} & \textbf{OKVQA} \\
        & \# trainable & Backbone & Backbone & CIDEr~$\uparrow$ & Acc.~$\uparrow$ & Acc.~$\uparrow$ & Acc.~$\uparrow$ \\
        \hline
        InternVL-2.0~\citep{chen_expanding_2025}  & 8B & AR & - & - &  - &   62.9 &  62.9 \\
        LLaVA-Next~\citep{li_llava-next-interleave_2024} & 13B & AR & - & - &  82.8 & 60.5 & - \\
        BLIP~\citep{li_blip_2022} & 13B & AR & - & - & 65.0 & 19.6 & - \\
        IDEFICS~\citep{laurencon_obelics_2023} & 9B & AR & - & - & 50.9 & - & - \\
        QWEN-VL~\citep{bai_qwen-vl_2023} & 7B & AR & - & - & 78.2 & 38.9 & - \\
        OpenFlamingo~\citep{awadalla_openflamingo_2023} & 9B & AR & - & 65.5 & 43.5 & - & - \\
        Flamingo~\citep{alayrac_flamingo} & 9B & AR & - & 79.4 & 51.8 & 28.8 & 44.7\\
        \hline
        CM3Leon~\citep{yu_scaling_2023}  & 7B & AR & AR & 61.6 & 47.6 & 37.6 & 23.8 \\
        Chameleon~\citep{team_chameleon_2025} & 7B & AR & AR & 18.0 & - & - & - \\
        LWM~\citep{liu_world_2025} & 7B & AR & AR & - & 55.8 & 11.6 & - \\
        Show-O (256$\times$256)~\citep{xie_show-o_2025} & 1.3B & AR & FM & - & 64.7 & - & - \\
        Show-O (512$\times$512)~\citep{xie_show-o_2025} & 1.3B & AR & FM & - & 69.4 & - & - \\
        Transfusion~\citep{zhou_transfusion_2024} & 7B & AR & FM & 29.0 & - & - & - \\
        \hline
        DualDiff (256$\times$256)~\citep{li_dual_2025} & 2B & FM & FM & - & 59.5 & 19.4 & 28.5 \\
        DualDiff (512$\times$512)~\citep{li_dual_2025} & 2B & FM & FM & 56.2 & 60.1 & 29.9 & 25.3 \\
        \hline
        FullFlow (Ours) & 130M & FM & FM & 54.93 & 56.5 & 50.7 & 23.3\\
        \bottomrule
    \end{tabular}}
\end{table*}

%% file: figures/cfg/fig.tex
\newcommand{\plotCFGScore}{
    \begin{tikzpicture}
    \begin{axis}[
        width=0.9\linewidth,
        height=0.7\linewidth,
        xlabel={CFG Text Scale},
        ylabel={CIDEr $\uparrow$},
        grid=major,
        grid style={draw=gray!20},
        tick label style={font=\scriptsize},
        label style={font=\small},
        scaled x ticks=false,
        xtick scale label code/.code={},
        thick,
        mark=*,
        mark options={solid, scale=0.8},
        clip=false,
    ]
        \addplot[dit_red, thick, mark=*, solid] table [
            x=cfg_txt_scale,
            y=Cider,
            col sep=comma
        ] {./figures/cfg/data.csv};
    \end{axis}
    \end{tikzpicture}
}

\newcommand{\plotCFGTokens}{
    \begin{tikzpicture}
    \begin{axis}[
        width=0.9\linewidth,
        height=0.7\linewidth,
        xlabel={CFG Text Scale},
        ylabel={\# Tokens},
        grid=major,
        grid style={draw=gray!20},
        tick label style={font=\scriptsize},
        label style={font=\small},
        legend style={
            font=\scriptsize,
            draw=none,
            fill=white,
            fill opacity=0.85,
            text opacity=1,
            at={(1.03,0.5)},
            anchor=west,
            row sep=1pt,
            inner sep=2pt
        },
        legend cell align=left,
        scaled x ticks=false,
        xtick scale label code/.code={},
        thick,
        mark options={solid, scale=0.8},
        clip=false,
    ]
        \addplot[dit_yellow, thick, mark=square*, dashed] table [
            x=cfg_txt_scale,
            y=num_true_tokens,
            col sep=comma
        ] {./figures/cfg/data.csv};
        \addlegendentry{Truth}
        
        \addplot[dit_red, thick, mark=*, solid] table [
            x=cfg_txt_scale,
            y=num_gen_tokens,
            col sep=comma
        ] {./figures/cfg/data.csv};
        \addlegendentry{Gen.}
    \end{axis}
    \end{tikzpicture}
}

\begin{figure}[ht]
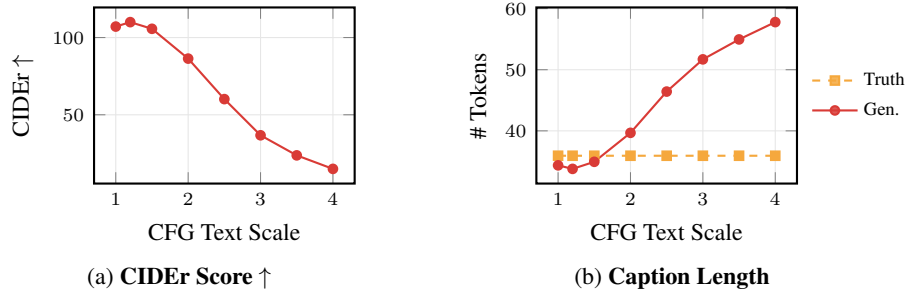

    \centering
    
    \hfill%
    \begin{subfigure}[t]{0.4\textwidth}
        \centering
        \plotCFGScore
        \caption{\textbf{CIDEr Score $\uparrow$}}
    \end{subfigure}%
    \hfill
    \begin{subfigure}[t]{0.4\textwidth}
        \centering
        \plotCFGTokens
        \caption{\textbf{Caption Length}}
    \end{subfigure}%
    \hfill\mbox{}%

    \caption{\small Effect of image-to-text classifier-free guidance scale $\gamma$ on caption quality and length. Mild guidance ($\gamma\approx1.2$) improves CIDEr and brings caption length closer to the reference; stronger guidance over-amplifies high-probability insertions in the reverse CTMC, producing overly verbose captions and degrading quality.}
    \label{fig:cfg_metrics}
\end{figure}

%% file: figures/sd3_t2i/fig.tex
\newcommand{\methodlabel}[1]{%
  \adjustbox{valign=c}{\rotatebox[origin=c]{90}{#1}}%
}

\newcommand{\cellimg}[1]{%
  \adjustbox{valign=c}{\includegraphics[width=0.15\textwidth]{#1}}%
}

\begin{figure}[ht]
  \centering
  \setlength{\tabcolsep}{1pt}
  \renewcommand{\arraystretch}{0.9}
  \begin{tabular}{@{}c@{\hspace{2pt}}ccccc@{}}
    \toprule
    & \textbf{Base} & \textbf{50k} & \textbf{100k} & \textbf{150k} & \textbf{200k} \\
    \midrule
    \methodlabel{RF} &
    \cellimg{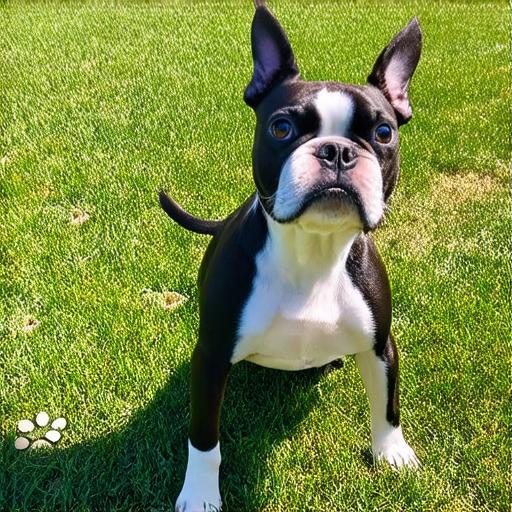} &
    \cellimg{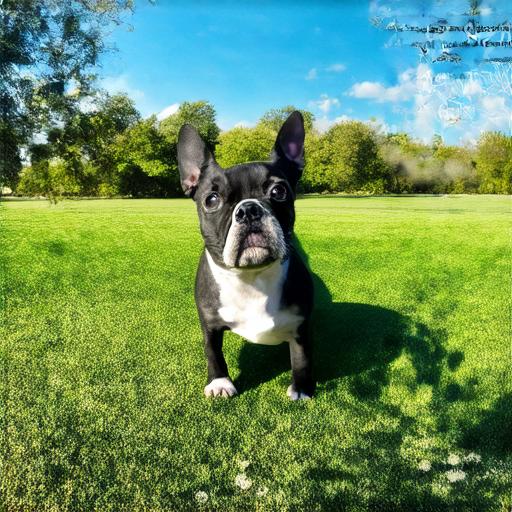} &
    \cellimg{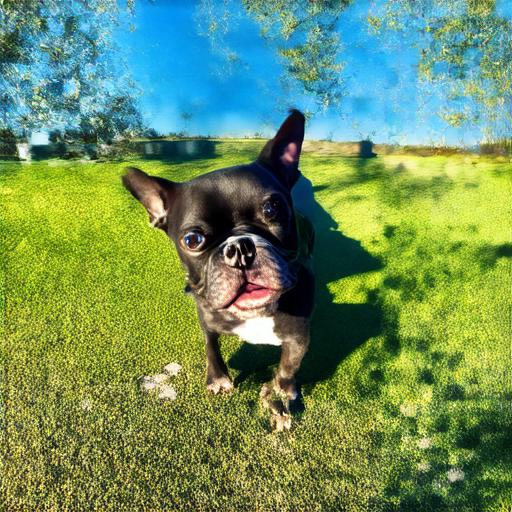} &
    \cellimg{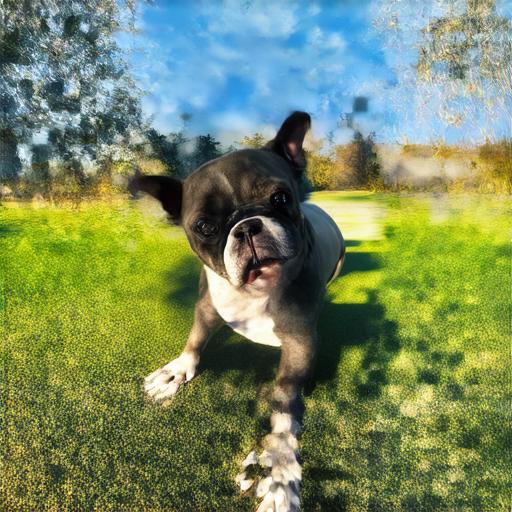} &
    \cellimg{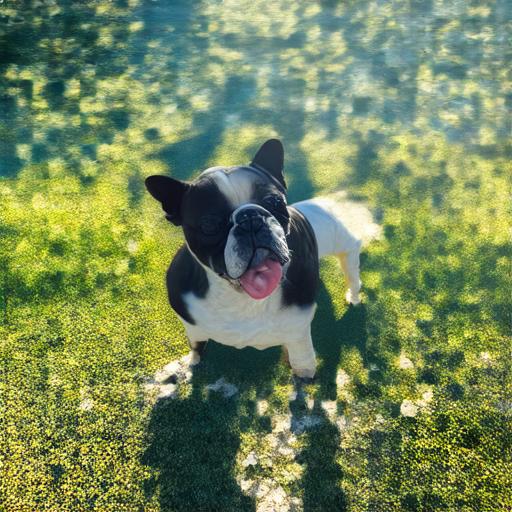} \\

    \methodlabel{Teacher-SN} &
    \cellimg{figures/sd3_t2i/images/base.jpg} &
    \cellimg{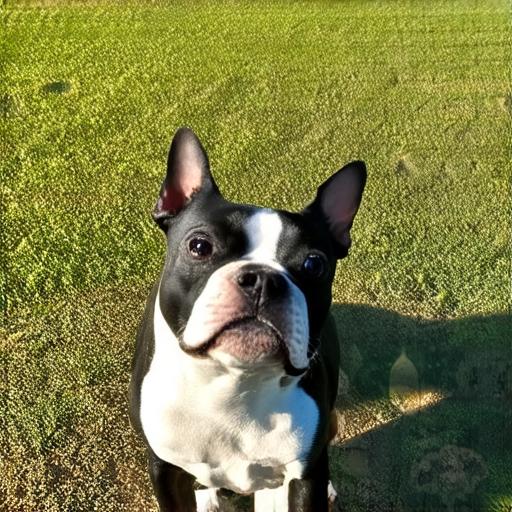} &
    \cellimg{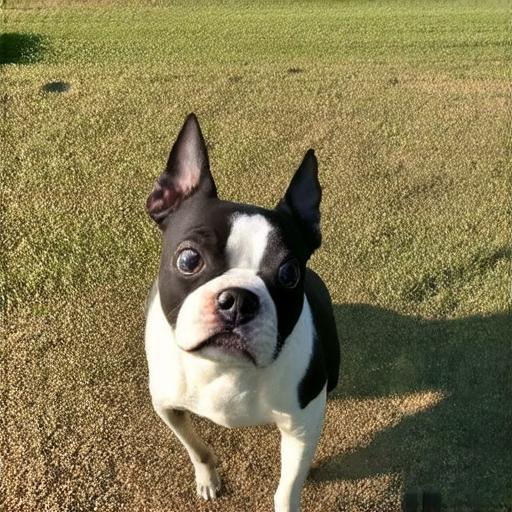} &
    \cellimg{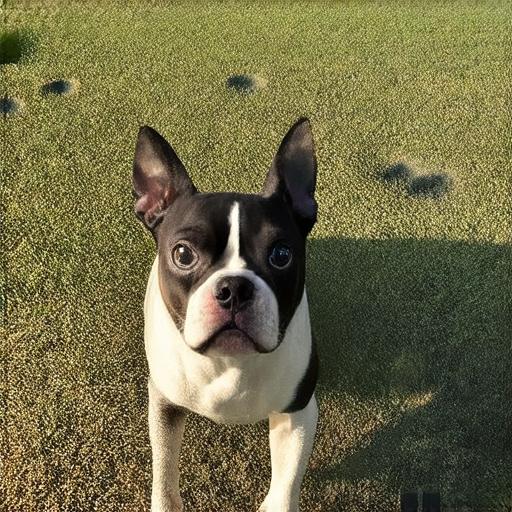} &
    \cellimg{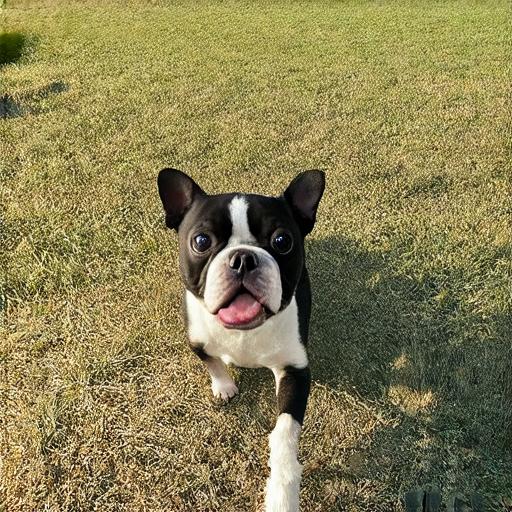} \\

    \methodlabel{Teacher-CT} &
    \cellimg{figures/sd3_t2i/images/base.jpg} &
    \cellimg{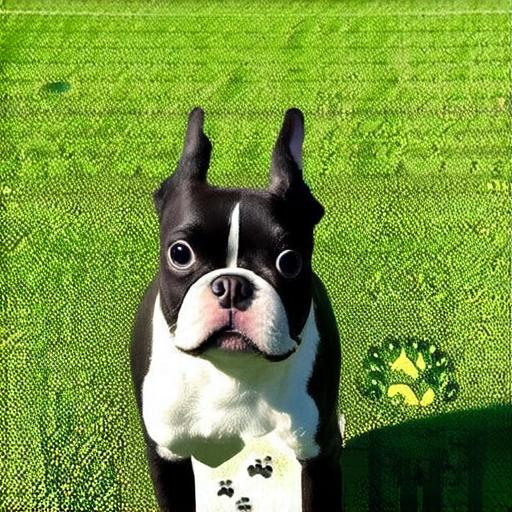} &
    \cellimg{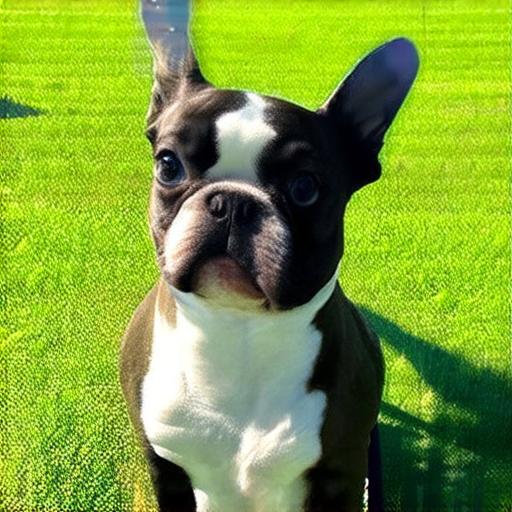} &
    \cellimg{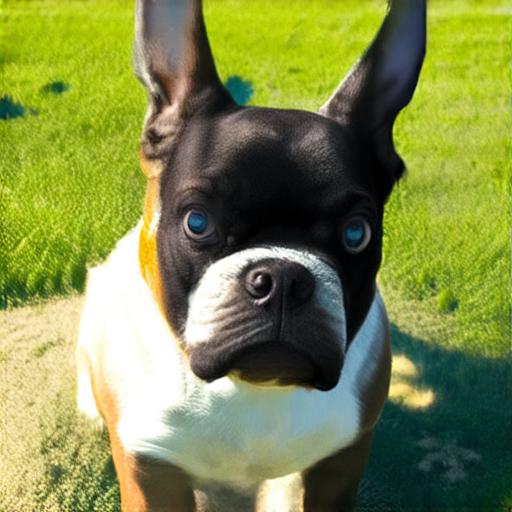} &
    \cellimg{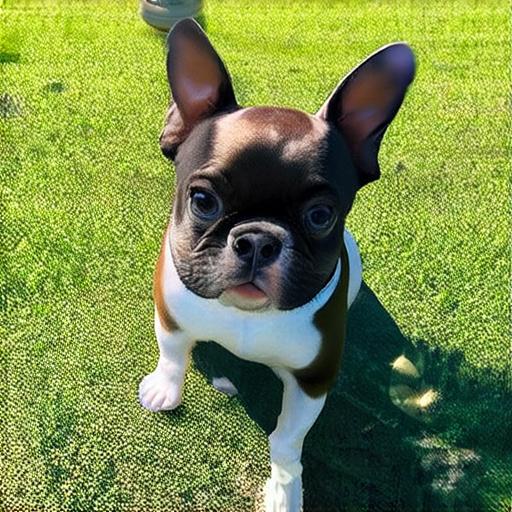} \\
    \bottomrule
  \end{tabular}
  \caption{\small Qualitative effect of the image-supervision target on text-to-image generation across training checkpoints (50k--200k steps). RF: standard rectified flow; Teacher-CT: clean-text teacher matching; Teacher-SN: same-noise teacher matching. RF accumulates distortion, Teacher-CT introduces transient artifacts and is unstable, while Teacher-SN preserves the visual prior throughout training and is therefore our default.}
  \label{fig:t2i_teachers}
\end{figure}

%% file: figures/time_sampling/app_fig.tex
\newcommand{\addMetricPlot}[4]{

    \addplot[draw=none, forget plot, name path=upper#1] table [
        x=ckpt_num, 
        y expr=\thisrow{#1_mean} + \thisrow{#1_ci_95}, 
        col sep=comma
    ] {./figures/time_sampling/#4};

    \addplot[draw=none, forget plot, name path=lower#1] table [
        x=ckpt_num, 
        y expr=\thisrow{#1_mean} - \thisrow{#1_ci_95}, 
        col sep=comma
    ] {./figures/time_sampling/#4};

    \addplot[fill=#2, opacity=0.3, forget plot] fill between[of=upper#1 and lower#1];

    \addplot[#2, #3] table [
        x=ckpt_num, 
        y=#1_mean, 
        col sep=comma
    ] {./figures/time_sampling/#4};
}

\newcommand{\plotMetric}[2]{
    \begin{tikzpicture}
    \begin{axis}[
        width=\linewidth,
        height=0.8\linewidth,
        xlabel={Checkpoints},
        ylabel={#1},
        xmin=25000, xmax=200000,
        xtick={50000,100000,150000,200000},
        xticklabels={50k,100k,150k,200k},
        scaled x ticks=false,
        xtick scale label code/.code={},
        grid=major,
        grid style={draw=gray!20},
        tick label style={font=\scriptsize},
        label style={font=\small},
        legend style={
            font=\scriptsize,
            draw=none,
            fill=white,
            fill opacity=0.85,
            text opacity=1,
            at={(0.03,0.97)},
            anchor=north west,
            row sep=1pt,
            inner sep=2pt
        },
        clip=false,
    ]
    
    \addMetricPlot{res_001_txt_train}{dit_blue}{thick, mark=*, solid}{#2}
    \addMetricPlot{res_001_txt_val}{dit_blue}{thick, mark=*, dashed}{#2}
    
    
    \addMetricPlot{res_001_75k_txt_train}{dit_red}{thick, mark=*, solid}{#2}
    \addMetricPlot{res_001_75k_txt_val}{dit_red}{thick, mark=*, dashed}{#2}
    
    \addMetricPlot{res_001_100k_txt_train}{dit_red}{thick, mark=*, solid}{#2}
    \addMetricPlot{res_001_100k_txt_val}{dit_red}{thick, mark=*, dashed}{#2}
    
    \addMetricPlot{res_001_125k_txt_train}{dit_red}{thick, mark=*, solid}{#2}
    \addMetricPlot{res_001_125k_txt_val}{dit_red}{thick, mark=*, dashed}{#2}
    
    \addMetricPlot{res_001_150k_txt_train}{dit_red}{thick, mark=*, solid}{#2}
    \addMetricPlot{res_001_150k_txt_val}{dit_red}{thick, mark=*, dashed}{#2}
    
    \addMetricPlot{res_001_175k_txt_train}{dit_red}{thick, mark=*, solid}{#2}
    \addMetricPlot{res_001_175k_txt_val}{dit_red}{thick, mark=*, dashed}{#2}
    
    \addMetricPlot{res_004_txt_train}{dit_yellow}{thick, mark=*, solid}{#2}
    \addMetricPlot{res_004_txt_val}{dit_yellow}{thick, mark=*, dashed}{#2}
    
    \end{axis}
    \end{tikzpicture}
}

\begin{figure}
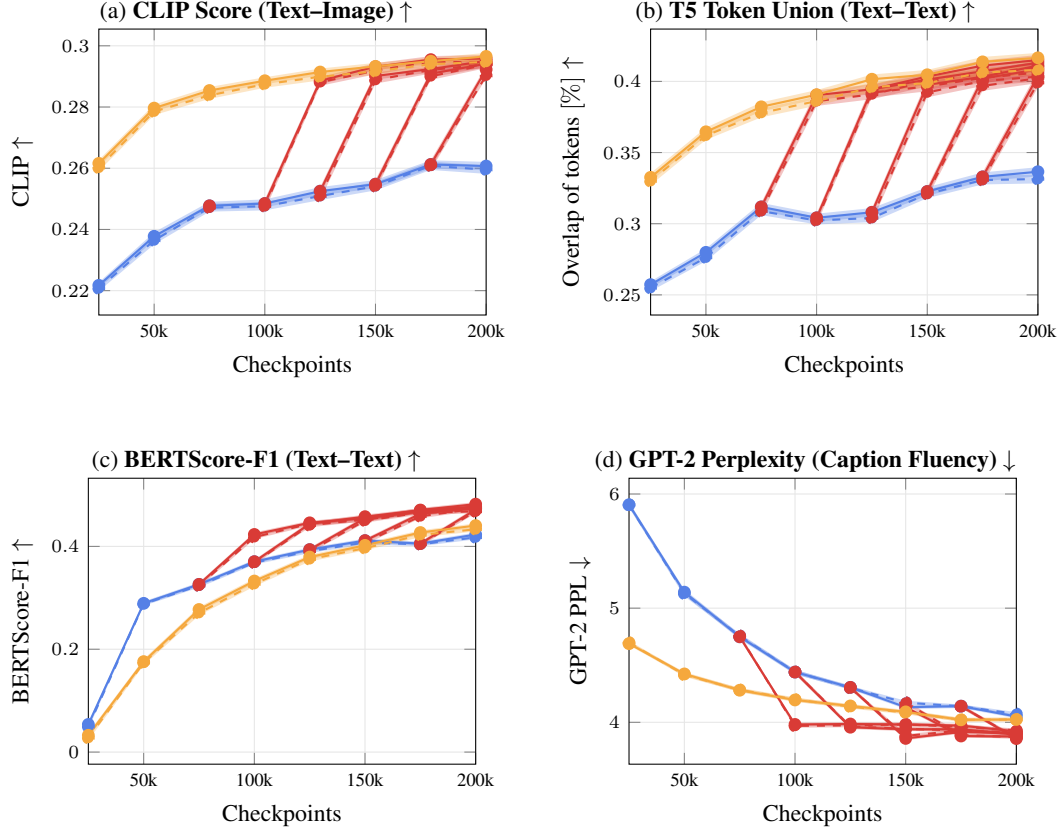

    \centering
    \begin{subfigure}[t]{0.48\textwidth}
        \centering
        \caption{\textbf{CLIP Score (Text--Image) $\uparrow$}}
        \plotMetric{CLIP $\uparrow$}{clip.csv}
    \end{subfigure}%
    \hfill
    \begin{subfigure}[t]{0.48\textwidth}
        \centering
        \caption{\textbf{T5 Token Union (Text--Text) $\uparrow$}}
        \plotMetric{Overlap of tokens $[\%] \uparrow$}{t5.csv}
    \end{subfigure}\\\vspace{2mm}
    \begin{subfigure}[t]{0.48\textwidth}
        \centering
        \caption{\textbf{BERTScore-F1 (Text--Text) $\uparrow$}}
        \plotMetric{BERTScore-F1 $\uparrow$}{bert.csv}
    \end{subfigure}%
    \hfill
    \begin{subfigure}[t]{0.48\textwidth}
        \centering
        \caption{\textbf{GPT-2 Perplexity (Caption Fluency) $\downarrow$}}
        \plotMetric{GPT-2 PPL $\downarrow$}{gpt.csv}
    \end{subfigure}%
    \caption{\small Extended corruption-schedule ablation from \Cref{sec:time_coupling}, showing four additional captioning metrics beyond CIDEr. \textcolor{dit_yellow}{\textbf{yellow}}: pure alternating-clean ($\pi_{\mathrm{ac}}$); \textcolor{dit_blue}{\textbf{blue}}: pure mixed-corruption ($\pi_{\mathrm{ind}}$); \textcolor{dit_red}{\textbf{red}}: five independent runs that switch from $\pi_{\mathrm{ind}}$ to $\pi_{\mathrm{ac}}$ at 75k, 100k, 125k, 150k, and 175k steps. Solid lines: train; dashed lines: held-out validation. Shaded bands: 95\% CI over 5k samples. Across all four metrics, the mixed-to-alternating schedule matches or exceeds pure $\pi_{\mathrm{ac}}$ at convergence while requiring only half the training budget after the switch.}
    \label{fig:timesamplingextended}
\end{figure}

\begin{table}[tbh]
  \centering
  \setlength{\tabcolsep}{3.2pt}
  \renewcommand{\arraystretch}{1.05}
  
\caption{\small Final-checkpoint (200k steps) statistics for the corruption-schedule ablation of \Cref{fig:timesamplingextended}, complementing the curves with exact numerical values. For each schedule, split, and captioning metric we report the mean over $n{=}5000$ validation samples, the sample standard deviation, and the half-width of the 95\% confidence interval used as the shaded band in \Cref{fig:timesamplingextended}.}
    \vspace{0.5em}
  
  \footnotesize
  \resizebox{\linewidth}{!}{
  \begin{tabular}{ll 
    S[table-format=3.3] S[table-format=3.3] S[table-format=1.4]
    S[table-format=1.3] S[table-format=1.3] S[table-format=1.4]
    S[table-format=1.3] S[table-format=1.3] S[table-format=1.4]
    S[table-format=1.3] S[table-format=1.3] S[table-format=1.4]
    S[table-format=1.3] S[table-format=1.3] S[table-format=1.4]}
    \toprule
     &  
     & \multicolumn{3}{c}{\textbf{CIDEr $\uparrow$}}
     & \multicolumn{3}{c}{\textbf{CLIP $\uparrow$}}
     & \multicolumn{3}{c}{\textbf{T5 Tok.\ Union $\uparrow$}}
     & \multicolumn{3}{c}{\textbf{BERTScore-F1 $\uparrow$}}
     & \multicolumn{3}{c}{\textbf{GPT-PPL $\downarrow$}} \\
    \cmidrule(lr){3-5}
    \cmidrule(lr){6-8}
    \cmidrule(lr){9-11}
    \cmidrule(lr){12-14}
    \cmidrule(lr){15-17}
    {Schedule} & {Split} 
    & {mean} & {std} & {95\% CI}
    & {mean} & {std} & {95\% CI}
    & {mean} & {std} & {95\% CI}
    & {mean} & {std} & {95\% CI}
    & {mean} & {std} & {95\% CI} \\
    \midrule

    $\pi_{\mathrm{ind}}$ & train &
    66.78 & 77.74 & 2.25 &
    0.26 & 0.05 & 0.00 &
    0.34 & 0.11 & 0.00 &
    0.42 & 0.14 & 0.00 &
    4.05 & 0.55 & 0.02 \\

     & val &
    65.43 & 76.87 & 2.21 &
    0.26 & 0.05 & 0.00 &
    0.33 & 0.11 & 0.00 &
    0.42 & 0.14 & 0.00 &
    4.07 & 0.55 & 0.02 \\

    \addlinespace[1pt]

    switch @ 75k & train 
    & 114.70 & 104.30 & 3.02 
    & {--} & {--} & {--}
    & 0.41 & 0.12 & 0.00 
    & 0.48 & 0.14 & 0.00 
    & 3.92 & 0.56 & 0.02 \\
    
    & val 
    & 111.80 & 103.19 & 2.97 
    & {--} & {--} & {--}
    & 0.41 & 0.12 & 0.00 
    & 0.48 & 0.14 & 0.00 
    & 3.93 & 0.56 & 0.02 \\
    
    \addlinespace[1pt]

    switch @ 100k & train 
    & 111.75 & 101.57 & 2.94 
    & 0.30 & 0.04 & 0.00 
    & 0.41 & 0.12 & 0.00 
    & 0.48 & 0.14 & 0.00 
    & 3.89 & 0.55 & 0.02 \\

     & val 
    & 107.63 & 100.21 & 2.88 
    & 0.29 & 0.04 & 0.00 
    & 0.40 & 0.12 & 0.00 
    & 0.47 & 0.14 & 0.00 
    & 3.93 & 0.54 & 0.02 \\

    \addlinespace[1pt]

    switch @ 125k & train 
    & 110.71 & 101.19 & 2.93 
    & 0.30 & 0.04 & 0.00 
    & 0.41 & 0.12 & 0.00 
    & 0.47 & 0.14 & 0.00 
    & 3.90 & 0.54 & 0.02 \\

     & val 
    & 108.76 & 100.40 & 2.89 
    & 0.29 & 0.04 & 0.00 
    & 0.40 & 0.12 & 0.00 
    & 0.47 & 0.14 & 0.00 
    & 3.91 & 0.55 & 0.02 \\

    \addlinespace[1pt]

    switch @ 150k & train 
    & 84.94 & 89.82 & 2.60 
    & 0.29 & 0.04 & 0.00 
    & 0.40 & 0.12 & 0.00 
    & 0.43 & 0.14 & 0.00 
    & 3.83 & 0.53 & 0.02 \\

     & val 
    & 84.68 & 91.90 & 2.64 
    & 0.29 & 0.04 & 0.00 
    & 0.40 & 0.12 & 0.00 
    & 0.43 & 0.14 & 0.00 
    & 3.83 & 0.54 & 0.02 \\

    \addlinespace[1pt]

    switch @ 175k & train 
    & 108.61 & 99.52 & 2.88 
    & 0.29 & 0.04 & 0.00 
    & 0.40 & 0.12 & 0.00 
    & 0.47 & 0.14 & 0.00 
    & 3.86 & 0.54 & 0.02 \\

     & val 
    & 107.11 & 100.64 & 2.89 
    & 0.29 & 0.04 & 0.00 
    & 0.40 & 0.12 & 0.00 
    & 0.47 & 0.14 & 0.00 
    & 3.86 & 0.54 & 0.02 \\

    \addlinespace[1pt]

    $\pi_{\mathrm{ac}}$ & train 
    99.38 & 94.95 & 2.75 &
    0.30 & 0.04 & 0.00 &
    0.42 & 0.12 & 0.00 &
    0.44 & 0.15 & 0.00 &
    4.02 & 0.59 & 0.02 \\

    & val 
    96.97 & 95.88 & 2.76 &
    0.29 & 0.04 & 0.00 &
    0.41 & 0.12 & 0.00 &
    0.43 & 0.15 & 0.00 &
    4.02 & 0.58 & 0.02 \\

    \bottomrule
  \end{tabular}}
  \label{tab:timesamplingextended}
\end{table}

\begin{table}[tbh]
  \centering
  \caption{\small Image-to-text generation statistics on SD3 with matched data, LoRA rank, and training setup, comparing the DualDiff-LoRA baseline against FullFlow at matched training steps and matched wall-clock time. All metrics are computed against held-out Moondream captions on the validation split.}
    
    \vspace{0.5em}
  
  \setlength{\tabcolsep}{3.2pt}
  \renewcommand{\arraystretch}{1.05}
  \footnotesize
  \resizebox{\linewidth}{!}{
  \begin{tabular}{l
    S[table-format=3.3] S[table-format=3.3] S[table-format=1.4]
    S[table-format=1.3] S[table-format=1.3] S[table-format=1.4]
    S[table-format=1.3] S[table-format=1.3] S[table-format=1.4]
    S[table-format=1.3] S[table-format=1.3] S[table-format=1.4]
    S[table-format=1.3] S[table-format=1.3] S[table-format=1.4]}
    \toprule
     & \multicolumn{3}{c}{\textbf{CIDEr $\uparrow$}}
     & \multicolumn{3}{c}{\textbf{CLIP $\uparrow$}}
     & \multicolumn{3}{c}{\textbf{T5 Tok.\ Union $\uparrow$}}
     & \multicolumn{3}{c}{\textbf{BERTScore-F1 $\uparrow$}}
     & \multicolumn{3}{c}{\textbf{GPT-PPL $\downarrow$}} \\
    \cmidrule(lr){2-4}
    \cmidrule(lr){5-7}
    \cmidrule(lr){8-10}
    \cmidrule(lr){11-13}
    \cmidrule(lr){14-16}
    {Model}
    & {mean} & {std} & {95\% CI}
    & {mean} & {std} & {95\% CI}
    & {mean} & {std} & {95\% CI}
    & {mean} & {std} & {95\% CI}
    & {mean} & {std} & {95\% CI} \\
    \midrule

    DualDiff.-LoRA
    & 2.06 & 5.36 & 0.15 
    & {--} & {--} & {--}
    & 0.11 & 0.06 & 0.00 
    & -0.28 & 0.78 & 0.02 
    & 7.66 & 2.30 & 0.06 \\
    
    \addlinespace[1pt]

    FullFlow, step-matched
    & 13.16 & 25.73 & 0.74 
    & 0.26 & 0.04 & 0.00 
    & 0.33 & 0.10 & 0.00 
    & 0.03 & 0.17 & 0.00 
    & 4.70 & 0.56 & 0.02 \\

    \addlinespace[1pt]

    FullFlow, time-matched
    & 107.11 & 100.64 & 2.89 
    & 0.29 & 0.04 & 0.00 
    & 0.40 & 0.12 & 0.00 
    & 0.47 & 0.14 & 0.00 
    & 3.86 & 0.54 & 0.02 \\

    \bottomrule
  \end{tabular}}

  \label{tab:main_comparison_extend}
\end{table}

%% file: appendix/examples.tex
\section{Qualitative Examples}\label{app:cross}
This appendix collects qualitative samples from the SD3 and FLUX.1-dev FullFlow models on unseen evaluation inputs, complementing the quantitative results in \Cref{sec:results}. \Cref{fig:image2text1,fig:image2text2,fig:image2text_flux} show image-to-text captions, \Cref{fig:t2i_sd3_app,fig:t2i_flux_app} compare text-to-image samples between each base model and its FullFlow finetune, \Cref{fig:vqav2} shows VQAv2 predictions, and \Cref{fig:joint_tok,fig:joint1,fig:joint2,fig:joint3,fig:joint4,fig:joint5} illustrate joint image--text generation along the $\tau=t^{2^p}$ trajectories from \Cref{sec:joint}. All samples are generated with the inference settings of \Cref{app:vqa_inference} on inputs that were not seen during training.

\begin{figure*}[ht]
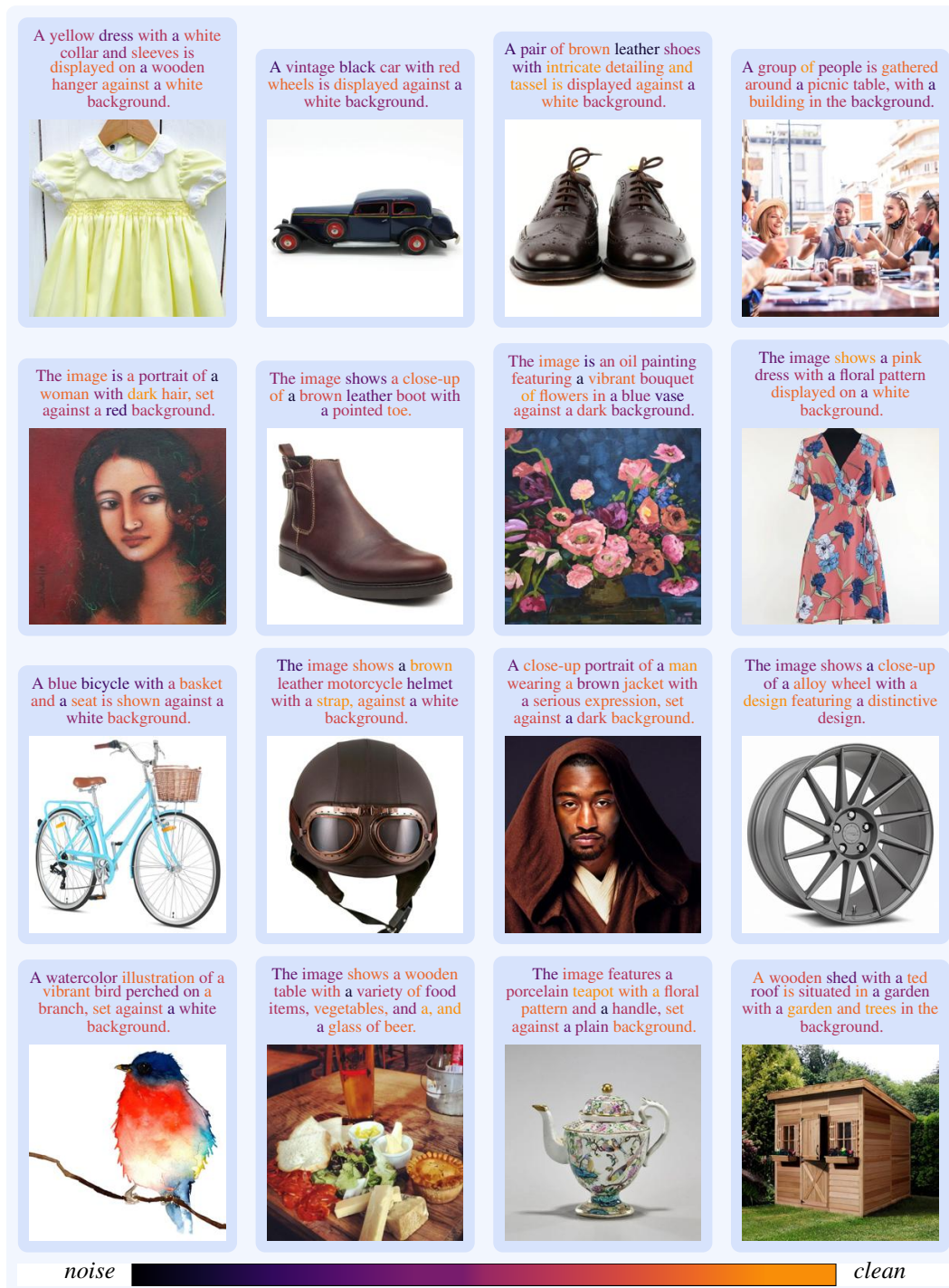

  \centering
  \qualexgrid{./figures/sd3_i2t/captions.csv}{./figures/sd3_i2t/}{0,1,2,32,4,5,6,7,8,9,10,11,12,13,14,15}
  \caption{\small Image-to-text captions from the SD3 FullFlow model on unseen LAION-Aesthetic images (part 1/2). Each tile shows the input image and the greedy caption sampled by the model.}
  \label{fig:image2text1}
\end{figure*}

\begin{figure*}[ht]
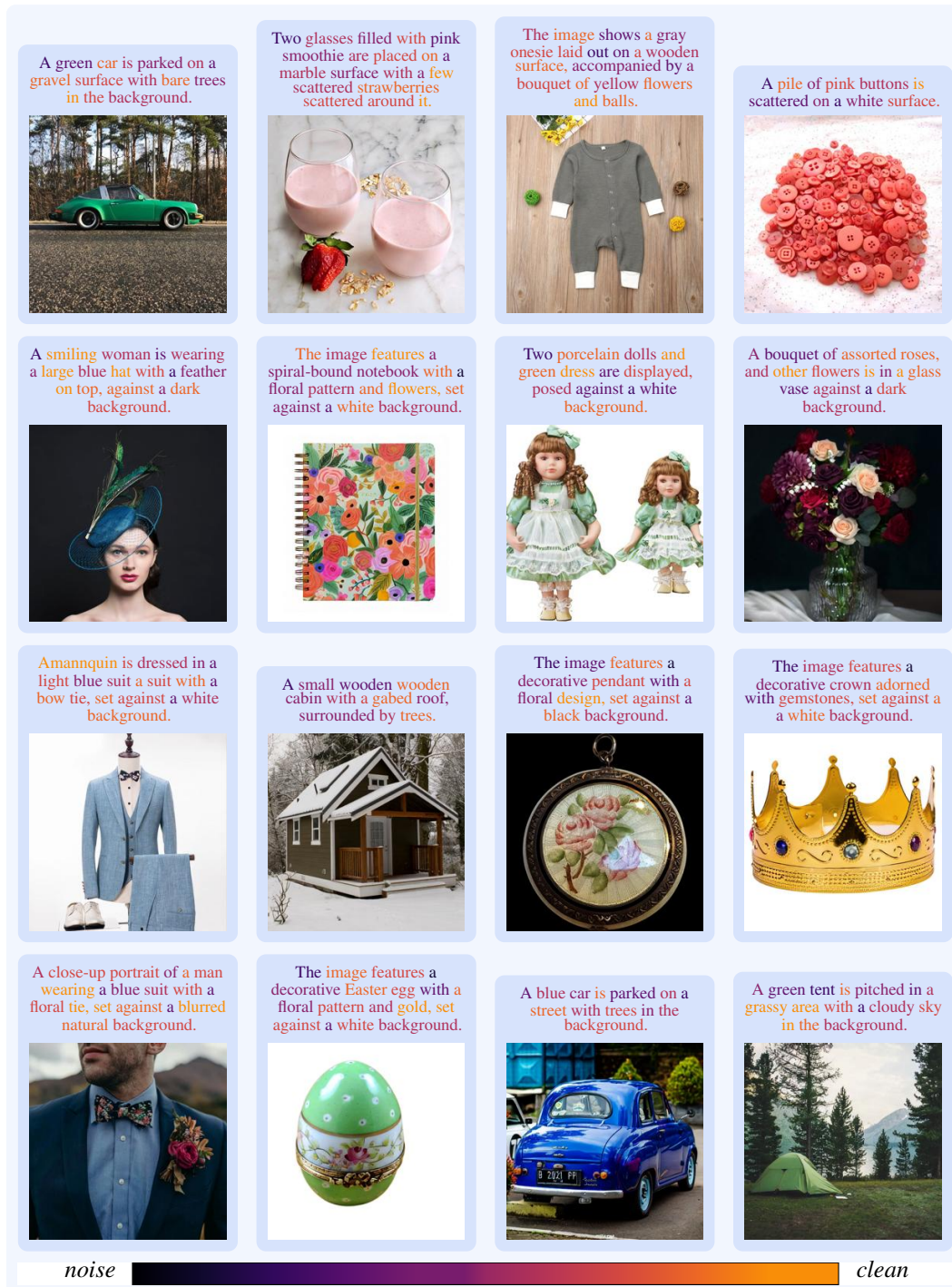

  \centering
  \qualexgrid{./figures/sd3_i2t/captions.csv}{./figures/sd3_i2t/}{16,17,18,19,20,21,22,23,24,25,26,27,28,29,30,31}
  \caption{\small Image-to-text captions from the SD3 FullFlow model on unseen LAION-Aesthetic images (part 2/2). Continuation of \Cref{fig:image2text1}.}
  \label{fig:image2text2}
\end{figure*}

\begin{figure*}[ht]
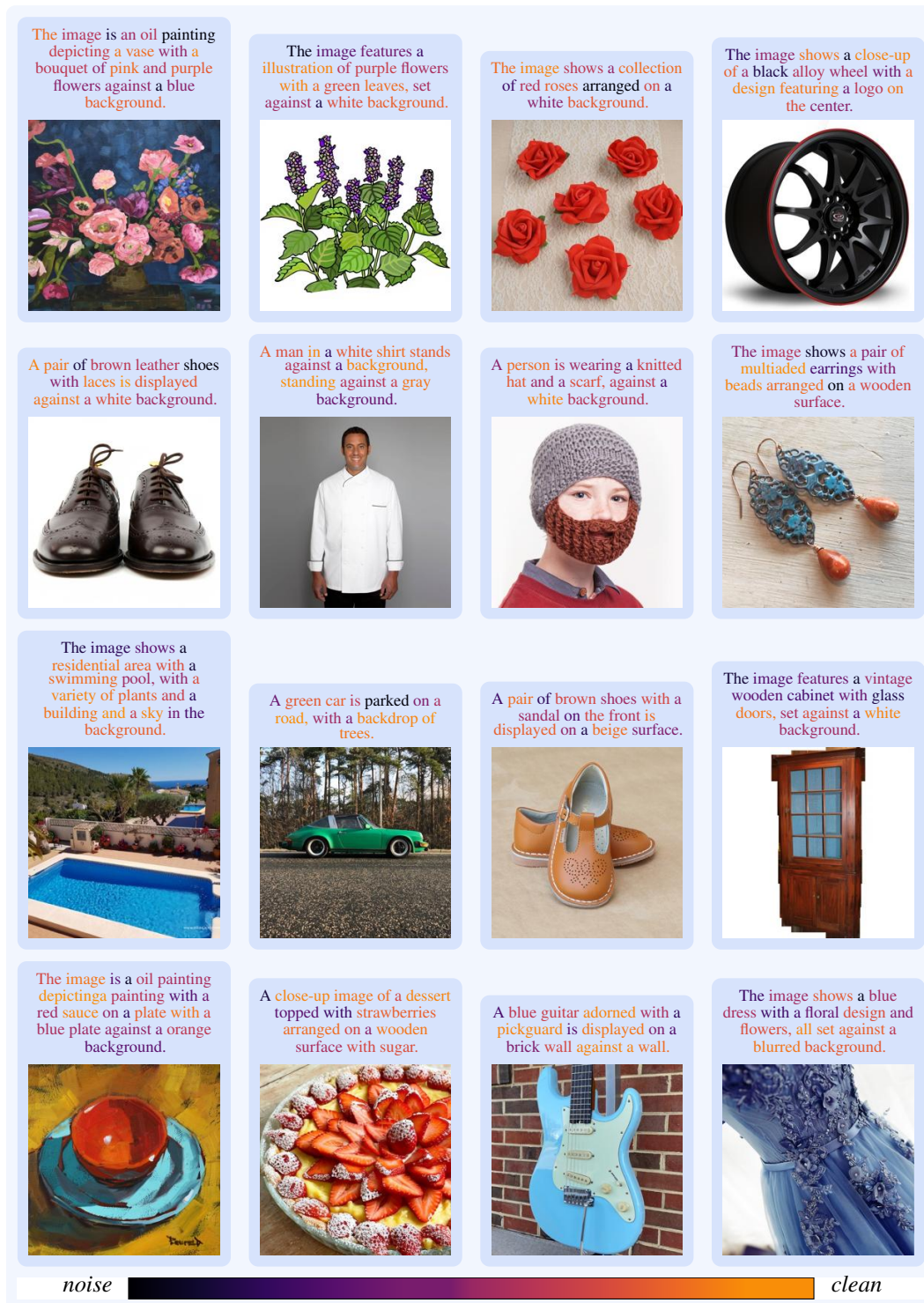

  \centering
  \qualexgrid{./figures/flux_i2t/captions.csv}{./figures/flux_i2t/}{0,1,2,3,4,5,6,7,8,9,10,11,12,13,14,15}
  \caption{\small Image-to-text captions from the FLUX.1 [dev] FullFlow model on unseen LAION-Aesthetic images, demonstrating that the same recipe transfers to a different rectified-flow backbone.}
  \label{fig:image2text_flux}
\end{figure*}

\begin{figure*}
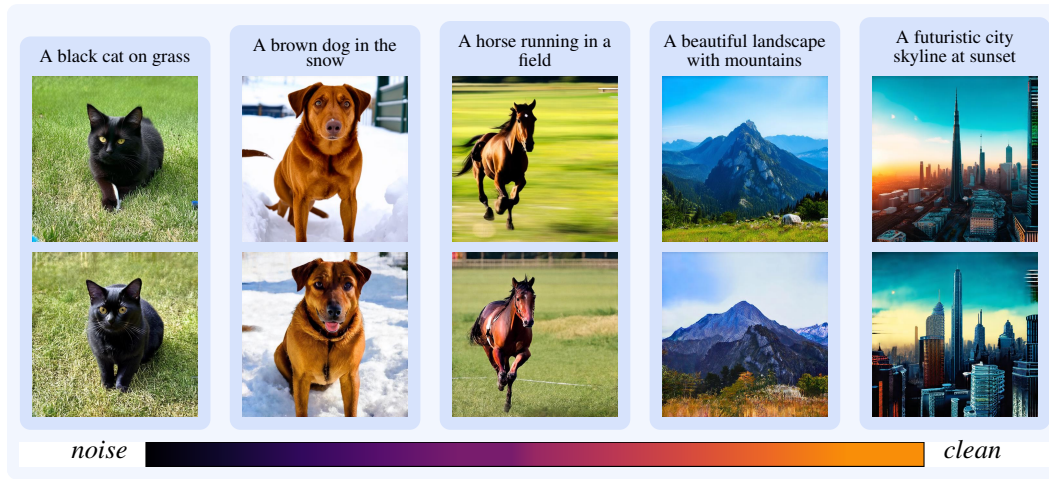

    \centering
    \qualexrowtwo{./figures/flux_t2i/captions.csv}{./figures/sd3_t2i/base/}{./figures/sd3_t2i/fine/}{0,1,2,3,4}
    \caption{\small Text-to-image samples from the SD3 base model (top row) and our FullFlow SD3 finetune (bottom row) under identical prompts, seeds, and schedulers. Differences between rows isolate the effect of the multimodal uplift on the visual prior.}
    \label{fig:t2i_sd3_app}
\end{figure*}

\begin{figure*}
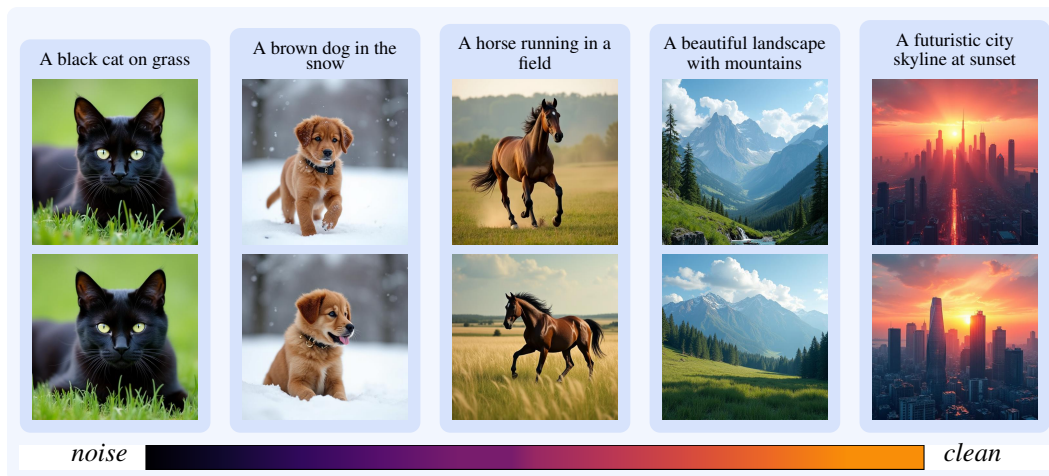

    \centering
    \qualexrowtwo{./figures/flux_t2i/captions.csv}{./figures/flux_t2i/base/}{./figures/flux_t2i/fine/}{0,1,2,3,4}
    \caption{\small Text-to-image samples from the FLUX.1 [dev] base model (top row) and our FullFlow FLUX.1 finetune (bottom row) under identical prompts, seeds, and schedulers, mirroring the SD3 comparison in \Cref{fig:t2i_sd3_app}.}
    \label{fig:t2i_flux_app}
\end{figure*}

\begin{figure*}[ht]
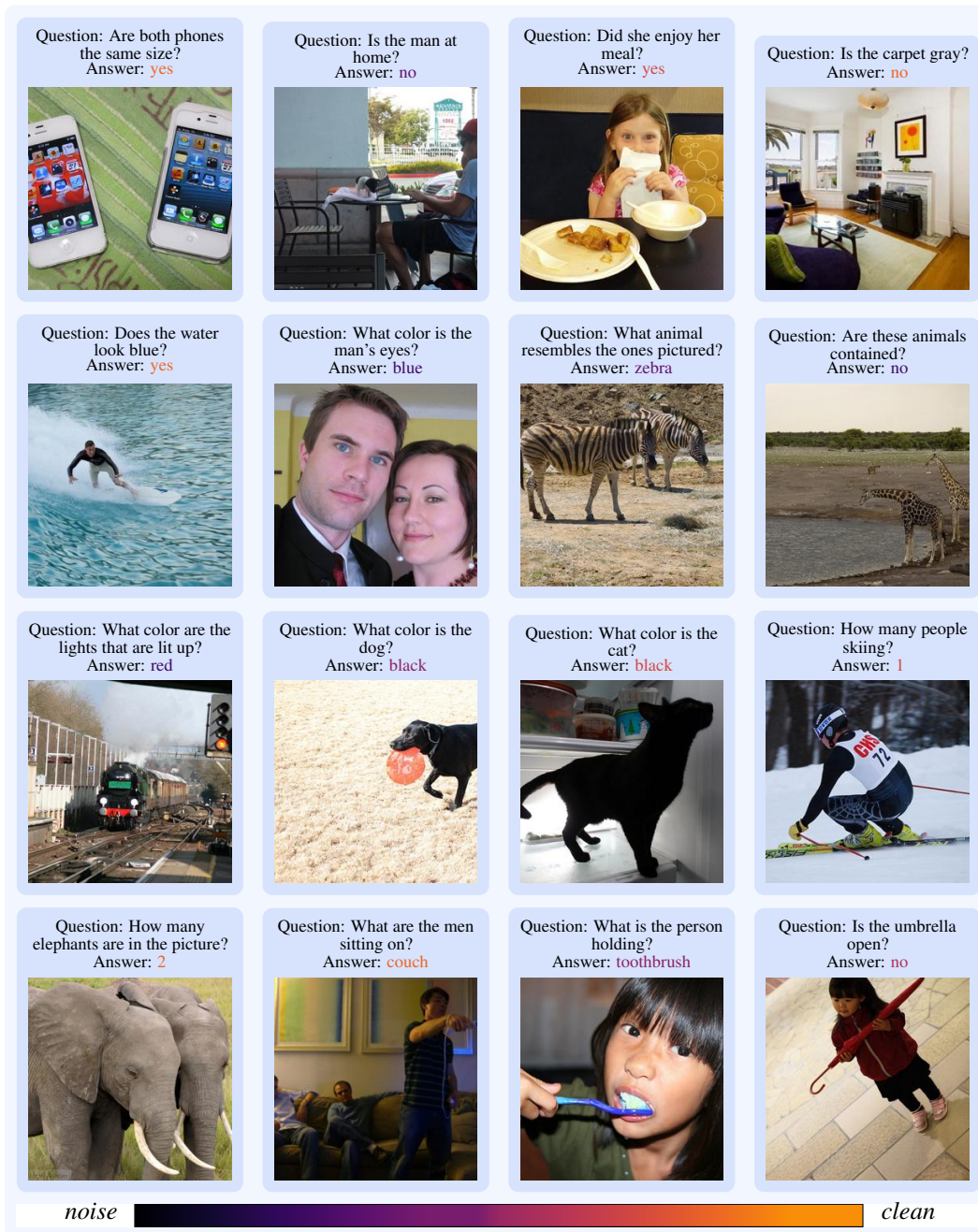

  \centering
  \qualexgrid{./figures/sd3_vqav2/captions.csv}{./figures/sd3_vqav2/images}{0,1,2,3,16,5,6,7,8,9,10,11,12,13,14,15}
  \caption{\small VQAv2 predictions from the SD3 FullFlow VQA finetune on unseen validation images. Each tile shows the input image, the question, and the model's predicted answer.}
  \label{fig:vqav2}
\end{figure*}

\subsection{Joint Generation: Qualitative Examples}
\label{app:joint_gen_examples}

\Cref{fig:joint_tok,fig:joint1,fig:joint2,fig:joint3,fig:joint4,fig:joint5} show additional examples of FullFlow's joint generation mode, where image and text are sampled simultaneously without conditioning on either modality. We sweep the trajectory parameter $p \in \{-5,-2.5,0,2.5,5\}$ using $\tau=t^{2^p}$ as detailed in \Cref{sec:joint}. Negative values of $p$ denoise text earlier than the image, positive values denoise the image earlier than text, and $p=0$ follows the diagonal trajectory where both noise levels advance together.

\begin{figure}[ht]
    \centering
    \input{figures/sd3_joint/app}
    \caption{\small Caption-length distribution under joint image--text sampling as a function of the trajectory parameter $p$ in $\tau=t^{2^p}$. Boxes show the interquartile range (Q1--Q3) with the median; whiskers extend to the min/max sample. Larger $p$ (image-first) yields longer captions, consistent with the qualitative examples below.}
    \label{fig:joint_tok}
\end{figure}
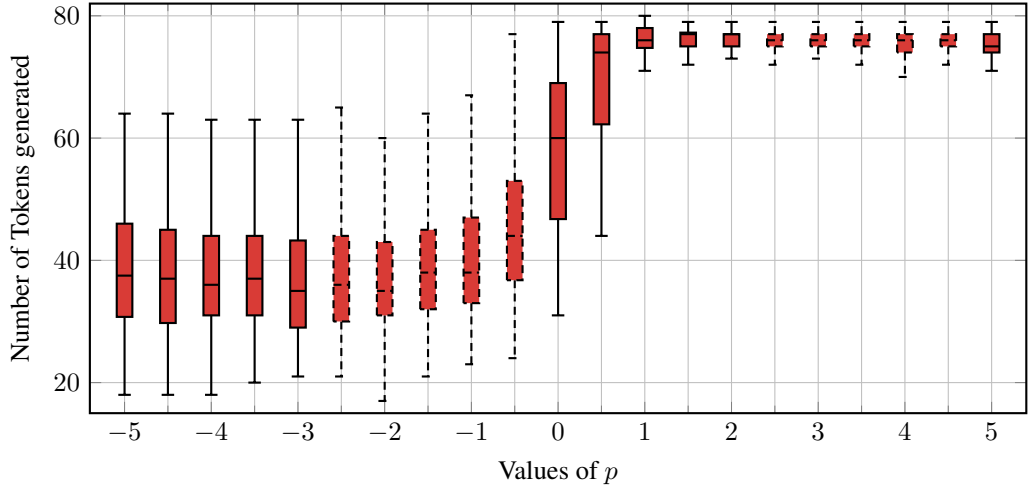

The qualitative examples are consistent with the CLIP-score trend in \Cref{fig:sd3_joint}: the balanced trajectory ($p=0$) yields the most coherent image--text pairs, with captions describing the main objects, colors, and spatial relations in the image. Text-first trajectories ($p<0$) remain competitive but can commit early to generic semantic content that the image then follows. Image-first trajectories ($p>0$) often defer language generation until the visual scene is mostly formed; at large positive values, captions tend to emphasize local visual details rather than the overall scene. \Cref{fig:joint_tok} shows that generated captions also become longer for larger $p$, suggesting that delayed text denoising encourages more verbose descriptions.

\begin{figure*}[ht]
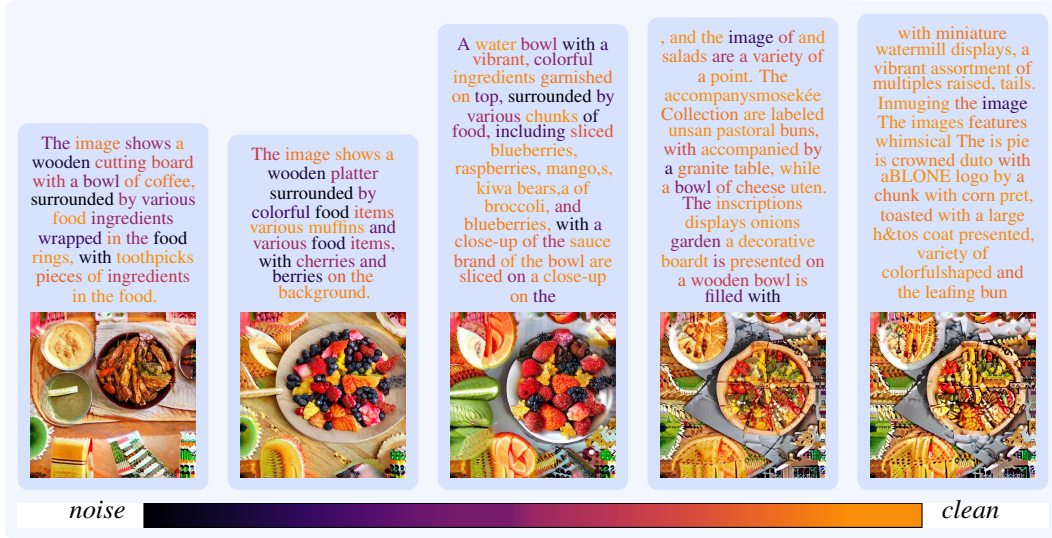

    \centering
    \qualexrow{./figures/sd3_joint/images_3/captions.csv}{./figures/sd3_joint/images_3}{0,1,2,3,4}
    \caption{\small Joint image--text samples from FullFlow (seed 1/5) for $p \in \{-5,-2.5,0,2.5,5\}$, shown left to right. Negative $p$ denoises text first, $p=0$ denoises both modalities together, and positive $p$ denoises the image first.}
    \label{fig:joint1}
\end{figure*}

\begin{figure*}[ht]
    \centering
    \qualexrow{./figures/sd3_joint/images_13/captions.csv}{./figures/sd3_joint/images_13}{0,1,2,3,4}
    \caption{\small Joint image--text samples from FullFlow (seed 2/5) for the same trajectory sweep as \Cref{fig:joint1}.}
    \label{fig:joint2}
\end{figure*}

\begin{figure*}[ht]
    \centering
    \qualexrow{./figures/sd3_joint/images_45/captions.csv}{./figures/sd3_joint/images_45}{0,1,2,3,4}
    \caption{\small Joint image--text samples from FullFlow (seed 3/5) for the same trajectory sweep as \Cref{fig:joint1}.}
    \label{fig:joint3}
\end{figure*}

\begin{figure*}[ht]
    \centering
    \qualexrow{./figures/sd3_joint/images_48/captions.csv}{./figures/sd3_joint/images_48}{0,1,2,3,4}
    \caption{\small Joint image--text samples from FullFlow (seed 4/5) for the same trajectory sweep as \Cref{fig:joint1}.}
    \label{fig:joint4}
\end{figure*}

\begin{figure*}[ht]
    \centering
    \qualexrow{./figures/sd3_joint/images_83/captions.csv}{./figures/sd3_joint/images_83}{0,1,2,3,4}
    \caption{\small Joint image--text samples from FullFlow (seed 5/5) for the same trajectory sweep as \Cref{fig:joint1}.}
    \label{fig:joint5}
\end{figure*}

%% file: figures/sd3_joint/app.tex
\begin{tikzpicture}
\begin{axis}[
    width=\linewidth,
    height=7cm,
    xlabel={Values of $p$},
    ylabel={Number of Tokens generated},
    xmin=-5.4,
    xmax=5.4,
    ymin=15,
    ymax=82,
    xtick={-5,-4,-3,-2,-1,0,1,2,3,4,5},
    minor x tick num=1,
    grid=both,
    boxplot/draw direction=y,
    every boxplot/.style={
        fill=dit_red,
        draw=black,
        thick,
    },
    boxplot={
        box extend=0.18,
    },
]

\foreach \x/\lw/\qone/\med/\qthree/\hw in {
-5.0/18.0/30.75/37.5/46.0/64.0,
-4.5/18.0/29.75/37.0/45.0/64.0,
-4.0/18.0/31.0/36.0/44.0/63.0,
-3.5/20.0/31.0/37.0/44.0/63.0,
-3.0/21.0/29.0/35.0/43.25/63.0,
-2.5/21.0/30.0/36.0/44.0/65.0,
-2.0/17.0/31.0/35.0/43.0/60.0,
-1.5/21.0/32.0/38.0/45.0/64.0,
-1.0/23.0/33.0/38.0/47.0/67.0,
-0.5/24.0/36.75/44.0/53.0/77.0,
0.0/31.0/46.75/60.0/69.0/79.0,
0.5/44.0/62.25/74.0/77.0/79.0,
1.0/71.0/74.75/76.0/78.0/80.0,
1.5/72.0/75.0/77.0/77.25/79.0,
2.0/73.0/75.0/77.0/77.0/79.0,
2.5/72.0/75.0/76.0/77.0/79.0,
3.0/73.0/75.0/76.0/77.0/79.0,
3.5/72.0/75.0/76.0/77.0/79.0,
4.0/70.0/74.0/76.0/77.0/79.0,
4.5/72.0/75.0/76.0/77.0/79.0,
5.0/71.0/74.0/75.0/77.0/79.0
}{
    \edef\tempplot{%
        \noexpand\addplot+[
            boxplot prepared={
                draw position=\x,
                lower whisker=\lw,
                lower quartile=\qone,
                median=\med,
                upper quartile=\qthree,
                upper whisker=\hw,
            },
        ] coordinates {};
    }%
    \tempplot
}

\end{axis}
\end{tikzpicture}